\documentclass{article}

\usepackage{mathrsfs,amsmath,amsfonts,amssymb,amstext,amscd, dsfont,pifont,
            amsthm,euscript,color,xcolor,accents,xr} 
\usepackage{algorithmic}
\usepackage{algorithm}
\usepackage{url}
\usepackage{mathtools}
\usepackage{epsfig}
\usepackage{float}
\usepackage{appendix}
\usepackage{floatflt}
\usepackage{nicefrac}
\usepackage{anysize}
\usepackage{enumerate}
\usepackage{epsfig}
\usepackage{graphicx}
\usepackage{mathtools}
\usepackage{upgreek}
\usepackage{pdfpages}
\usepackage{subcaption}  
\usepackage{bbold}
\usepackage{bbm,bm}

\usepackage[linecolor=magenta!60!, backgroundcolor=magenta!10!,textwidth=1.6cm, textcolor=magenta]{todonotes}

\numberwithin{equation}{section} 

\newcommand{\heavy}{\theta} 
\newcommand{\Grho}{G^*}

\newcommand*{\Scale}[2][4]{\scalebox{#1}{$#2$}}%
\def\nU{n_{\Scale[0.4]{\mathcal{U}}}}
\def\nX{n_{\Scale[0.4]{\mathcal{X}}}}

\input{macros_nicole}

\usepackage[onehalfspacing]{setspace} 

\definecolor{cite_color}{rgb}{0, 0.3, 0.6}

\usepackage[colorlinks=true,linkcolor=cite_color, citecolor=cite_color, backref]{hyperref}
\usepackage[round]{natbib}
\newcommand{\mycite}[1]{\cite{#1}} 
\newcommand{\mycitep}[1]{\citep{#1}}

\title{Optimal Convergence Rates for Neural Operators}

\author{Mike Nguyen\footnote{Corresponding Author} \\
Technical University  of Braunschweig \\
\texttt{mike.nguyen@tu-braunschweig.de} 
\and
Nicole M\"ucke\\
Technical University  of Braunschweig  \\ 
\texttt{nicole.muecke@tu-braunschweig.de}
}
\date{\today}

\begin{document}
\maketitle


\begin{abstract}
We introduce the neural tangent kernel (NTK) regime for two-layer neural operators and analyze their generalization properties. 
For early-stopped gradient descent (GD), we derive fast convergence rates that are known to be minimax optimal within the framework of non-parametric 
regression in reproducing kernel Hilbert spaces (RKHS). We provide bounds on the number of hidden neurons and the number of second-stage samples necessary 
for generalization. To justify our NTK regime, we additionally show that any operator approximable by a neural operator can also be approximated by an operator 
from the RKHS. A key application of neural operators is learning surrogate maps for the solution operators of partial differential equations (PDEs). We consider 
the standard Poisson equation to illustrate our theoretical findings with simulations.
\end{abstract}


\section{Introduction}

Operator learning has emerged as a transformative methodology in machine learning, focusing on learning mappings 
between spaces of functions. This approach is particularly powerful for surrogate modeling tasks in fields 
like uncertainty quantification, inverse problems, and design optimization, where the goal often involves 
approximating complex operators, such as the solution operator of partial differential equations (PDEs).

{\bf Learning linear operators.} A classical field where learning linear operators arises is functional data analysis (FDA), particularly 
in functional linear regression with functional responses, also known as function-to-function regression. 
FDA focuses on understanding and modeling data best described as functions rather than discrete points. 
Unlike traditional data, where observations might be scalar or vector-valued, functional data are represented 
by smooth curves, surfaces, or other continuous objects. These functions are typically observed over a continuum, 
such as time, space, or wavelength, and are inherently infinite-dimensional 
\mycitep{febrero2024functional,morris2015functional}. The model can be seen as a generalization of 
multivariate regression, where the regression coefficient is now an unknown operator. A standard approach 
to this problem involves applying Tikhonov regularization (see \mycite{benatia2017functional}), with classical 
rates of convergence resembling a typical bias-variance trade-off.
\\
The framework of learning operators extends naturally to nonparametric approaches, such as conditional mean 
embeddings (CMEs) in reproducing kernel Hilbert spaces (RKHS). In these settings, the goal is to represent 
conditional distributions as operators acting on functions, enabling the integration of probabilistic 
modeling with functional data analysis \mycitep{klebanov2020rigorous, park2020measure, grunewalder2012conditional}. 
Conditional mean embeddings generalize the idea of representing conditional expectations as operators in RKHS, bridging 
the gap between statistical learning of distributions and operator learning. In \mycite{li2022optimal} 
classical Tikhonov regularization has also been 
analyzed in this setting, and optimal convergence rates have been established. 
\\
These approaches have inspired further research, as highlighted in \mycite{mollenhauer2022learning}, where the 
problem of learning linear operators is placed  into the broader framework of supervised learning in random design 
and analyzed through the family of spectral 
regularization methods, encompassing the earlier examples as special cases. This line of work is advanced 
in \mycite{meunier2024optimal} and \mycite{li2024towards}, which establish upper bounds for the finite-sample 
risk of general vector-valued spectral algorithms. These results are applicable to both well-specified and 
misspecified scenarios, where the true regression function may lie outside the hypothesis space, and achieve 
minimax optimality across various regimes.

{\bf Learning non-linear operators.} The extension of linear techniques for operator learning into non-linear contexts has been realized through 
innovative approaches incorporating neural networks, enabling greater flexibility and efficiency in handling 
complex functional relationships.  In \mycite{shimizu2024neural}, the authors introduce a method that integrates 
deep learning into CME by leveraging an end-to-end neural network (NN) optimization framework. The approach 
replaces the computationally expensive Gram matrix inversion required in traditional CME methods with a 
kernel-based objective, significantly improving scalability and computational efficiency.
For non-linear function-on-function regression, \mycite{rao2023modern} proposes a novel framework featuring a 
continuous hidden layer with continuous neurons to model functional responses. Two strategies, function-on-function 
direct neural networks and function-on-function basis neural networks, are developed to enhance flexibility in 
modeling complex functional relationships. While demonstrating strong empirical performance through simulations 
and real data applications, this method lacks theoretical learning guarantees. 
\mycite{shi2024nonlinear} presents a functional deep neural network architecture with a fully data-dependent 
dimension reduction strategy. The architecture comprises three steps: a kernel embedding for integral transformation 
using a smooth, data-driven kernel; a projection step using eigenfunction bases for dimension reduction; and a deep 
ReLU network for prediction. This approach is accompanied by theoretical analyses, including approximation and 
generalization error bounds, providing robust theoretical support for the method. These developments highlight 
the significant advancements in non-linear operator learning, driven by neural network-based frameworks that 
blend theoretical rigor with practical flexibility and scalability.


Another key application area for learning non-linear operators mapping between function spaces is solving 
partial differential equations (PDEs), where the solution operators that connect initial or boundary conditions 
to the PDE's solution are of particular interest \mycitep{boulle}. Understanding and approximating these operators 
is essential, as they encapsulate the underlying dynamics of the system being 
studied \mycitep{gonon2024overviewmachinelearningmethods, booksdf}.
\\
Traditional approaches for approximating operators, especially for solving PDEs, have been 
well-established for decades. Techniques such as the Finite Element Method (FEM), Finite Difference 
Method (FDM), and Finite Volume Method (FVM) are classical methods employed to estimate these 
operators \mycitep{hughes2003finite}. These methods discretize the problem space into manageable 
subdomains, enabling numerical approximations of the solution. While highly effective, they have 
notable limitations. A primary drawback is their computational cost, which increases exponentially 
with the problem's dimensionality, a phenomenon known as the curse of dimensionality. Additionally, 
these methods require detailed discretization of the underlying function space, which can be challenging 
to achieve for high-dimensional or complex geometries. Another significant challenge arises when dealing 
with inverse problems or noisy data, where classical methods often struggle with stability and 
robustness \mycitep{strang1973analysis}.
\\
As a result, there has been a growing interest in novel operator estimation techniques that incorporate 
machine learning, and more specifically, neural networks. Neural networks offer the capability to learn 
complex patterns and relationships from data, making them well-suited for operator learning tasks that 
might be intractable via traditional methods \mycite{goodfellow2016deep}. A particularly intriguing 
development is the emergence of neural network-based operator learning techniques, which can be broadly 
categorized into "encoder/decoder" architectures and neural operators. Both categories of methods have 
demonstrated remarkable success in operator learning, showcasing their ability to handle high-dimensional 
inputs, complex geometries, and noisy data.

{\bf Encoder/Decoder Operators.}
This framework typically involves three key stages: encoding the input function into a finite-dimensional 
parameter space, mapping these parameters using neural network architectures, and subsequently decoding 
the parameters back into a functional representation.
The encoding phase often employs parameterization techniques where the function is projected onto a lower-dimensional 
basis, allowing for computationally feasible manipulation. Spectral Neural Operators represent a prominent 
example utilize transformations such as Fourier or wavelet transforms to convert input functions into spectral 
coefficients, delivering notable robustness and efficiency, particularly in scenarios featuring periodic 
behaviors \mycite{fanaskov2024spectralneuraloperators}.
In parallel, PCA Operators utilize Principal Component Analysis to transform input functions into scores 
corresponding to principal components. This transformation highlights significant variance patterns within 
the data, thereby facilitating effective learning and retention of crucial functional 
characteristics \mycite{JMLR:v24:23-0478, kovachki2024operator}. 
Further notable examples include the DeepONet architecture, which comprises a branch network that encodes 
the discrete input function space and a trunk network that encodes the domain of the output 
functions \mycite{articleasdsad}. Additionally, the application of Variational 
Autoencoders \mycite{kingma2022autoencodingvariationalbayes} and Convolutional 
Neural Networks (CNNs) is prominent; the latter utilizes encoder-decoder architectures for 
addressing image-to-image translation problems, where the translation can be interpreted as 
an operator task \mycite{7803544}.

{\bf Neural Operators.}
On the other hand, neural operators constitute an innovative extension of classical neural networks 
designed to directly learn mappings between function spaces. The architecture of neural operators 
resembles that of traditional neural networks but with a crucial distinction: each layer incorporates 
a kernel integral operator to leverage the entire functional information. This design enables them to 
effectively process and map the complex relationships inherent in function spaces. 
One prominent example is the Fourier Neural Operator (FNO), which employs fast Fourier transforms to 
model the kernel integral transformations central to many operators related to partial differential equations (PDEs). 
This approach offers a robust framework capable of efficiently managing high-dimensional input spaces 
with reduced computational complexity \mycite{li2021fourier, qin2024betterunderstandingfourierneural, Kovachki2023NeuralOL}. 
Another example are Low Rank Neural Operators simplify computations by approximating the kernel integral 
operators with low-rank representations, thus reducing the computational complexity while retaining essential 
information \mycite{Kovachki2023NeuralOL}. 
Physics-Informed Neural Operators (PINOs) integrate physical laws and constraints directly into the learning 
process via modified loss functions, which helps in ensuring that the learned mappings adhere to known 
physical principles. This hybrid approach combines data-driven methodologies with theoretical insights, 
enhancing the model's generalization capabilities across diverse scenarios from fluid dynamics to material 
sciences and many others \mycite{doi:10.1126/sciadv.abi8605, RAISSI2019686}. 
Graph Neural Operators extend the neural operator framework to graph-structured data, efficiently capturing 
spatial and topological information intrinsic to many real-world problems. By adapting neural operators to 
graph settings, these models skillfully navigate complex data relationships, making them suitable for 
applications in network analysis and beyond 
\mycite{li2020neuraloperatorgraphkernel, Kovachki2023NeuralOL, sharma2024graph}.

\vspace{0.2cm}

{\bf Contribution.} We aim to learn potential non-linear operators, mapping between normed function spaces. 
We cast this problem into the framework of supervised learning and focus on neural operators that 
encompasses all the 
aforementioned examples. We assume our first-stage data are functions from some possibly 
infinite dimensional normed space, drawn independently and identically from some unknown 
probability distribution. Since we deal with infinite dimensional objects, further discretization is necessary 
for any practical implementation. We suppose that some evaluations of the data functions are available, 
which we call the second-stage samples, see Section 2. 
Our primary objective is to establish fast convergence rates for the mean 
squared error. Following the paradigm in \mycite{kovachki2024data}, we investigate the data complexity 
of learning non-linear operators within the framework of the vector-valued neural tangent kernel (vvNTK). 
In doing so, we meticulously account for the precise number of neurons and the second-stage samples required 
to achieve these rates.
\\
To attain fast convergence rates, however, it is necessary to restrict the class of operators to be estimated. 
More precisely, we impose an a priori smoothness assumption defined in terms of powers of the integral 
operator associated with the vvNTK, commonly referred to as a H\"older source condition in classical RKHS 
methods. Specifically, we limit our scope to learning operators that reside in the range of powers of the 
integral operator, with the well-specified case included. 
To justify this assumption, we demonstrate that any operator approximable by 
a neural operator can also be approximated by an operator from the vvRKHS.
\\
Additionally, we derive convergence rates that account for the assumed eigenvalue decay of the associated 
vvNTK integral operator. Our bounds match those from traditional nonparametric methods in RKHS, which are 
known to be minimax optimal; see also \mycite{nguyen2023random}.

\vspace{0.2cm}

{\bf Organization.} 
The rest of the paper is organized as follows. In Section 2, we present our setting and review relevant results
on learning with kernels, and learning with random features. In Section 3, we
present and discuss our main results. Finally, numerical experiments are presented in Section 4. 
All proofs are deferred to the Appendix.


\section{Setup}
\label{sec:setting}

In this section we provide the mathematical framework for our analysis. 

We let $\cX$ $ \subseteq \mathbb{R}^{d_x}$ be the input space of our function spaces.  The unknown data 
probability measure on the input data space $\cX$ is denoted by $\rho$.
Further we let $\cU$ be a normed function input space, mapping from $\cX \to \cY \subset \mathbb{R}^{d_y}$ 
and $\cV$ be a normed function target space, mapping from $\mathcal{X} \to \tilde{\mathcal{Y}} \subset \mathbb{R}$, 
to be specified in Assumption \ref{ass:input}.   
The unknown data distribution on the data space $\mathcal{W}=\mathcal{U} \times \mathcal{V}$ is denoted 
by $\mu$, while the marginal distribution on $\mathcal{U}$ is denoted as $\mu_{u}$ and the regular 
conditional distribution on $\mathcal{V}$ given $u \in \mathcal{U}$ is denoted by $\mu(\cdot | u)$, see 
e.g. \mycite{Shao_2003_book}. 

Given an operator $G: \mathcal{U} \to \mathcal{V}$  we further define the expected risk as 
\begin{equation}
\label{eq:expected-risk}
\mathcal{E}(G) := \mathbb{E}_{\mu}[ \|G(u)- v\|_{L^2(\rho_x)}^2 ]\;.
\end{equation}  
It is known that the global minimizer of $\cE$ over the set of all measurable operators is 
given by the regression operator $\Grho$, defined as 
\begin{align}
G^*(u):= \int_{\mathcal{V}} v(.) \rho(dv|u).  \label{targetfct}
\end{align}


\subsection{Two-Layer Neural Operator}

The hypothesis class considered in this paper is given by the following set of two-layer neural operators: 
Let $M$ $ \in \mathbb{N}$ be the network width. 
Given an activation function $\sigma: \mathbb{R} \to \mathbb{R}$ acting point wise and some continuous 
operator $A:\mathcal{U}\rightarrow \mathcal{F}(\mathcal{X},\mathbb{R}^{d_k}) $  we consider the class 

\begin{align}
\label{Oclass}
\mathcal{F}_{M} &:=\left\{ G_\theta :\mathcal{U} \to \mathcal{V}  \mid \; G_\theta(u)(x) =  
\frac{1}{\sqrt M} \left\langle a, \sigma\left( B_1 A(u)(x) + B_2u(x)+B_3c(x) \right)\right \rangle\;, \right. \nonumber \\
& \quad \left. \theta =(a, B_1,B_2,B_3 ) \in \mathbb{R}^M \times \mathbb{R}^{ M \times d_k}\times \mathbb{R}^{ M \times d_y}  \times \mathbb{R}^{ M \times d_b}\right\}\;, 
\end{align}

where we denote with $\langle \cdot , \cdot \rangle$ the euclidean vector inner product. 
We condense all parameters in $\theta = (a, B_1,B_2,B_3 )=(a, B )\in \Theta$, with 
$B=(B_1,B_2,B_3)\in\mathbb{R}^{M\times\tilde{d}}$, $\tilde{d}:=d_k+d_y+d_b$ and equip the 
parameter space $\Theta$ with the euclidean 
vector norm $||\cdot||_\Theta$
\begin{align}
\|\theta \|_\Theta^2 = ||a||_2^2 + ||B||_F^2   = \|a\|_2^2 + \sum_{m=1}^M\|b_m\|_2^2 ,   \label{tnorm}
\end{align}

for any $\theta \in \Theta$, where we used the notation $B=(b_1,\dots,b_M)^\top\in\mathbb{R}^{M\times\tilde{d}}$ .

\vspace{0.2cm}
For the activation function $\sigma$, we impose the following assumption.

\vspace{0.2cm}

\begin{assumption}[Activation Function]
\label{ass:neurons} 
There exists  $C_\sigma>0$  such that $\left\|\sigma^{\prime \prime}\right\|_{\infty} \leq C_\sigma,\left\|\sigma^{\prime}\right\|_{\infty} \leq C_\sigma$ , $|\sigma(u)| \leq 1+|u| \text { for } \forall u \in \mathbb{R}$.
In addition we assume that the second derivative $\sigma''$ is Lipschitz continuous.  
\end{assumption}

\vspace{0.2cm}

In our operator class \ref{Oclass}, we first apply an operator transformation $B_1A(u)$, where $B_1$ 
defines a matrix and typically, $A$ defines a kernel integral operator. Depending on the choice of $A$, our 
operator class covers a wide range of practically well-studied neural operators, such as Fourier Neural Operators, 
Graph Neural Operators, and Low-rank Neural 
Operators \mycite{Kovachki2023NeuralOL, huang2024operator, kovachki2024operator}. The kernel integral 
operation $A(u)$, can be seen as an intuitive extension of matrix multiplication in traditional neural 
networks to an infinite-dimensional setting. Inspired by Residual Neural Networks, the second 
transformation, $B_2 u(\,.)$ with $B_2\in\mathbb{R}^{M\times d_y}$ is just a matrix multiplication 
and keeps information of the original input $u(\,.)$. The final component, $B_3 c(\,.)$ with 
$c: \mathcal{X}\to \mathbb{R}^{d_b}$ defines the bias. Unlike standard neural networks, the bias 
here is typically not a constant but a function. This function can be a linear transformation or a shallow 
neural network with trainable parameters. However, to keep our proofs succinct, we assume that $c$ is given 
and instead multiply the bias with a trainable matrix $B_3\in\mathbb{R}^{M\times d_b}$, see Figure \ref{op:class}. 

\vspace{0.3cm}

\begin{figure}[t]
    \centering
    \includegraphics[width=0.5\linewidth]{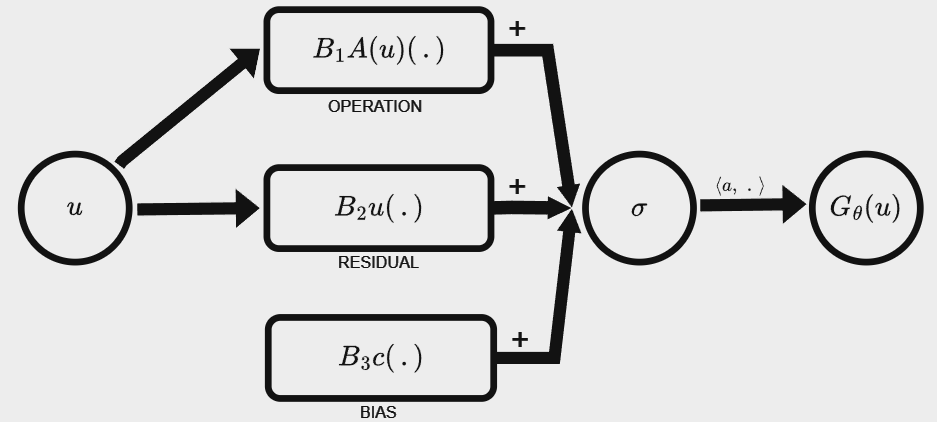}
\caption{Depiction of the architecture of our operator class.}
\label{op:class}    
\end{figure}

\vspace{0.3cm}

In practice, one typically wraps the neural operator between a lifting operator $L$ and a projection operator 
$P$, forming $P\circ G_\theta  \circ L$. Here $L$ and $P$ act pointwise, meaning $L(u)(x)=L(u(x))$, and they are 
typically defined via linear transformations or shallow neural networks  \mycite{Kovachki2023NeuralOL} . However, to 
keep the proofs concise, we assume in our work that $L$ and $P$ are the identity operators. 

Note that our operator class requires the function spaces  $\mathcal{U}$ and $\mathcal{V}$  to have the same 
input space $\mathcal{X}\subset\mathbb{R}^{d_x}$. However, if their input spaces differ, they are typically uplifted 
to a common higher-dimensional space \mycite{Kovachki2023NeuralOL}. For example, let $\mathcal{U}$ contain functions 
mapping from $\tilde{\mathcal{X}}$ and $\mathcal{V}$ contain functions mapping from $\mathcal{X}$ , 
for $\text{dim}(\tilde{\mathcal{X}}) \neq \text{dim}(\mathcal{X})$. In this case, one could consider 
the operator class 
\begin{eqnarray*}
\mathcal{F}_{M} &:=\left\{ G_\theta :\mathcal{U} \to \mathcal{V}  \mid \; G_\theta(u)(x) =  
\frac{1}{\sqrt M} \left\langle a, \sigma\left( B_1 A(u)(x) + B_2u(\xi (x))+B_3c(x) \right)\right \rangle\; \right\},  
\end{eqnarray*}
for a lifting or projection map $\xi: \mathcal{X} \to \tilde{\mathcal{X}}$. This generalization can be 
easily adapted to our analytical results.

\vspace{0.3cm}

Our goal is to minimize the expected risk \eqref{eq:expected-risk} over the set $\cF_M$, i.e. 
\begin{equation*}
\min_{F_\theta \in \cF_M} \cE(G_\theta)\;.
\end{equation*}
Here,  the distributions $\mu$, $\mu_x$ and $\rho_X$ are known only through $\nu$ i.i.d. 
samples $((u_1, v_1), ..., (u_{n_{\Scale[0.4]{\mathcal{U}}}}, v_{n_{\Scale[0.4]{\mathcal{U}}}})) \in (\cU \times \cV)^n$, 
evaluated at $\nX$ points $(x_1, \dots x_{\nX})\in\mathcal{X}^{\nX}$. 
Hence, we seek for a solution for the empirical risk minimization problem 
\[  \min_{G_\theta \in \cF_M} \hat \cE(G_\theta) \;, \quad \hat \cE(G_\theta) = \frac{1}{n_{\Scale[0.4]{\mathcal{U}}}}\sum_{j=1}^{n_{\Scale[0.4]{\mathcal{U}}}} \|G_\theta(u_j)- v_j)\|_{\nX}^2 \;, \] 
where $||\cdot||_{\nX}$ denotes the empirical norm, defined as $||f||_{\nX}^2:=  \frac{1}{\nX}\sum_{j=1}^{\nX} |f(x_j)|^2$.

{\bf Gradient Descent.} We aim at analyzing the generalization properties of gradient descent, 
whose basic iterations are given by   
\begin{align}
\heavy_{t+1}^j &= \heavy_t^j - \alpha \partial_{\theta^j} \hat \cE(G_{\theta_t}) \nonumber  \\
&= \heavy_t^j - \frac{\alpha}{n_{\Scale[0.4]{\mathcal{U}}}}\sum_{i=1}^{n_{\Scale[0.4]{\mathcal{U}}}} 
\left \langle (G_{\theta_t}(u_i) - v_i) , \partial_{\theta^j} G_{\theta_t}(u_i)\right\rangle_{\nX},  
\label{GDalgor}
\end{align}
with $\alpha > 0$ being the stepsize, for some initialization $\theta_0 \in \Theta$ 
and  $\langle u, u'\rangle_{\nX} :=\frac{1}{\nX}\sum_{i=1}^{\nX}u(x_i)u'(x_i)$ denotes the empirical 
inner product. \\


{\bf Initialization.} 
\label{initiali}
Similar as in \mycite{nguyen2023neuronsneed, nitanda2020optimal}, we assume a symmetric initialization. 
The parameters for the output layer are initialized as 
$a_{m}^{(0)}=\tau$ for $m \in\left\{1, \ldots, {M}/{2}\right\}$ and 
$a_{m}^{(0)}=-\tau$ for $m \in\{{M}/{2}+1, \ldots, M\}$, for some $\tau >0$. 
Let $\pi_0$ be a uniform distribution on the sphere 
$\mathbb{S}^{\tilde{d}-1}_1=\{b \in \mathbb{R}^{\tilde{d}} \mid\|b\|_{2}=1\} \subset \mathbb{R}^{\tilde{d}}$. 
The parameters for the input layer are initialized as $b_{m}^{(0)}=b_{m+{M}/{2}}^{(0)}$ 
for $m \in\{1, \ldots, {M}/{2}\}$, where $(b_{m}^{(0)})_{m=1}^{{M}/{2}}$
are independently drawn from the distribution $\pi_{0}$. 
The aim of the symmetric initialization is to make 
an initial operator $G_{\theta_0}\equiv 0$, 
where $\theta_{0}=\left(a^{(0)}, B^{(0)}\right)$. 
Note that this symmetric trick does not have an impact on the limiting NTK, see \mycite{zhang2020type},  
and is just for theoretical simplicity. Indeed, we can relax the 
symmetric initialization by slightly changing our source condition in Assumption \ref{ass:source} to 
\begin{align}
G^* -G_{\theta_0} = \mathcal{L}_\infty^r H^* \;,
\end{align}
for some $H^* \in L^2_\mu$, satisfying $\|H^*\|_{L^2_\mu} \leq R$. The proofs can then be easily adapted.


\subsection{Neural Tangent Kernel (NTK)}

In the more classical setting of nonparametric regression with neural networks, employing techniques 
involving the NTK has been well analyzed \mycitep{nguyen2023neuronsneed, bietti2019inductive, 
chen2020generalized, golikov2022neural, xu2024convergence}. We extend this approach to the setting 
of operator learning, which involves a vector-valued kernel. We therefore recall the definitions and 
properties of vector-valued kernels. For more insights into vector-valued kernels and their RKHSs, we 
refer to \mycite{vvk1,vvk2}.

\begin{definition}
\label{def-vvk}
 
\begin{enumerate}
\item For two function spaces  $\mathcal{U}, \mathcal{V}$ we denote by $\mathcal{F}(\mathcal{U},\mathcal{V})$ the space of operators $F:\mathcal{U}\rightarrow \mathcal{V}$ and by  $\mathcal{C}(\mathcal{U},\mathcal{V})$ the subspace of continuous operators. $\mathcal{L}(\mathcal{U},\mathcal{V})$ denotes the space of linear operators from $\mathcal{U}\to\mathcal{V}$ and $\mathcal{L}(\mathcal{V})$ denotes the space of linear operators from $\mathcal{V}\to\mathcal{V}$.
For $F: \mathcal{U} \to \mathcal{V}$, we write $F^*: \mathcal{U} \to \mathcal{V}$ to denote the adjoint operator. 
\item Given a topological space $\mathcal{U}$ and a separable Hilbert space $\mathcal{V}$, a map $K: \mathcal{U} \times \mathcal{U} \longrightarrow \mathcal{L}(\mathcal{V})$ is
called a $\mathcal{V}$-reproducing kernel on $\mathcal{U}$ if
$$
\sum_{i, j=1}^N\left\langle K\left(u_i, u_j\right) v_j, v_i\right\rangle_{\mathcal{V}} \geq 0
$$
for any $u_1, \ldots, u_N$ in $\mathcal{U}, v_1, \ldots, v_N$ in $\mathcal{V}$ and $N \geq 1$.

\item Given $u \in \mathcal{U}, K_u: \mathcal{V} \rightarrow$ $\mathcal{F}(\mathcal{U} ; \mathcal{V})$ denotes the linear operator whose action on a vector $y \in \mathcal{V}$ is the function $K_u v \in \mathcal{F}(\mathcal{U},\mathcal{V})$ defined by
$$
\left(K_u v\right)(t)=K(t, u) v \quad t \in \mathcal{U} .
$$

Given a $\mathcal{V}$-reproducing kernel $K$, there is a unique Hilbert space $\mathcal{H}_K \subset \mathcal{F}(\mathcal{U} ; \mathcal{V})$ satisfying
$$
\begin{array}{ll}
K_u \in \mathcal{L}\left(\mathcal{V}, \mathcal{H}_K\right) & \forall \,\,\,u \in \mathcal{U} \\
F(u)=K_u^* F & \forall \,\,\,u \in \mathcal{U}, F \in \mathcal{H}_K,
\end{array}
$$
where $K_u^*: \mathcal{H}_K \rightarrow \mathcal{V}$ is the adjoint of $K_u$. Note that from the second property we have that $K(u,t)=K^*_uK_t$. The space $\mathcal{H}_K$ is called the reproducing kernel Hilbert space associated with $K$, given by
$$
 \mathcal{H}_K=\overline{\operatorname{span}}\left\{K_u v \mid u \in \mathcal{U}, v \in \mathcal{V}\right\} .
$$
\item A reproducing kernel $K: \mathcal{U} \times\mathcal{U} \rightarrow \mathcal{L}(\mathcal{V})$ is called Mercer provided that $\mathcal{H}_K$ is a subspace of $\mathcal{C}(\mathcal{U}; \mathcal{V})$.

\end{enumerate}
\end{definition}

The connection between kernel methods and standard neural networks is established via the NTK, 
see \mycite{jacot2018neural, lee2019wide}.  For standard fully connected neural networks, the feature map of the NTK 
is defined via the gradient of the neural network at initialization. Similar we can define a feature map for neural operators by 
$\Phi^M: \cU \to \mathcal{F}(L^2(\mathcal{X},\rho_x),\Theta);\, \Phi^M(u) := \Phi^M_u,$


\[  \Phi_u^M(v) := \nabla_\theta\left\langle G_{\theta_0}(u),\,v\right\rangle_{L^2(\rho_x)} \;. \]


This map defines a $L^2(\mathcal{X},\rho_x)$-reproducing kernel in the sense of Definition \ref{def-vvk}, for 
any $u , u' \in \cU$ via 
\begin{align*}
 K_M(u, u') &=  (\Phi^M_u)^*  \Phi^M_{u'} = \sum_{p=1}^{P} \partial_pG_{\theta_{0}}(u)\otimes\partial_p G_{\theta_{0}}(u')\;\\
&=\frac{1}{M} \sum_{m=1}^{M} \sigma\left(\left\langle b_{m}^{(0) }, J(u)\right\rangle \right) \otimes\sigma\left(\left\langle b_{m}^{(0) }, J(u')\right\rangle\right)\,+\\
&\,\,\,\,\,\,\,\,\,\frac{\tau^2}{M} \sum_{m=1}^{M} \sum_{j=1}^{\tilde{d}}\sigma^{\prime}\left(\left\langle b_{m}^{(0) }, J(u)\right\rangle \right)J(u)^{(j)}\otimes \sigma^{\prime}\left(\left\langle b_{m}^{(0)}, J(u')\right \rangle \right)J(u')^{(j)},
\end{align*}
where we used for the parameter matrix at initialization the notation 
$B^{(0)}=(b_1^{(0)},\dots,b_M^{(0)})^\top\in\mathbb{R}^{M\times\tilde{d}}$, $P=M\cdot(\tilde{d}+1)$ denotes 
the amount of parameters and for any $u\in\mathcal{U}$ we set $J(u)\in \mathcal{F}(\mathcal{X},\mathbb{R}^{d_k+d_y+d_b})$ with $J(u)(x):=(A(u)(x),u(x),c(x))^\top$. 
Further we define the tensor product $\otimes$ with respect to the $L^2(\rho_x)$ inner product 
$\langle \cdot , \cdot \rangle_{L^2_{\rho_x}}$, i.e. 
$(u\otimes v)(f):= u\cdot\langle v\,,\,f\rangle_{L^2_{\rho_x}}$ . The corresponding RKHS is given by 
\[ \mathcal{H}_{M}=
\left\{ H \in\, \mathcal{F}( L^2(\mathcal{X},\rho_x); L^2(\mathcal{X},\rho_x))\,\mid \,H (u)\,= \, 
\langle \tilde{\theta},  \nabla_{\theta} G_{\theta_{0}}(u)\tilde{\theta}\rangle,  \tilde{\theta} \in \Theta\right\},  \]
(\mycite{vvk2} Example 4).


Given the Assumptions \ref{ass:neurons} and \ref{ass:input}  we have for all 
$u\in\mathcal{U}$, $\|J(u)(x)\|_1:= \sum_{j=1}^{\tilde{d}}|(J(u)(x))^{(j)}|\leq  1$ and for all $M$,  
\begin{align*}
&\sup_{u,u' \in\mathcal{U}}\|K_M(u,u')\|\leq 4 +\tau^2C_\sigma^2 :=\kappa^2.
\end{align*}


Further we have that $\nabla_{\theta} G_{\theta_{0}}(.)$  is continuous, therefore 
$K_M$ is mercer. In the following result we show that for all 
$u, u'\in\mathcal{U}$ , $K_M(u,u')$ converges for $M \to \infty$ to 
$K_\infty(u,u')\in\mathcal{C}(L^2(\rho_x),L^2(\rho_x))$  in Hilbert-Schmidt norm, with
\begin{align*}
K_\infty(u,u') &=\mathbb{E}_{\theta_{0}} \sigma\left(\left\langle b^{(0) }, 
J(u)\right\rangle \right) \otimes\sigma\left(\left\langle b^{(0) }, J(u')\right\rangle\right)\,+\\
&\,\,\,\,\,\,\,\,\,\tau^2\sum_{j=1}^{\tilde{d}}
\mathbb{E}_{\theta_{0}}\sigma^{\prime}\left(\left\langle b^{(0) }, J(u) 
\right\rangle \right)J(u)^{(j)}\otimes \sigma^{\prime}\left(\left\langle b^{(0)}, J(u')\right \rangle 
\right)J(u')^{(j)} \,.
\end{align*}

The proof is provided in Appendix \ref{app:conc-inequ}. 

\begin{proposition}
\label{OPbound2}
Given the Assumptions \ref{ass:neurons}, \ref{ass:input} we have for any $u,u'\in \mathcal{U}$  
with probability at least $1-\delta $, where $\delta \in( 0,1)$,
\begin{align*}
\|K_M(u,u')-K_\infty(u,u')\|_{HS}\leq \frac{4\kappa^2}{\sqrt{M}} \log \frac{2}{\delta} .
\end{align*}
\end{proposition}


\section{Main Results}
\label{sec:main-results}

In this section we state all of our results and give an outline of the proof.

\subsection{Assumptions and Main Results}

In this section we formulate our assumptions and state our main results. 
\vspace{0.2cm}

\begin{assumption}[Data Distribution and Operatorclass]
\label{ass:input}
We assume that the input space $\mathcal{U}$ denotes a Banach space, consisting of functions that map from $\mathcal{X}$ to $\mathcal{Y}$. 
For example we could choose $\mathcal{U}=\mathcal{C}(\mathcal{X},\mathcal{Y})$. Similar we assume for our output space $\mathcal{V}$ to be a separable  function Hilbert space, that is naturally embedded in $L^2(\mathcal{X},\mathbb{R},\rho_x)$. 
Further we assume that $\mathcal{U}$, $\mathcal{V}$, the operator $A:\mathcal{U}\rightarrow \mathcal{F}(\mathcal{X},\mathbb{R}^{d_k}) $ and the bias function $c:\mathcal{X}\to\mathbb{R}^{d_b}$ is bounded, i.e. for any $(u,v) \in  supp(\mu)$ and $x\in supp(\rho_x)$ we have $|v(x)|\leq 1$ and 
 $\|J(u)(x)\|_1\leq  1$, where 
$J(u)(x):=(A(u)(x),u(x),c(x))^\top \in \mathbb{R}^{\tilde{d}}$. 
\end{assumption}
\vspace{0.2cm}

We impose these assumptions on the input and output spaces to facilitate the application of the theory 
of random feature approximation for vector valued kernels as developed in \mycite{nguyen2023random} and 
to justify the point evaluations in our gradient descent algorithm \eqref{GDalgor}.\\ 
We bound $\|J(u)(x)\|_1$ in order to keep the NTK finite and limit the bound of $\|J(u)(x)\|_1$ and $|v(x)|$ 
to 1, to keep this paper more concise, but the expressions could also be bounded by any arbitrary constants 
$C_v,C_J>0$.

\vspace{0.2cm}

{\bf Approximation property.} Inspired by the approximation properties of neural operators, our first theorem demonstrates that the 
vvRKHS associated with the vvNTK can approximate all operators that can also be approximated by 
neural operators, whose parameters do not deviate significantly from their initialization. 
The proof can be found in Appendix \ref{app:proof-error-bounds}.

\vspace{0.2cm}

\begin{theorem}[Approximation Property]
\label{approxtheo}
Let $\theta_{0}=(a^{(0)},B^{(0)}) \in \Theta$ be a random initialization, such that 
$\{(a^{(0)}_m,B^{(0)}_m)\}_{m=1}^{M}$ is iid with respect to some measure $\pi_0$. 
Let $K_\infty$ and $\mathcal{H}_\infty$ denote as before the NTK and the RKHS with respect to $\pi_0$.
Suppose Assumptions \ref{ass:input} and \ref{ass:neurons} are satisfied. Further, let $\varepsilon>0$ such 
that there exists some $G_{\theta^*}\in\mathcal{F}_M$ for some $M>0$, with 
$$
\|G_{\theta^*}-G^*\|_{L^2_{\mu_u}}\leq \varepsilon.
$$
Then there exists some $H_\infty\in\mathcal{H}_\infty$ such that with probability at least 
$1-\delta$ (over the initialization),
$$
\| H_\infty-G^*\|_{L^2_{\mu_u}}\leq \varepsilon+\frac{2\kappa \sqrt{\log(2/\delta)}}{M^{\frac{1}{4}}}\left(\|a^{(0)}\|+R\right)+\frac{C R^3}{\sqrt{M}},
$$
with $C >0,R:=\|\theta^*-\theta_{0}\|$ .
\end{theorem}

\vspace{0.3cm}

{\bf Rates of convergence.} In the following, the establish fast convergence rates for the excess risk and 
thus need to impose 
refined apriori assumptions in terms of a source condition and eigenvalue decay.

We let $\cL_\infty: L^2(\mathcal{U},\mathcal{V}, \mu_u) \to L^2(\mathcal{U} ,\mathcal{V}, \mu_u),\,\, \mathcal{L}_\infty G:= \int_{\mathcal{U}} K_{\infty}(u,\cdot)G(u) \mu_u(d u)$ denote the kernel integral operator associated to the NTK $K_\infty$. Here the integral is defined in the sense of Bochner integration (see, e.g., \mycite{Dashti2017}, Sec. A.2). Note  that $\cL_\infty$ is bounded and  self-adjoint. Moreover, it is compact 
and thus has discrete spectrum $\{ \mu_j\}_j$, with $\mu_j \to 0$ as $j \to \infty$. By our assumptions, 
$\cL_\infty$ is of trace-class, i.e., has summable eigenvalues (see e.g. \mycite{Ingo}).

\vspace{0.2cm}

\begin{assumption}[Source Condition]
\label{ass:source}
We assume \footnote{Note that the power of $\mathcal{L}_\infty$ is defined via its SVD: 
$\mathcal{L}_\infty^r  := \sum_{i=1}^\infty \mu_j^r \langle \phi_j, \cdot \rangle \phi_j$.} 
\begin{align}
G^*  = \mathcal{L}_\infty^r H^* \;, \label{hsource}
\end{align}
for some $H_\rho \in L^2(\mathcal{U} , \mu_u)$, satisfying $\|H^*\|_{L^2} \leq R$ with $R>0$, $r\geq 0$. 
\end{assumption} 

This assumption characterizes the hypothesis space and relates to the regularity of the regression 
operator $G^*$. The bigger $r$ is, the smaller the hypothesis space is, the stronger the assumption is, 
and the easier the learning problem is, as 
$\cL_\infty^{r_{1}}\left(L^{2}_{\mu_u}\right) \subseteq \mathcal{L}_\infty^{r_{2}}\left(L^{2}_{\mu_u}\right)$ 
if $r_{1} \geq r_{2}$. Note that $G^* \in \cH_\infty$ holds for all $r\geq \frac{1}{2}$, in fact we have 
$ran(\mathcal{L}^{\frac{1}{2}})=\cH_\infty$ (see e.g. \mycite{Ingo}). 

\vspace{0.2cm}

The next assumption relates to the capacity of the hypothesis space.

\vspace{0.2cm}

\begin{assumption}[Effective Dimension]
\label{ass:dim}
For any $\lambda >0$ we assume 
\begin{align}
\cN_{\cL_\infty}(\lam):= \operatorname{tr}\left(\cL_\infty(\cL_\infty+\lambda I)^{-1}\right) 
\leq c_{b} \lambda^{-b}, \label {effecDim}
\end{align}
for some $b \in[0,1]$ and $c_{b}>0$. 
\end{assumption}

The number $\cN_{\cL_\infty}(\lam)$ is called {\it effective dimension} or {\it degrees of 
freedom} \mycite{Caponetto}. It is related to covering/entropy number conditions, 
see \mycite{Ingo}. The condition (\ref{effecDim}) is naturally satisfied with $b=1$, 
since $\cL_\infty$ is a trace class operator which implies that its eigenvalues $\left\{\mu_{i}\right\}_{i}$ 
satisfy $\mu_{i} \lesssim i^{-1}$. Moreover, if the eigenvalues of $\cL_\infty$ satisfy a polynomial 
decaying condition $\mu_{i} \sim i^{-c}$ for some $c>1$, or if $\cL_\infty$ is of finite rank, then the 
condition (\ref{effecDim}) holds with $b=1 / c$, or with $b = 0$. The case $b = 1$ is referred to as the 
capacity independent case. A smaller $b$ allows deriving faster convergence rates for the studied algorithms.

\vspace{0.3cm}

Our next general result establishes an upper bound for the excess risk in terms of the 
stopping time $T$ and the number of neurons $M$ of our gradient descent algorithm \eqref{GDalgor}, 
under the assumption that the weights remain in a 
vicinity of the initialization. The proof is outlined in Section \ref{outline-of-proof} and 
further detailed in Appendix \ref{app:proof-error-bounds}.

\vspace{0.2cm}

\begin{theorem}
\label{maintheo}
Suppose Assumptions \ref{ass:neurons}, \ref{ass:input},   \ref{ass:source} and \ref{ass:dim} are satisfied. 
Assume further that $\alpha\in (0,\kappa^{-2})$,  $T \leq C\nU^{\frac{1}{2r+b}}$, $2r + b >1$, $\nU\geq n_0:= e^{\frac{2r+b}{2r+b-1}}$, $\nX\geq \tilde{C}B_\tau^2 T^{2r} \log^2(T)$,
\begin{align}
M\geq \tilde{C} B_\tau^6\log^2(\nU) \cdot \begin{cases}
T&: r\in\left(0,\frac{1}{2}\right)\\
T^{1+b(2r-1)}  &: r\in\left[\frac{1}{2},1\right] \\
T^{2r} &: r \in(1,\infty)\,\\
\end{cases}\,,
\end{align}
and
\begin{equation}
\label{eq:ball}
   \forall \;\; t \in [T]\;:\; \;\; \| \theta_t -\theta_0 \|_\Theta \leq B_\tau \;,  
\end{equation}   
for some $B_\tau\geq 1+\tau$.
Then we have with probability at least $1-\delta$,\vspace{0.1cm}
\begin{align}
\label{eq:final-bound}
\| G_{\theta_T}- G^*\|_{L_2(\mu_u)}
&\leq \;  \bar{C} \; T^{-r }\; \log^3(2/\delta ) \;,
\end{align}
\vspace{0.2cm}
with $C, \tilde{C}, \bar{C}>0$ independent of $\nU,\nX, M, T,B_\tau$.
\end{theorem}

\vspace{0.2cm}
Finally, the above result directly leads to rates of convergence for operator learning in the vvNTK regime.

\vspace{0.2cm}

\begin{corollary}
\label{cor:rates}
Suppose all assumptions of Theorem \ref{maintheo} are satisfied. Choosing $T_{\nU}=\nU^{\frac{1}{2r+b}}$ 
gives with probability at least $1-\delta$
\[
\| G_{\theta_{T_{\nU}}}- G^*\|_{L_2(\mu_u)} \leq \;  
\bar{C} \; \log^3(2/\delta )\cdot \left( \frac{1}{\nU} \right)^{\frac{r}{2r+b}} \;,
\]
for some $\bar C >0$.
\end{corollary}

\vspace{0.2cm}

We comment on the above result. Theorem \ref{maintheo} provides an upper bound for the excess risk 
in estimating nonlinear operators using a shallow neural operator network trained with gradient descent. 
The risk bound is expressed in terms of the stopping time and demonstrates that gradient descent is, 
in fact, a descent algorithm, decreasing the risk at each step with high probability, provided the 
algorithm is stopped early. We also specify the minimum number of neurons required to achieve this 
bound, which depends on our a priori assumptions about the model (source condition and eigenvalue decay). 
Notably, in the misspecified case $r \in (0,1/2)$, this minimum number increases and matches the results 
found in \mycite{nguyen2023random} for nonparametric regression in the NTK regime.
In addition, our result provides a minimum number of second stage samples from the discretization to 
achieve this bound.

\vspace{0.2cm}
We can reformulate Theorem \ref{maintheo} to obtain the sample complexity required to achieve a certain 
$\varepsilon$-accuracy. On order for the excess risk to have accuracy of $\varepsilon >0$, the algorithm 
requires $\cO(\varepsilon^{-\frac{2r+b}{r}})$ first stage samples and 
$\cO(\varepsilon^{-2}\log^2(\varepsilon^{-1/r}))$ second stage samples.

\vspace{0.2cm}
Choosing $T=O(\nU^{\frac{1}{2r+b}})$ in Theorem \ref{maintheo} leads to the convergence rate 
$O(\nU^{-\frac{r}{2r+b}})$, which is known to be minimax optimal 
in the RKHS framework \mycite{Caponetto,Muecke2017op.rates}.

\vspace{0.2cm}

For the well-specified case $b=1$, $r=\frac{1}{2}$, where we can only assume that 
$G^* \in \mathcal{H}_\infty$, we obtain from Theorem \ref{maintheo}: 
\[ \| G_{\theta_{T_n}}- G^*\|_{L_2(\mu_u)} =  \mathcal{O}\left(n^{\frac{1}{4}}\right). \] 
Note that in this case we only need $M \geq \mathcal{O}\left(\sqrt{\nU}\log^2\nU\right)$ many neurons and $\nX \geq \mathcal{O}\left(\sqrt{\nU}\log^2\nU\right)$ many second stage samples.

\vspace{0.2cm}

In the kernel learning framework, the assumption $2r+b>1$ refers to easy learning problems and 
if $2r+b\leq 1$ one speaks of hard learning problems \mycite{pillaudvivien2018statistical}. 
In this paper we only investigate easy learning problems and leave the question, how many 
neurons $M$ and second stage samples are needed to obtain optimal rates in hard learning 
problems \mycite{Lin_2020}, open for future work.

\vspace{0.2cm}

We finally want to point out that all bounds given in Theorem \ref{maintheo} and Corollary \ref{cor:rates} 
require the iterates $\theta_t$ to stay close 
to the initialization. We show below in Theorem \ref{weighttheo} that this assumption is indeed satisfied.

\vspace{0.3cm}

{\bf The weights barely move.} 
Our next result shows that the Assumption \eqref{eq:ball} is indeed satisfied and the weights remain 
in a vicinity of the initialization $\theta_0$. 
The proof is provided in Appendix \ref{app:weight-bounds}.

\vspace{0.2cm}

\begin{theorem}[Bound for the Weights]
\label{weighttheo}
Suppose the Assumptions \ref{ass:neurons}, \ref{ass:input},   \ref{ass:source} and \ref{ass:dim} are satisfied.
Let $\delta \in (0,1]$ and $T \geq 3$ and $M\geq M_0$, $\nU,\nX \geq n_0$, where $M_0,n_0$ are defined in \eqref{Mnfinbounds}.
With probability at least $1-10\delta$ it holds 
\[
 \forall \;\; t \in [T]\;:\; \;\; \| \theta_t -\theta_0 \|_\Theta \leq B_\tau\;,  
\] 
where 
\[ B_\tau := C_{R,r,\kappa,\alpha} \log(T) T^{(\frac{1}{2}-r)^{+}}\cdot \cB_\delta\left(\frac{1}{\alpha T}\right)\;,  \]

$$
 \cB_\delta\left(\frac{1}{\alpha T} \right) :=  1+ 10\kappa\sqrt{\cN_{\mathcal{L}_{\infty}}((\alpha T)^{-1})}\left(\sqrt{\frac{\alpha T}{\nU}}+\sqrt{\frac{\alpha T}{\nX}}\right)\log ^{3/2}\left(\frac{12}{\delta} \right),
$$
and $C_{R,r,\kappa,\alpha}>0$ defined in \eqref{eq:def-Bdelta}.
\end{theorem}

\vspace{0.3cm}

\begin{corollary}[Refined Bounds]
\label{cor:weights}
Suppose the assumptions of Theorem \ref{weighttheo} are satisfied. Let 
$ T  = c \nU^{\frac{1}{2r +b}}$, $2r+b>1$. 

\begin{enumerate}
\item Let $r \geq \frac{1}{2}$ and $\nU \geq n_0$, for some $n_0 \in \mbn$ depending on $\delta, r, b$. 
With probability at least $1-\delta$
\begin{equation}
\label{eq:ball-1}
 B_\tau(\delta , T)  
\leq C_{R,r,\kappa,\alpha} \log(T) \;,
\end{equation}
if
\begin{align*}
 M &\geq C_{\kappa, \alpha ,r,\sigma,\|\mathcal{L}_\infty\|}\; \tilde{d}^2 \log^4(T) T^{2r} \log^2(\tilde{d}/\delta) \;,\,\,\\
 \nX&\geq C_{\alpha,b,\kappa,\|\mathcal{L}_{\infty}\|} T^{1+b}\log^3(1/\delta)  \end{align*}

\item Let $r < \frac{1}{2}$. With probability at least $1-\delta$ ,
\begin{align*}
B_\tau(\delta , T) \leq C_{R,r,\kappa,\alpha,b} \log(T) T^{1-2r}.  
\end{align*} 
This holds if we choose 
\begin{align*}
M &\geq C_{\kappa, \alpha ,r,\sigma,\|\mathcal{L}_\infty\|}\; \tilde{d}^2 \log^4(T) T^{5-8r} \log^2(\tilde{d}/\delta) ,\\
\nX&\geq C_{\alpha,b,\kappa,\|\mathcal{L}_{\infty}\|} T^{2r+b}\log^3(1/\delta).
\end{align*}
\end{enumerate}
For some constants $c,\, C_{R,r,\kappa,\alpha,b},\,C_{\kappa, \alpha ,r,\sigma,\|\mathcal{L}_\infty\|},C_{\alpha,b,\kappa,\|\mathcal{L}_{\infty}\|} >0$.
\end{corollary}

\vspace{0.3cm}

\begin{remark}[Analysis of NTK Spectrum.]
It is known that a certain eigenvalue decay of the kernel integral operator $\cL_\infty$ implies 
a bound on the effective dimension. Thus, assumptions on the decay of the effective dimension directly 
relate to the approximation ability of the underlying RKHS, induced by the NTK. 
\\    
Meanwhile, there are some results known that shed light on the RKHS of standard neural networks, that are 
induced by specific NTKs and 
activation functions \mycite{bietti2019inductive}, \mycite{geifman2020similarity}, \mycite{chen2020deep},
\mycite{bietti2020deep},\mycite{fan2020spectra}. 
For neural operators, however, the RKHS has not yet been explored, at least to the best of our knowledge. 
We hope to build on the existing results for neural networks in future research to gain a deeper 
understanding of the RKHS in the context of neural operators. 
\\
However, the approximation properties of neural operators are well-studied \mycite{huang2024operator, 
schwab2021deeplearninghighdimension, Marcati_2023, JMLR:v22:21-0806, kovachki2024operator}. 
For example, \mycite{kovachki2024operator} show that all continuous operators can be arbitrarily well 
approximated by neural operators with a non-polynomial and continuous activation function. 
\end{remark}


\subsection{Outline of Proof} 
\label{outline-of-proof}

Our proof is based on a suitable error decomposition. To this end, 
we further introduce additional linearized iterates in $\cH_M$:

\begin{align}
\label{eq:def-recurions}
F_{t+1}^M(u) &=  F_t^M(u) - \frac{\alpha}{\nU}\sum_{j=1}^{\nU} \ell'(F_t^M(u_j ) , v_j)K_M(u_j , u)   \;, \\ 
H_t(u)(x) &=  \left\langle\nabla (G_{\theta_{0}}(u)(x)), \theta_{t}-\theta_{0}\right\rangle_{\Theta} \;,
\end{align}
with initialization $F^M_0=H_0=0$. The iterative algorithm of $F_t^M$ is known as kernel gradient descent (KGD) with respect to the kernel $K_M$ and $H_t$ describes a linear Taylor approximation of the neural Operator around $\theta_{0}$.  
We may split

\begin{align}
\label{errordecomp}
\|G_{\theta_T} - G^*\|_{L^2_{\mu_u}}
&\leq 
\|G_{\theta_T}  - \mathcal{S}_M H_T\|_{{L^2_{\mu_u}}} + 
\|\mathcal{S}_M (H_T - F_T^M )\|_{L^2_{\mu_u}}  + 
\|\mathcal{S}_MF_T^M-  G^*\|_{L^2_{\mu_u}}  \;,
\end{align}
where $\mathcal{S}_M : \cH_M \hookrightarrow L^2(\cU , \mu_u)$ is the inclusion of $\cH_M$ into $L^2(\cU , \mu_u)$.

\vspace{0.3cm}

The first error in \eqref{errordecomp} describes an Taylor approximation error.  More precisely we use a Taylor expansion in $\theta_{t}$ around the initialization 
$\theta_0$. For any $x \in \mathcal{X}$ and $t \in [T]$, we have
\begin{align}
\label{eq:taylor}
G_{\theta_{t}}(u)(x)&= G_{\theta_{0}}(u)(x) + \mathcal{S}_M\left\langle\nabla (G_{\theta_{0}}(u)(x)), \theta_{t}-\theta_{0}\right\rangle_{\Theta} 
+ r_{(\theta_{t},\theta_{0})}(u)(x) \nonumber \\
&= \mathcal{S}_M H_t(u)(x)+r_{(\theta_{t},\theta_{0})}(u)(x) \; .  
\end{align} 
Here, $r_{(\theta_{t},\theta_{0})}(x)$ denotes the  Taylor remainder and can 
be uniformly bounded by 
\[\|G_{\theta_T}  - \mathcal{S}_M H_T\|_{L^2_{\mu_u}}\leq \| r_{(\theta_{t},\theta_{0})} \|_\infty \lesssim 
B_\tau \; \frac{\|\theta_t-\theta_0\|_\Theta^2}{\sqrt{M}} \;, \] 
as Proposition \ref{prop6} shows. This requires the iterates $\{\theta_t\}_{t \in [T]}$ to stay close to the initialization 
$\theta_0$, i.e. 
\[ \sup_{t \in [T]} \| \theta_t - \theta_0\|_\Theta \leq B_\tau \;,\] 
with high probability, for some $B_\tau < \infty$. 
We show in Theorem \ref{weighttheo} that this is satisfied for sufficiently many neurons.

\vspace{0.3cm}

The second error term in \eqref{errordecomp} depends on the number of neurons $M$ and on the number of second stage samples $\nX$, see Theorem \ref{prop:second-term}. 
More precisely, we obtained 
\[  \|\mathcal{S}_M (H_T - F_T^M )\|_{L^2_{\mu_u}}  \lesssim \log(T)B_\tau^3 \left(\frac{1}{\sqrt{M}}+\frac{1}{\sqrt{\nX}}\right)\;,\] 
\cite{nguyen2023neuronsneed} considered a similar error tern, but for neural networks instead of operators. We used improved spectral regularisation inequalities to obtain rates for second stage samples and a shorter, more concise proof than in \cite{nguyen2023neuronsneed}.

\vspace{0.3cm}

The last error term in \eqref{errordecomp} describes the generalisation error of KGD. We apply the results in \cite{nguyen2023random} and find 
that with high probability, 
\[ \|\mathcal{S}_MF_T^M-  G^*\|_{L^2_{\mu_u}}\lesssim T^{-r} \;, \]
for sufficiently many neurons, see Proposition \ref{prop:random-feature-result}.

As a result, we arrive at an overall bound of Theorem \ref{maintheo}
\[ \| G_{\theta_T}- G^*\|_{L^2_{\mu_u}}
\lesssim \log(T)B_\tau^3 \left(\frac{1}{\sqrt{M}}+\frac{1}{\sqrt{\nX}}\right) +   T^{-r} \;. \]


\section{Numerical Illustration}
\label{sec:numerics}
In this section we support our theoretical investigation that in the well-specified case, the optimal generalization properties can be achieved with a neural operator comprising only $M = O(\sqrt{\nU})$ neurons and $\nX =  O(\sqrt{\nU})$ second stage samples (see Theorem \ref{maintheo}). Remarkably, the empirical evidence presented in Figure \ref{figure1} and \ref{figure2} aligns with this theoretical prediction. More precisely we considered the one dimensional Poisson problem,
\begin{align*}
-\Delta v & =u & \text { in } & (0,1)\\
v & =0 & \text { on } & \{0,1\}.
\end{align*}
We aim to estimate the solution operator that maps the input function $u$ to the solution $v$. In the one-dimensional case with zero boundary conditions, the analytical form of the solution operator is straightforward to obtain. Specifically, we have: 
$$
v(x) = \int_{0}^1 G(x,y) u(y) dy
$$
with $G(x, y)=\frac{1}{2}(x+y-|x-y|)-x y$.

Therefore, we randomly sampled $800$ polynomials $u_i$ and computed their analytical solutions $v_i$. These functions $\{u_i,v_i\}_{i=1}^{400}$ were then used as training data to train our neural operator using GD, and a symmetric initialization (\ref{initiali}). The test set  $\{u_i,v_i\}_{i=400}^{800}$ was used for evaluation.

\begin{figure}[h]
\label{figure0}
\centering \hspace{-1.5cm}
\includegraphics[width=0.45\columnwidth, height=0.22\textheight]{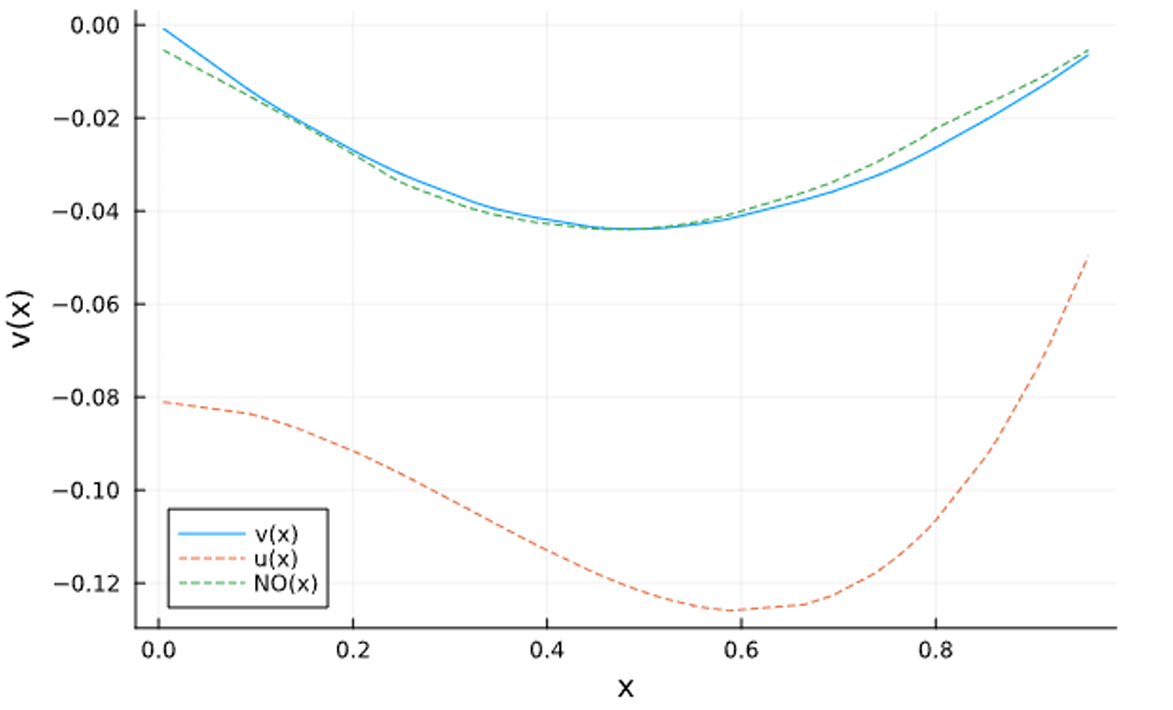}
\caption{A random realization of a polynomial $u$, its solution $v$ and the estimator $NO(u)$.} 

\end{figure}

\begin{figure}[h]
    \begin{minipage}{0.47\textwidth}
        \centering
        \includegraphics[width=\linewidth]{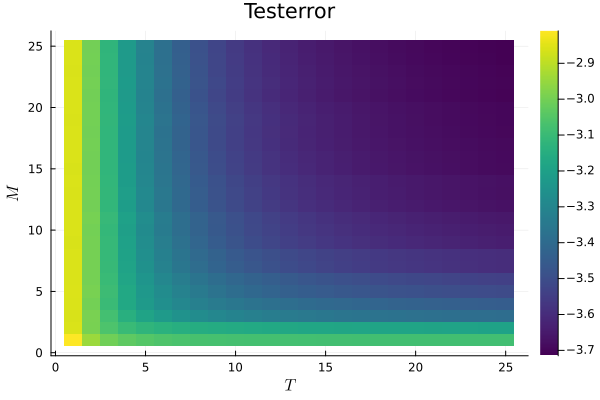}
        \caption{The logarithmic test-error for different choices of neurons $M$ and iterations $T$ and fixed $\nX = 50$.}
        \label{figure1}
    \end{minipage}%
    \hfill
    \begin{minipage}{0.47\textwidth}
        \centering
        \includegraphics[width=\linewidth]{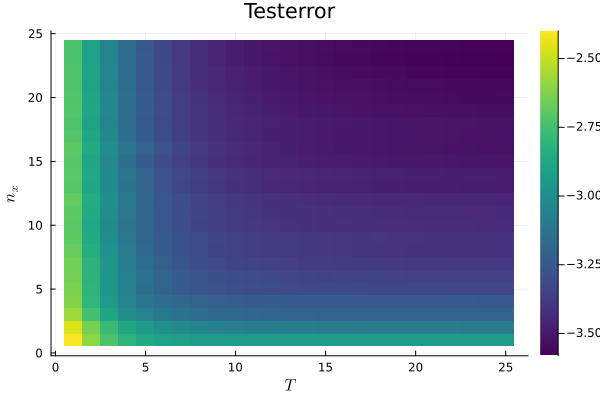}
        \caption{The logarithmic test-error for different choices of $\nX$ and iterations $T$ and fixed $M=50$.}
        \label{figure2}
    \end{minipage}
\end{figure}

\begin{figure}[h]
    \begin{minipage}{0.47\textwidth}
        \centering
        \includegraphics[width=\linewidth]{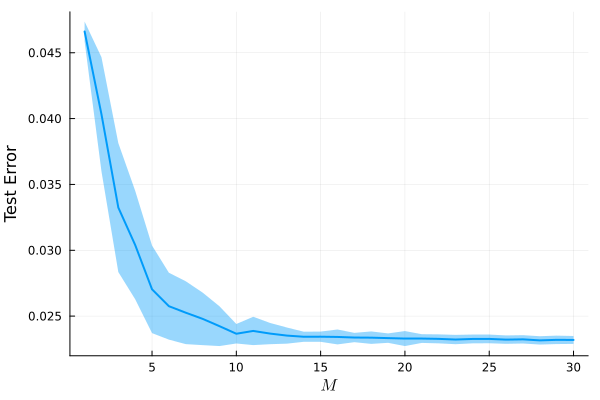}
        \caption{The test-error for different choices of $M$ and fixed $T=50$ and $n_x=50$.}
        \label{figure3}
    \end{minipage}%
    \hfill
    \begin{minipage}{0.47\textwidth}
        \centering
        \includegraphics[width=\linewidth]{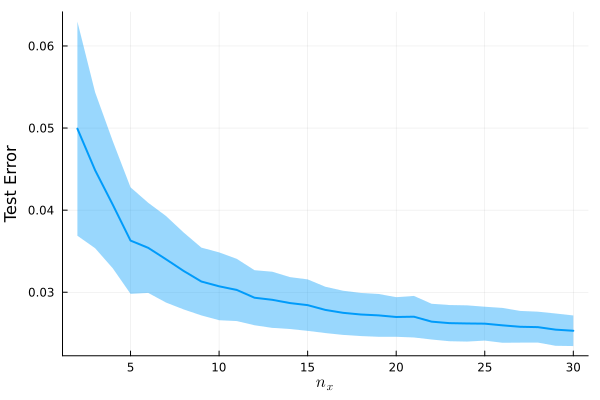}
        \caption{The test-error for different choices of $\nX$ and fixed $T=50$ and $M=50$.}
        \label{figure4}
    \end{minipage}
\end{figure}

As shown in Figures \ref{figure3} and \ref{figure4}, after reaching $\nX = M= T = \sqrt{\nU} = 20$, there is no 
further significant improvement in the neural operator’s performance. For additional numerical simulations that 
underscore the utility and significance of neural operators, see for example 
\mycite{lowery2024kernelneuraloperatorsknos,Kovachki2023NeuralOL,li2021fourier,gin2020deepgreen,10.5555/3648699.3648788,
10.1063/5.0210744,SMAI-JCM_2021__7__121_0}.


\bibliography{bib_iteration}



\appendix

\section{Preliminaries}
\label{prelimss}

For our analysis we need some further notation.  

We denote by $\mathcal{S}_M : \cH_M \hookrightarrow  L^2(\mathcal{U} , \mu_u)$ the inclusion of $\mathcal{H}_M$ into $L^2(\cU , \mu_u)$ for $M \in \mathbb{N}\cup \{\infty\}$.

The adjoint operator $\cS^{*}_M: L^{2}(\mathcal{U}, \rho_u) \longrightarrow \mathcal{H}_{M}$ is identified as
$$
\cS^{*}_M G=\int_{\mathcal{U}}  K_{M,u}G (u)\mu_u(d u)\;,
$$
where $K_{M,u}$ denotes the element of $\mathcal{H}_{M}$ equal to the function $t \mapsto K_M(u, t)$. 
The covariance operator $\Sigma_M: \mathcal{H}_{M} \longrightarrow \mathcal{H}_{M}$ and the kernel 
integral operator $\mathcal{L}_M: L^2(\cU , \mu_u) \to L^2(\cU , \mu_u) $ are given by
\begin{align}
   \Sigma_M G&\coloneqq \cS^*_M\cS_M G = \int_{\mathcal{U}}  K_{M,u} G(u)\mu_u(d u) \;, \\ 
   \mathcal{L}_M G&\coloneqq \cS_M \cS^*_M f = \int_{\mathcal{U}}  K_{M,u} G(u)\mu_u(d u) \;, \label{Intoperator}
\end{align}

which can be shown to be positive, self-adjoint, trace class (and hence is compact) (see e.g. \cite{Ingo}).
The empirical versions of these operators are given by
\begin{center}
\begin{align*}
&\widehat{\cS}_{M}: \mathcal{H}_{M} \longrightarrow \mathcal{V}^{\nU},  &&\left(\widehat{\cS}_{M} G\right)_{j}= K_{M,u_{j}}^*G, \\
&\widehat{\cS}_{M}^{*}: \mathcal{V}^{\nU} \longrightarrow \mathcal{H}_{M}, && \widehat{\cS}_{M}^{*} \mathbf{v}=\frac{1}{\nU} \sum_{j=1}^{\nU}  K_{M,u_{j}}v_{j}, \\
&\widehat{\Sigma}_{M}:=\widehat{\cS}_{M}^{*} \widehat{\cS}_{M}: \mathcal{H}_{M} \longrightarrow \mathcal{H}_{M},&& \widehat{\Sigma}_{M}=\frac{1}{\nU} \sum_{j=1}^{\nU}K_{M,u_{j}}K_{M,u_{j}}^*,
\end{align*}
\end{center}
where $ \mathcal{V}^{\nU}$ is equipped with the norm $\|\mathbf{v}\|_ {{\mathcal{V}}^{\nU}}^2:= \frac{1}{\nU}\sum_{i=1}^{\nU} \|v_i\|^2_{L^2(\rho_x)}$ .
We introduce the following definitions similar to definition 2 of  \cite{rudi2017generalization}.
This framework was originally established to bound the generalization error  $\|f_t^M-g_\rho\|$ where the function $f_t^M$ follows the tikhonov algorithm with respect to an real valued kernel. 
Fortunately these definitions will also be useful to bound the weights of the neural operator.
\begin{center}
\begin{align*}\label{weightops}
&\mathcal{Z}_M: \Theta \rightarrow \mathcal{H}_M,  &&\left(\mathcal{Z}_M \theta\right)(\cdot)=\nabla G_{\theta_{0}}(.)^{\top} \theta, \\[5pt]
&\mathcal{Z}_M^*:\mathcal{H}_M \rightarrow \Theta,
&& \mathcal{Z}_M^* G=\int_\mathcal{U}\int_\mathcal{X} \nabla G_{\theta_{0}}(u) (x)G(u)(x) d\rho_x(x)d \mu_u(u), 
\end{align*}
\end{center}
\begin{center}
\begin{align*}
&\widehat{\mathcal{Z}}_M: \Theta \rightarrow \left(\mathbb{R}^{ \nU \times \nX},\frac{1}{\sqrt{\nU \nX}}\|.\|_F\right),  &&\left(\widehat{\mathcal{Z}}_M \theta \right)_{i,j}=\nabla G_{\theta_{0}}(u_i)(x_j)^{\top} \theta, \\
&\widehat{\mathcal{Z}}_M^*: \left(\mathbb{R}^{\nU \times \nX},\frac{1}{\sqrt{\nU\nX}}\|.\|_F\right) \rightarrow \Theta,
&& \widehat{\mathcal{Z}}_M^* a=\frac{1}{\nU \nX}\sum_{i=1}^{\nU}\sum_{j=1}^{\nX}\nabla G_{\theta_{0}}(u_i)(x_j)a_{i,j} , 
\end{align*}
\end{center}
\begin{center}
\begin{align}
&\mathcal{C}_M: \Theta \rightarrow \Theta,
&& \mathcal{C}_M=\int_\mathcal{U}\int_{\mathcal{X}}\nabla G_{\theta_{0}}(u)(x) \nabla G_{\theta_{0}}(u)(x)^{\top} d \rho_x(x)d\mu_u(u),\\
&\widehat{\mathcal{C}}_M:\Theta \rightarrow \Theta,
&& \widehat{\mathcal{C}}_M=\frac{1}{\nU \nX}\sum_{i=1}^{\nU} \sum_{j=1}^{\nX} \nabla G_{\theta_{0}}(u_i)(x_j) \nabla G_{\theta_{0}}(u_i)(x_j)^{\top}.\label{Zoperators}
\end{align}
\end{center}

\begin{remark}
Note that $\mathcal{C}_M$ and $\widehat{\mathcal{C}}_M$ are self-adjoint and positive operators, 
with spectrum in $[0, \kappa^2]$ and we further have 
$\mathcal{C}_M=\mathcal{Z}_M^*\mathcal{Z}_M$,   
$\widehat{\mathcal{C}}_M=\widehat{\mathcal{Z}}_M^*\widehat{\mathcal{Z}}_M$,  
$\Sigma_M=\mathcal{Z}_M\mathcal{Z}_M^*$.  
\end{remark}

\section{Error Bounds}
\label{app:proof-error-bounds}

Recall the error decomposition from Section \ref{outline-of-proof}:

\begin{align}
\label{errordecomp2}
\|G_{\theta_T} - G^*\|_{L^2_{\mu_u}}
&\leq \underbrace{ \|G_{\theta_T}  - \mathcal{S}_M H_T\|_{L^2_{\mu_u}}}_{\mathcal{I}} + 
\underbrace{  \|\mathcal{S}_M (H_T - F_T^M )\|_{L^2_{\mu_u}}}_{\mathcal{II}}  + 
\underbrace{  \|\mathcal{S}_MF_T^M-  G^*\|_{L^2_{\mu_u}}}_{\mathcal{III}}  \;.
\end{align}

This section is devoted to bounding each term on the right hand side of \eqref{errordecomp2}. 
To this end, we need the following assumptions:

\begin{assumption}
\label{ass:Taylor-is-satisfied}
Let  $T \in \mbn$. Assume we have
\begin{equation}
\label{eq:AT}
 \forall \;\; t \in [T]\;:\; \;\; \| \theta_t -\theta_0 \|_\Theta \leq B_\tau\;,  
\end{equation} 
for some constant $B_\tau \geq 1+ \tau$.
\end{assumption}

We will show in Section \ref{app:weight-bounds} that this assumption is satisfied with high probability, as long as $M$ is large enough.


\subsection{Bounding $\mathcal{I}$} 

In this section we provide an estimate of the first error term 
$ \|G_{\theta_T} - \mathcal{S}_M H_T\|_{L^2_{\mu_u}}$ in \eqref{errordecomp2}.

\vspace{0.3cm}

\begin{proposition}
\label{prop:taylor-remainder} 
Suppose the Assumptions \ref{ass:Taylor-is-satisfied} is satisfied. 

\[   \| G_{\theta_T}  - \mathcal{S}_M H_T\|_{L^2_{\mu_u}} \leq   \frac{ C_{\nabla G}(B_\tau)B_\tau^2 }{\sqrt{M}}  \; , \]
with $C_{\nabla G}(B_\tau)\coloneqq  162 \max \left\{C_\sigma^2,C_\sigma^2\left(B_\tau^2+\tau^2 \right)\right\}.$ 
Recall that the norm is defined as $\|G \|_{L^2_{\mu_u}} ^2:= \int_{\mathcal{U}} \|G(u)\|_{L^2_{\rho_x}}^2d \mu_x $.
\end{proposition}

\vspace{0.2cm}

\begin{proof}[Proof of Proposition \ref{prop:taylor-remainder}]
From the Assumption \ref{ass:Taylor-is-satisfied} and  Proposition \ref{prop6}, $a)$ we immediately obtain for all $x \in \cX, u \in\mathcal{U}$
\begin{align}
\label{eq:taylor2}
| G_{\theta_T}(u)(x) - \mathcal{S}_M H_T(u)(x)|  &=  |r_{(\theta_0,\theta_T)}(u)(x) | \nonumber \\
&\leq \frac{C_{\nabla G}(B_\tau)}{\sqrt{M}} \; \|\theta_T-\theta_0\|^2_\Theta \nonumber \\
&\leq  \frac{C_{\nabla G}(B_\tau)B_\tau^2 }{\sqrt{M}} \; .
\end{align}

\end{proof}


\subsection{Bounding $\mathcal{II}$} 

In this section we estimate the second term $ \|\mathcal{S}_M (H_T - F_T^M )\|_{L^2_{\mu_u}}$ in \eqref{errordecomp2}. 
A short calculation proves the following recursion:

\vspace{0.3cm}

\begin{lemma}
\label{lem:general-recursion}
Let $t \in \mbn$, $M \in \mbn$. Define $\hat  u_t:=H_t - F_t^M \in \mathcal{H}_M$. Then $(\hat  u_t)_t$ follows the recursion $\hat  u_0 = 0$ and 
\begin{align*}
 \hat u_{t+1} &= (Id - \alpha \widehat{\Sigma}_M) \hat  u_t - \alpha \mathcal{Z}_M \hat  \xi_t^{(1)}  - \alpha \mathcal{Z}_M \hat  \xi_t^{(2)} - \alpha \widehat{\cS}^*_M  \hat  \xi_t^{(3)}\\
&= \alpha \sum_{s=0}^{t-1} (Id - \alpha \widehat{\Sigma}_M)^s \left(\mathcal{Z}_M \hat  \xi_{t-s}^{(1)} + \mathcal{Z}_M \hat  \xi_{t-s}^{(2)} + \widehat{\cS}^*_M \hat  \xi_{t-s}^{(3)} \right) \;, 
 \end{align*}
where
\begin{align*}
\hat  \xi_t^{(1)} &=   \frac{1}{\nU\nX} \sum_{i=1}^{\nU}\sum_{j=1}^{\nX}(G_{\theta_t}(u_i)(x_j)  -  v_i(x_j))\left(\nabla G_{\theta_t}(u_i)(x_j) - \nabla G_{\theta_0}(u_i)(x_j) \right)\in\Theta \;, \\
\hat  \xi_t^{(2)} &= \frac{1}{\nU} \sum_{i=1}^{\nU}\left(\hat{\mathbf{s}}_{i}-\mathbf{s_i}\right)\in\Theta  \;,\\
\hat  \xi_t^{(3)} &= \bar r_{(\theta_0 , \theta_t)}  \in \cV^{\nU}\;,
\end{align*}
with  $\bar {r}_{(\theta_0 , \theta_t)} = \left(r_{(\theta_0 , \theta_t)}(u_1),\dots, r_{(\theta_0 , \theta_t)}(u_{\nU})\right) $ and  $\,\hat{\mathbf{s}}_{i}\,,\mathbf{s}_i\in \Theta$ , $\hat{\mathbf{s}}^{(m)}_{i}:=\left\langle(G_{\theta_t}(u_i) -  v_i), \partial_m G_{\theta_0}(u_i)  \right\rangle_{\nX},\,\,$ $\mathbf{s}^{(m)}_i:=\left\langle(G_{\theta_t}(u_i) -  v_i), \partial_m G_{\theta_0}(u_i)  \right\rangle_{L^2_{\rho_x}}$ .
\end{lemma}

\begin{proof}[Proof of Lemma \ref{lem:general-recursion}]
Plugging in $H_t$ the recursive definition of $\theta_ t$ leads to
\begin{align}\label{shortcalc}
\nonumber H_{t+1} &= \left \langle \nabla G_{\theta_0}(u) , \theta_{t+1}-\theta_{0}\right\rangle_{\Theta}\\
& = H_t - \alpha \mathcal{Z}_M \hat  \xi_t^{(1)}  - \alpha \mathcal{Z}_M \hat  \xi_t^{(2)} - \alpha  \widehat{\cS}^*_M \hat  \xi_t^{(3)}- \frac{\alpha}{n} \sum_{i=1}^n   K_M(u_i,\,.) (H_{t}(u_i) - v_i).
\end{align}
Here we used the fact that $G_{\theta_{0}}\equiv 0$. Therefore, applying a Taylor expansion again yields $G_{\theta_{t}}=H_t + {r}_{(\theta_0 , \theta_t)}$.
Combining the above recursive formula of $H_t$ with the recursive definition of $F_t^M$ results in the first equation of the statement. The second equation can be derived by iterating the first equation $t$ times.
\end{proof}

\vspace{0.4cm}

\begin{theorem}
\label{prop:second-term}
Let $\delta \in (0,1]$,  $\alpha < 1/\kappa^2$. 
Suppose Assumptions \ref{ass:neurons}, \ref{ass:source}, \ref{ass:Taylor-is-satisfied},  are satisfied. 
Let $M\geq M_0:=(8 \tilde{d} \kappa^2 \beta_\infty \alpha T)\vee 8\kappa^4\|\mathcal{L}_\infty\|^{-1}\log^2 \frac{2}{\delta}$ with $\beta_\infty=\log \frac{4 \kappa^2(\mathcal{N}_{\mathcal{L}_\infty}(\alpha T)+1)}{\delta\|\mathcal{L}_\infty\|} $.
Then we have with probability at least $1-\delta$,
\begin{align*}
  \forall t \in [T]: \;\;\;  ||  \cS_M \hat u_{t}||_{L^2(\mu_u)} &\leq 4\xi(M,\nX,B_\tau,\kappa) \left(\log(T)+6\right) \;, 
\end{align*}
where $u_t:=H_t - F_t^M$ and
\begin{align*}
&\xi(M,\nX,B_\tau,\kappa,\delta)\\
&:=\left(\frac{C_{\nabla G}(B_\tau)B_\tau}{\sqrt{M}}\left(\kappa  B_\tau+ \frac{C_{\nabla G}(B_\tau)}{\sqrt{M}}B_\tau^2+1\right) + \frac{4\kappa^2\left(B_\tau+1\right)}{\sqrt{\nX}} \log \left(\frac{16}{\delta} \right)+ \frac{(2 \kappa+1) C_{\nabla G}(B_\tau)B_\tau^2 }{\sqrt{M}}\right).
\end{align*}
If we further assume $M\geq M_{II}:=\max\{M_0,M_1\}$ with $M_1:= 48C_{\nabla G}(B_\tau)^2B_\tau^4\kappa T^{2r}\left(\log(T)+6\right)^2$ and $\nX\geq 32\kappa^4\left(B_\tau+1\right)^2T^{2r} \left(\log(T)+6\right)^2\log \left(\frac{16}{\delta}\right)$, we have
\begin{align*}
  \forall t \in [T]: \;\;\;  ||  \cS_M \hat u_{t}||_{L^2(\mu_u)} &\leq \frac{1}{T^r} \;. 
\end{align*}
\end{theorem}

\begin{proof}[Proof of Theorem \ref{prop:second-term}]
Applying \cite[Proposition A.15]{nguyen2023random}  gives (for  $\lambda:= (\alpha T)^{-1}$ and $M\geq (8 \tilde{d} \kappa^2 \beta_\infty \alpha T)\vee 8\kappa^4\|\mathcal{L}_\infty\|^{-1}\log^2 \frac{2}{\delta}$) with probability at least $1-\delta$
\begin{align}
\label{eq:from-useful}
||  \cS_M \hat u_{t}||_{L^2} &= ||\Sigma_M^{\frac{1}{2}} \hat u_{t}||_{\cH_M} 
\leq 2 \; || \widehat{\Sigma}_M ^{\frac{1}{2}} \hat u_{t} ||_{\cH_M} + 2 \; \frac{|| \hat u_{t} ||_{\cH_M}}{\sqrt{\alpha T}} \,,
\end{align}

Let $a \in \{0, 1/2\}$. Using Lemma \ref{lem:general-recursion}, we find for any $t\in[T]$,
\begin{align}
|| \widehat{\Sigma}_M ^a \hat u_{t+1} ||_{\cH_M} 
      &=  \alpha \left\|  \sum_{s=0}^{t-1} \widehat{\Sigma}_M ^a (Id - \alpha \widehat{\Sigma}_M)^s \left(\mathcal{Z}_M \hat  \xi_{t-s}^{(1)} + \mathcal{Z}_M \hat  \xi_{t-s}^{(2)} + \widehat{\cS}^*_M \hat  \xi_{t-s}^{(3)} \right)\right\|_{\cH_M} \nonumber  \\
      &\leq \alpha^{1-a}  \sum_{s=0}^{T-1} || (\alpha\widehat{\Sigma}_M )^a (Id - \alpha \widehat{\Sigma}_M)^s \mathcal{Z}_M \hat  \xi_{T-s}^{(1)}||_{\cH_M} + \nonumber \\
          &\hspace{0.2cm} \alpha^{1-a}  \sum_{s=0}^{T-1} || (\alpha\widehat{\Sigma}_M )^a (Id - \alpha \widehat{\Sigma}_M)^s \mathcal{Z}_M \hat  \xi_{T-s}^{(2)}||_{\cH_M} + \nonumber \\
        &\hspace{0.2cm}    \alpha^{1-a}  \sum_{s=0}^{T-1} || (\alpha\widehat{\Sigma}_M )^a (Id - \alpha \widehat{\Sigma}_M)^s  \widehat{\cS}^*_M \hat  \xi_{T-s}^{(3)}||_{\cH_M} \;.
\end{align}

From Proposition \ref{sumbound} we further obtain with probability at least $1-\delta$,

\begin{align}
\label{eq:from-useful2}
\| \widehat{\Sigma}_M ^a \hat u_{t+1} \|_{\cH_M} 
      \leq &\alpha^{1-a} \left(\frac{2\eta_{a+\frac{1}{2}}(T)}{\sqrt{\alpha}}+\frac{2\eta_{a}(T)}{\sqrt{\alpha T}}\right)\max_{s\leq T}\left(\|\hat  \xi_{s}^{(1)}\|_{\Theta} +\|\hat  \xi_{s}^{(2)}\|_{\Theta} +\|  \xi_{s}^{(3)}\|_{\mathcal{V}^{\nU}}\right) ,
\end{align}

\vspace{0.2cm} 
where $\eta_a$ was defined in Lemma \ref{prop1}.

{\bf Bounding $\|\hat  \xi_{t}^{(1)}\|_{\Theta}$:}
We have
\begin{align*}
\|\hat  \xi_{t}^{(1)}\|_{\Theta}& =  \left\| \frac{1}{\nU\nX} \sum_{i=1}^{\nU}\sum_{j=1}^{\nX}(G_{\theta_t}(u_i)(x_j)  -  v_i(x_j))\left(\nabla 
G_{\theta_t}(u_i)(x_j) - \nabla G_{\theta_0}(u_i)(x_j) \right)\right\|_{\Theta}\\
&\leq  \frac{1}{\nU\nX} \sum_{i=1}^{\nU}\sum_{j=1}^{\nX}\left|G_{\theta_t}(u_i)(x_j)  -  v_i(x_j)\right|\left\|\nabla 
G_{\theta_t}(u_i)(x_j) - \nabla G_{\theta_0}(u_i)(x_j) \right\|_{\Theta}.
\end{align*}
From the Assumption \ref{ass:input} we have $|v_i(x_j)|\leq1$ and from Proposition \ref{NNObound} we have  
\begin{align*}
\left|G_{\theta_t} (u)(x)\right| \leq \kappa\|\theta_t-\theta_{0}\|_{\Theta}+ \frac{C_{\nabla G}(B_\tau)}{\sqrt{M}}\|\theta_t-\theta_{0}\|^2_{\Theta},
\end{align*}
with $C_{\nabla G}(B_\tau)\coloneqq  162 \max \left\{C_\sigma^2,C_\sigma^2\left(B_\tau^2+\tau^2 \right)\right\}.$ Using these bounds we obtain 

\begin{align*}
\|\hat  \xi_{t}^{(1)}\|_{\Theta}& \leq\left(\kappa\|\theta_t-\theta_{0}\|_{\Theta}+ \frac{C_{\nabla G}(B_\tau)}{\sqrt{M}}\|\theta_t-\theta_{0}\|^2_{\Theta}+1\right)\sup_{x\in\mathcal{X},u\in\mathcal{U}}\left\|\nabla 
G_{\theta_t}(u)(x) - \nabla G_{\theta_0}(u)(x) \right\|_{\Theta}\\
&\leq \left(\kappa\|\theta_t-\theta_{0}\|_{\Theta}+ \frac{C_{\nabla G}(B_\tau)}{\sqrt{M}}\|\theta_t-\theta_{0}\|^2_{\Theta}+1\right)\left(\frac{C_{\nabla G}(B_\tau)}{\sqrt{M}}\|\theta_t-\theta_{0}\|_\Theta\right)\\
&\leq \frac{C_{\nabla G}(B_\tau)B_\tau}{\sqrt{M}}\left(\kappa B_\tau+ \frac{C_{\nabla G}(B_\tau)}{\sqrt{M}}B_\tau^2+1\right),
\end{align*}
\vspace{0.2cm} 

where we used Proposition \ref{prop:LipschitzGradient} for the second inequality.

{\bf Bounding $\|\hat  \xi_{t}^{(2)}\|_{\Theta}$:}
 By definition of $\hat  \xi_{t}^{(2)}$ we have
\begin{align}\left\|\hat  \xi_{t}^{(2)}\right\|_{\Theta} &= \left\|\frac{1}{\nU\nX} \sum_{i=1}^{\nU}\sum_{j=1}^{\nX} \nabla G_{\theta_{0}}(u_i)(x_j)(G_{\theta_t}(u_i)(x_j) -  v_i(x_j)) - \mathbb{E}_{\rho_x}\nabla G_{\theta_{0}}(u_i)(x)(G_{\theta_t}(u_i)(x) -  v_i(x))\right\|_{\Theta}   \;,\\[5pt]
&\leq \left\|\frac{1}{\nU\nX} \sum_{i=1}^{\nU}\sum_{j=1}^{\nX} \nabla G_{\theta_{0}}(u_i)(x_j)v_i(x_j)  - \mathbb{E}_{\rho_x}\nabla G_{\theta_{0}}(u_i)(x)v_i(x)\right\|_{\Theta} \;+,\\[5pt]
&\,\,\,\,\,\,\,\, \left\|\frac{1}{\nU\nX} \sum_{i=1}^{\nU}\sum_{j=1}^{\nX} \nabla G_{\theta_{0}}(u_i)(x_j)G_{\theta_t}(u_i)(x_j)  - \mathbb{E}_{\rho_x}\nabla G_{\theta_{0}}(u_i)(x)G_{\theta_t}(u_i)(x) \right\|_{\Theta}   \\[5pt]
&:=I+II.
\end{align}
For $I$ we set $W_j:=\frac{1}{\nU}\sum_{i=1}^{\nU} \nabla G_{\theta_{0}}(u_i)(x_j)v_i(x_j)$. Note that $\|W_j\|_2\leq\kappa^2:=B,\, \mathbb{E}\|W_j\|_2^2\leq\kappa^4:=V^2$. Therefore we have from Proposition \ref{concentrationineq0} with probability at least $1-\delta$,
$$
I\leq \frac{4\kappa^2}{\sqrt{\nX}} \log \left(\frac{4}{\delta} \right).
$$
For $II$, we cannot directly apply Proposition \ref{concentrationineq0}, since the summands involve $\theta_t$, which means they are no longer independent. However from Proposition \ref{cilast} we have that with probability at least $1-\delta$,
$$
II\leq \frac{4\kappa^2}{\sqrt{\nX}} \log \left(\frac{4}{\delta} \right)\left\|\theta_t-\theta_0\right\|_{\Theta}+ \frac{2 \kappa C_{\nabla G}(B_\tau)}{\sqrt{M}}\|\theta_t-\theta_{0}\|^2_{\Theta}.
$$
Therefore we have,

\begin{align}
\left\|\hat  \xi_{t}^{(2)}\right\|_{\Theta}&\leq \frac{4\kappa^2}{\sqrt{\nX}} \log \left(\frac{4}{\delta} \right)\left(\left\|\theta_t-\theta_0\right\|_{\Theta}+1\right)+ \frac{2 \kappa C_{\nabla G}(B_\tau)}{\sqrt{M}}\|\theta_t-\theta_{0}\|^2_{\Theta}\\
&\leq \frac{4\kappa^2\left(B_\tau+1\right)}{\sqrt{\nX}} \log \left(\frac{4}{\delta} \right)+ \frac{2 \kappa C_{\nabla G}(B_\tau)B_\tau^2 }{\sqrt{M}}.
\end{align}

{\bf Bounding $\|  \xi_{t}^{(3)}\|_{\mathcal{V}^{\nU}}$:}

Using Proposition \ref{prop6} we have

\begin{align*}
\left\|  \xi_{t}^{(3)}\right\|_{\mathcal{V}^{\nU}} =  \sqrt{\frac{1}{\nU}\sum_{i=1}^{\nU} \|r_{(\theta_t,\theta_{0})}(u)\|^2_{L^2(\rho_x)}}\leq \frac{C_{\nabla G}(B_\tau)B_\tau^2}{\sqrt{M}}.
\end{align*}

{\bf To sum up:}

Plugging the bounds of the errors $  \xi_{t}^{(i)}, \, i \in \{1,2,3\}$ into \ref{eq:from-useful2} leads to
\begin{align}\label{Hafer}
\| \widehat{\Sigma}_M ^a \hat u_{T+1} \|_{\cH_M} 
      &\leq \alpha^{1-a} \left(\frac{2\eta_{a+\frac{1}{2}}(T)}{\sqrt{\alpha}}+\frac{2\eta_{a}(T)}{\sqrt{\alpha T}}\right)\max_{s\leq T}\left(\|\hat  \xi_{s}^{(1)}\|_{\Theta} +\|\hat  \xi_{s}^{(2)}\|_{\Theta} +\|  \xi_{s}^{(3)}\|_{\mathcal{V}^{\nU}}\right) \\
     &\leq  \alpha^{1-a} \left(\frac{2\eta_{a+\frac{1}{2}}(T)}{\sqrt{\alpha}}+\frac{2\eta_{a}(T)}{\sqrt{\alpha T}}\right) \cdot \xi(M,\nX,B_\tau,\kappa,\delta),
\end{align}

with $$\xi(M,\nX,B_\tau,\kappa,\delta):=\left(\frac{C_{\nabla G}(B_\tau)B_\tau}{\sqrt{M}}\left(\kappa B_\tau+ \frac{C_{\nabla G}(B_\tau)}{\sqrt{M}}B_\tau^2+1\right) + \frac{4\kappa^2\left(B_\tau+1\right)}{\sqrt{\nX}} \log \left(\frac{4}{\delta} \right)+ \frac{(2 \kappa+1) C_{\nabla G}(B_\tau)B_\tau^2 }{\sqrt{M}}\right).$$

Plugging \eqref{Hafer} into \eqref{eq:from-useful} gives, 

\begin{align*}
||  \cS_M \hat u_{T}||_{L^2} &
\leq 2 \; || \widehat{\Sigma}_M ^{\frac{1}{2}} \hat u_{T} ||_{\cH_M} + 2 \; \frac{|| \hat u_{T} ||_{\cH_M}}{\sqrt{\alpha T}}\\
&\leq  4\xi(M,\nX,B_\tau,\kappa,\delta) \left(\eta_{1}(T)+\frac{\eta_{1/2}(T)}{\sqrt{T}}+\frac{\eta_{1/2}(T)}{\sqrt{T}}+\frac{\eta_{0}(T)}{T}\right) \\
&=  4\xi(M,\nX,B_\tau,\kappa,\delta) \left(\log(T)+6\right).
\end{align*}

Note that, to derive this final bound we conditioned on four events, where each of the events holds true with probability at least $1-\delta$. By Proposition \ref{conditioning}, we therefore obtain the last inequality holds with probability at least $1-4\delta$. Redefining  $\delta=4\delta$ proves the claim.
\end{proof}


\subsection{Bounding $\mathcal{III}$}

For bounding the last term in our error decomposition we use the results obtained in \cite{nguyen2023random}.

\begin{proposition}[Theorem A.4. in \cite{nguyen2023random}]
\label{prop:random-feature-result}
Suppose Assumptions \ref{ass:neurons}, \ref{ass:source}, \ref{ass:input} and \ref{ass:dim} are satisfied. 
Let $\delta \in (0,1]$, $\alpha < 1/\kappa^2$,  $T \leq C_1\nU^{\frac{1}{2r+b}}$, $2r + b >1$ and $\nU\geq n_0:= e^{\frac{2r+b}{2r+b-1}}$. With probability at least $1-\delta,$ we have
\[ \| \mathcal{S}_Mf_{T}^M-  G^*\|_{L^2(\mu_u)} \leq C_2\; \log^3(2/\delta ) \; T^{-r } \;,\]
provided that 

\begin{align}
\label{eq:assumpt-M-3}
M\geq M_{III}(T,r):= C_3\log(\nU) \cdot \begin{cases}
T&: r\in\left(0,\frac{1}{2}\right)\\
T^{1+b(2r-1)}  &: r\in\left[\frac{1}{2},1\right] \\
T^{2r} &: r \in(1,\infty)\,\\
\end{cases},
\end{align}
where the constants $C_1,C_2,C_3>0$ do not depend on $\nU,\nX, \delta, M, T$.

\end{proposition}
\begin{remark}
In Theorem A.4. of \cite{nguyen2023random}, different assumptions on $\nU$ were proposed, however in the proof of Theorem 3.5 from \cite{nguyen2023random}, it was shown that $T \leq C_1\nU^{\frac{1}{2r+b}}$ is enough to guarantee the bound of \ref{prop:random-feature-result}.
\end{remark}

\begin{proof}[Proof of Theorem \ref{maintheo}]

Following the error decomposition of \eqref{errordecomp2} and using the bounds of Proposition \ref{prop:taylor-remainder}, \ref{prop:second-term}, \ref{prop:random-feature-result}
we obtain with probability at least $1-\delta$,
\begin{align}
\label{finbound2}
\| G_{\theta_T}- G^*\|_{L_2(\mu_x)}
&\leq \frac{ C_{\nabla G}(B_\tau)B_\tau^2 }{\sqrt{M}}  + T^{-r} + C_2\; \log^3(2/\delta ) \; T^{-r } \;\\
&\leq C\; \log^3(2/\delta ) \; T^{-r } \;.
\end{align}

Collecting the restrictions on the number of neurons and samples from the above Propositions, we obtain the final lower bounds $T \leq C_1\nU^{\frac{1}{2r+b}}$, $2r + b >1$ and $\nU\geq n_0:= e^{\frac{2r+b}{2r+b-1}}$, $\,\nX\geq 32\kappa^2\left(B_\tau+1\right)T^{2r} \left(\log(T)+6\right)\log \left(\frac{16}{\delta}\right)$.
\begin{align}\label{M0final}
M\geq M_0(T,\delta,\tilde{d},r,b,\alpha,B_\tau)&:= \max\left\{M_{II},M_{III}\right\},
\end{align}

with $M_{II}$, $M_{III}$  defined in \ref{prop:second-term} and \ref{eq:assumpt-M-3}.

\end{proof}

\begin{proof}[Proof of Theorem \ref{approxtheo}]
 We prove the existence of some $H_\infty\in\mathcal{H}_\infty$ such that with probability at least $1-\delta$ (over the initialisation),
\begin{align}\label{hjdfsdffg}
\| H_\infty-G_{\theta^*}\|_{L^2_{\mu_u}}\leq \frac{2\kappa \sqrt{\log(2/\delta)}}{M^{\frac{1}{4}}}\left(\|a^{(0)}\|+R\right)+\frac{C_\sigma R^3}{\sqrt{M}}.
\end{align}

The rest of the statement then directly follows from
\begin{align}
\|H_\infty-G^*\|_{L^2_{\mu_u}}&\leq\| H_\infty-G_{\theta^*}\|_{L^2_{\mu_u}} +\|G_{\theta^*}-G^*\|_{L^2_{\mu_u}}\\
&\leq\varepsilon+\frac{2\kappa \sqrt{\log(2/\delta)}}{M^{\frac{1}{4}}}\left(\|a^{(0)}\|+R\right)+\frac{C_\sigma R^3}{\sqrt{M}}.
\end{align}
To prove \eqref{hjdfsdffg}, we once again use a Taylor approximation to obtain
\begin{align*}
G_{\theta^*}(u)(x)&= G_{\theta_0}(u)(x)+\left\langle\nabla G_{\theta_{0}}(u)(x), \theta^*-\theta_{0}\right\rangle_{\Theta}+r_{(\theta^*,\theta_{0})}(u)(x)\\
&=H_M+r_{(\theta^*,\theta_{0})}(u)(x),
\end{align*}
with $H_M:=G_{\theta_0}(u)(x)+\left\langle\nabla G_{\theta_{0}}(u)(x), \theta^*-\theta_{0}\right\rangle_{\Theta}$.
Note that $G_{\theta_0}(u)(x)= \left\langle\nabla G_{\theta_{0}}(u)(x), (a^{(0)},0)\right\rangle_{\Theta}$, where $0\in\mathbb{R}^{M\times \tilde{d}}$ denotes the zero matrix and $(a^{(0)},0)\in \Theta$.
Therefore we have that
$$H_M \in \mathcal{H}_{M}=\left\{ H \in\, \mathcal{F}( L^2(\mathcal{X},\rho_x); L^2(\mathcal{X},\rho_x))\,\mid \,H (u)(x)\,= \, \left\langle\nabla G_{\theta_{0}}(u)(x),\tilde{\theta}\right\rangle_{\Theta} ,\, \tilde{\theta} \in \Theta\right\}.$$ Since $\mathcal{H}_M=ran\left(\mathcal{L}_M^{\frac{1}{2}}\right)$ we have some operator $F\in\, \mathcal{F}( L^2(\mathcal{X},\rho_x); L^2(\mathcal{X},\rho_x))$ with 
$H_M =  \mathcal{L}_M^{\frac{1}{2}}F$. Further Note that 
\begin{align}\label{Fnormbb}
\|F\|_{L^2_{\mu_u}}=\|\mathcal{L}_M^{\frac{1}{2}}F\|_{\mathcal{H}_M}=\|H_M\|_{\mathcal{H}_M}\leq \|a^{(0)}\|+R.
\end{align}

Now we set $H_\infty:=\mathcal{L}_\infty^{\frac{1}{2}}F$ and obtain from Proposition \ref{prop6} and \eqref{Fnormbb}

\begin{align*}
\| H_\infty-G_{\theta^*}\|_{L^2_{\mu_u}}&\leq \| H_\infty-H_M\|_{L^2_{\mu_u}}+\|r_{(\theta^*,\theta_{0})}\|_{L^2_{\mu_u}}\\
&\leq \|H_\infty-H_M\|_{L^2_{\mu_u}}+\frac{C_\sigma R^3}{\sqrt{M}} \\
&\leq \|\mathcal{L}^{\frac{1}{2}}_\infty-\mathcal{L}^{\frac{1}{2}}_M\|\left(\|a^{(0)}\|+R\right)+\frac{C_\sigma R^3}{\sqrt{M}} .
\end{align*}

From Proposition \ref{ineq1}  we have $\|\mathcal{L}^{\frac{1}{2}}_\infty-\mathcal{L}^{\frac{1}{2}}_M\|\leq \|\mathcal{L}_\infty-\mathcal{L}_M\|^{\frac{1}{2}} $
and from \cite{nguyen2023random} (Proposition A.21.) we have with probability at least $1-\delta$,
$$
\left\|\mathcal{L}_\infty-\mathcal{L}_M\right\|_{H S} \leq \frac{4\kappa^2}{\sqrt{M}} \log \frac{2}{\delta}.
$$

Therefore we have
\begin{align*}
\| H_\infty-G_{\theta^*}\|_{L^2_{\mu_u}}&\leq \frac{2\kappa \sqrt{\log(2/\delta)}}{M^{\frac{1}{4}}}\left(\|a^{(0)}\|+R\right)+\frac{C_\sigma R^3}{\sqrt{M}} ,
\end{align*}

what proves the claim.

\end{proof}

\section{Bounds on the Weights}
\label{app:weight-bounds}


\begin{proof}[Proof of Theorem \ref{weighttheo}] 
We prove this by induction over $T \in \mbn$.

Assume that the claim holds for $T \in \mbn$. We show that this implies the claim for $T+1$, for all $T \in \mbn$.

A short calculation, similar to \ref{shortcalc} shows that the weights $(\theta_t)_t$ follow the recursion 
\[ \theta_{T+1} - \theta_0 = \alpha \sum_{t=0}^{T}(Id - \alpha \widehat{\cC}_M)^t ( \zeta^{(1)}_{T-t} + 
 \zeta ^{(2)}_{T-t} + \zeta ^{(3)}_{T-t} + \zeta ^{(4)}_{T-t} + \zeta ^{(5)}_{T-t} )\;, \]
with
\begin{align*}
\zeta^{(1)}_{s} &= \frac{1}{\nU\nX}\sum_{i=1}^{\nU} \sum_{j=1}^{\nX}\left( G_{\theta_s}(u_i)(x_j) - G^*(u_i)(x_j)\right) \left(\nabla G_{\theta_s} (u_i)(x_j) - \nabla G_{\theta_0}(u_i)(x_j )\right) \;, \\
\zeta^{(2)}_{s} &= \frac{1}{\nU\nX}\sum_{i=1}^{\nU} \sum_{j=1}^{\nX} (G^*(u_i) (x_j)- v_i(x_j)) (\nabla G_{\theta_s} (u_i)(x_j) - \nabla G_{\theta_0}(u_i) (x_j))\;, \\
\zeta^{(3)}_{s} &= \widehat{\cZ}^*_M \bar r_{(\theta_ 0 , \theta_{s}) }, \,\quad \quad\quad \quad\quad \quad  (\bar r_{(\theta_ 0 , \theta_{s}) })_{i,j}:= r_{(\theta_ 0 , \theta_{s}) }(u_i)(x_j)\\
\zeta^{(4)}_{s} &=  \widehat{\cZ}^*_M \bv - \cZ_M^* G^* , \quad \quad\quad \quad \bv_{i,j}:= v_i(x_j) \\
\zeta^{(5)}_{s} &=  \cZ_M^* G^* \;. 
\end{align*}
Hence, 
\begin{align*}
||  \theta_{T+1} - \theta_0  ||_\Theta 
&\leq \alpha \sum_{t=0}^{T}\left\|  (Id - \alpha \widehat{\cC}_M)^s\zeta^{(1)}_{T-t} \right\|_\Theta   
  + \alpha  \sum_{t=0}^{T} \left\|  (Id - \alpha \widehat{\cC}_M)^s\zeta^{(2)}_{T-t} \right\|_\Theta   \nonumber \\
& \;\;\; + \alpha   \sum_{t=0}^{T} \left\|   (Id - \alpha \widehat{\cC}_M)^s\zeta^{(3)}_{T-t} \right\|_\Theta    
+   \alpha \sum_{t=0}^{T} \left\|   (Id - \alpha \widehat{\cC}_M)^s\zeta^{(4)}_{T-t} \right\|_\Theta \nonumber \\
& \;\;\; + \alpha \sum_{t=0}^{T} \left\|   (Id - \alpha \widehat{\cC}_M)^s\zeta^{(5)}_{T-t} \right\|_\Theta \;.
\end{align*}

\vspace{0.3cm}

{\bf Bounding $\alpha \sum_{s=0}^{t}\left\|  (Id - \alpha \widehat{\cC}_M)^s\zeta^{(1)}_{t-s} \right\|_\Theta  $.}

First, note that 
\begin{align*}
 \alpha \sum_{t=0}^{T}\left\|  (Id - \alpha \widehat{\cC}_M)^t\zeta^{(1)}_{T-t} \right\|_\Theta 
&\leq \alpha \sum_{t=0}^{T} \left\| (Id - \alpha \widehat{\cC}_M)^{T-t}  \right\| \cdot ||\zeta^{(1)}_{t}||_\Theta \;\\
&\leq \alpha \sum_{t=0}^{T}  ||\zeta^{(1)}_{t}||_\Theta \;.
\end{align*}
Note that we already bounded $\zeta^{(1)}_{t}$ in the proof of Theorem \ref{prop:second-term}. However, to bound the weights, we need a more refined bound. From Proposition \ref{xi1bound} we obtain for $T \leq C_{\alpha,r,\kappa} \nU^{\frac{1}{2r+b}}$ , $\,\nX \geq t^{2\max\{0, \frac{1}{2}-r\}}$ and $M \geq M_{III}$ with $M_{III}$ from Theorem \ref{prop:random-feature-result}, that there exists some $\tilde{C}_{\alpha,\tau,\sigma,r,\kappa}>0$, such that with probability at least $1-\delta$,
\begin{align*}
 &\alpha \sum_{t=0}^{T}  ||\zeta^{(1)}_{t}||_\Theta\\
 &\leq \tilde{C}_{\alpha,\sigma,r,\kappa} \left(\frac{T B_\tau^5+TB_\tau^3 T^{\max\{0, \frac{1}{2}-r\}}}{M}\,+\,\left(\frac{T\log(T)B_\tau^2}{\sqrt{\nU M}} +\frac{T\log(T)B_\tau^2}{\sqrt{\nX M}}\right) \log \left(\frac{12}{\delta} \right) \,+\,\frac{\eta_r(T)B_\tau^2}{\sqrt{M}} \right)\,.
\end{align*}

If further
\begin{align}\label{Mzeta1}
&M\geq M(\zeta^{(1)}):=\max\{M_{III},M_{\alpha,\tau,\sigma,r,\kappa}\},  \\
&\nU \geq  \nU (\zeta^{(1)}):=\max\{C_{\alpha,r,\kappa}^{-1}T^{2r+b},T\log^2(T)\log(12/\delta)\},\\
&\nX \geq  \nX (\zeta^{(1)}):=\max\{T^{\max\{0,1-2r\}},T\log^2(T)\log(12/\delta)\}\\
\end{align}

where $M_{\alpha,\tau,\sigma,r,\kappa}:=C_{\alpha,\sigma,r,\kappa} B_\tau^4 T $ for some $C_{\alpha,\sigma,r,\kappa} >0$, we have
\begin{align*}
 &\alpha \sum_{t=0}^{T}  ||\zeta^{(1)}_{t}||_\Theta \leq B_\tau/5\,.
\end{align*}
\vspace{0.3cm}

{\bf Bounding $\alpha \sum_{t=0}^{T}\left\|  (Id - \alpha \widehat{\cC}_M)^t\zeta^{(2)}_{T-t} \right\|_\Theta  $.} 

From Lemma \ref{xi2bound}, we obtain with probability at least $1-\delta$, 
\[ \max_{t\in [T]}\| \zeta_{t}^{(2)} \|_{\Theta}  \leq \frac{\tilde{C}_{ \sigma}   \,\tilde{d}\, B^2_\tau}{\sqrt{\nU  \cdot M}} \; \;\sqrt{\log(\tilde{d}/\delta)}  \;,\] 
Hence, 
\begin{align}
\label{eq:the-second}
 \alpha \sum_{t=0}^{T}\left\|  (Id - \alpha \widehat{\cC}_M)^t\zeta^{(2)}_{T-t} \right\|_\Theta 
&\leq \alpha \sum_{t=0}^{T} \left\| (Id - \alpha \widehat{\cC}_M)^{T-t}  \right\| \cdot ||\zeta^{(2)}_{t}||_\Theta \nonumber \\
&\leq  \frac{\alpha \tilde{C}_{ \sigma}   \,\tilde{d}\, B^2_\tau T}{\sqrt{\nU  \cdot M}} \; \;\sqrt{\log(\tilde{d}/\delta)}    \nonumber \\
&\leq \frac{1}{5} B_\tau, 
\end{align}

provided that $\nU \geq  \nU (\zeta^{(1)})$, $M \geq M(\zeta^{(2)})$, with  
\begin{align}\label{Mzeta2}
M(\zeta^{(2)}) &= C_{ \alpha,\sigma}   \, B^2_\tau  \,T\,\tilde{d}^2\,\log(\tilde{d}/\delta) \;, 
\end{align}
for some $C_{ \alpha,\sigma}  >0$.

\vspace{0.3cm}

{\bf Bounding $\alpha \sum_{t=0}^{T}\left\|  (Id - \alpha \widehat{\cC}_M)^t\zeta^{(3)}_{T-t} \right\|_\Theta  $.}

Since with probability at least $1-\delta/5$, for all $t \in [T]$, 
\[ \| \theta_t - \theta_0 \|_\Theta \leq B_\tau \;, \]

we obtain from Proposition \ref{sumbound} and \ref{prop6},

\begin{align}
\label{eq:the-third}
\alpha \sum_{t=0}^{T}\left\|  (Id - \alpha \widehat{\cC}_M)^t\zeta^{(3)}_{T-t} \right\|_\Theta
&=  \alpha \sum_{t=0}^{T}\left\|  (Id - \alpha \widehat{\cC}_M)^{T-t}\zeta^{(3)}_{t} \right\|_\Theta \nonumber \\
&\leq  \frac{\sqrt{\alpha}}{\sqrt {\nU\nX}} \sum_{t=0}^{T} \left\| (Id - \alpha \widehat{\cC}_M)^{T-t} (\sqrt{\alpha}\widehat{\cZ}_M^*)\right\| 
\cdot \| \bar r _{(\theta_0 , \theta_s)}\|_F \nonumber \\ 
&\leq   \frac{2\sqrt{\alpha T}}{\sqrt {\nU\nX}} 
\cdot \max_{s \in [T]} \| \bar r _{(\theta_0 , \theta_s)}\|_F \nonumber \\ 
&\leq   2\sqrt{\alpha T} \frac{C_{\nabla G}(B_\tau)}{\sqrt{M}} \max_{s \in [T]}\|\theta_s-\theta_{0}\|^2_{\Theta}\nonumber \\ 
&\leq \frac{1}{5} \;  B_\tau \;,
\end{align}

provided that $M \geq M(\zeta^{(3)})$, with  
\begin{align}\label{Mzeta3}
M(\zeta^{(3)}) &= C_{ \alpha,\sigma}   \, B^4_\tau  T\;, 
\end{align}
for some $C_{ \alpha,\sigma}>0$.
\vspace{0.3cm}

{\bf Bounding $\alpha \sum_{s=0}^{t}\left\|  (Id - \alpha \widehat{\cC}_M)^s\zeta^{(4)}_{t-s} \right\|_\Theta  $.}

\begin{align*}
 \alpha \sum_{t=0}^{T}\left\|  (Id - \alpha \widehat{\cC}_M)^t\zeta^{(4)}_{T-t} \right\|_\Theta 
&= \alpha \sum_{t=0}^{T}\left\|  (Id - \alpha \widehat{\cC}_M)^t ( \widehat{\cZ}^*_M \by - \cZ_M^* G^* )   \right\|_\Theta  \\
&\leq  \alpha \sum_{t=0}^{T}\left\|  (Id - \alpha \widehat{\cC}_M)^t \widehat{\cC}_{M, \lam}^{1/2} \right\| 
\cdot || \widehat{\cC}_{M, \lam}^{-1/2} \cC_{M, \lam}^{1/2} || \nonumber \\ 
& \;\;\;\; \cdot  
\left\|\mathcal{C}_{M, \lambda}^{-1 / 2}\left(\widehat{\mathcal{Z}}_M^* y-\mathcal{Z}_M^* G^*\right) \right\|_\Theta \;.
\end{align*}

For the first term we  obtain from \eqref{Spinatknödel},
\begin{align*}
 \alpha \sum_{t=0}^{T}\left\|  (Id - \alpha \widehat{\cC}_M)^t \widehat{\cC}_{M, \lam}^{1/2} \right\| 
&\leq \sqrt{ \alpha} \sum_{t=0}^{T}\left\|  (Id - \alpha \widehat{\cC}_M)^t (\alpha\widehat{\cC}_{M})^{1/2} \right\| + 
 \alpha \sqrt{\lambda} \sum_{t=0}^{T}\left\|  (Id - \alpha \widehat{\cC}_M)^t\right\| \\
&\leq 2 \; \sqrt{ \alpha T}  + \sqrt{\lam } \cdot (\alpha T) \\
&= 4 \; \sqrt{ \alpha T}\;,
\end{align*}
if we assume that $\lam = 1/(\alpha T)$.

From Proposition \ref{OPboundevents0}, with probability at least $1-\delta$ we have for $\lam = 1/(\alpha T)$,
\[ || \widehat{\cC}_{M, \lam}^{-1/2} \cC_{M, \lam}^{1/2} || \leq 2 \;,\]
if $M\geq C_{\kappa,\|\mathcal{L}_{\infty}\|}\log^2 \frac{1}{\delta}$ and  $\nU,\nX\geq C_{\kappa,\|\mathcal{L}_{\infty}\|}T \log(T)\log^2\left(1/\delta\right)$ for some $C_{\kappa,\|\mathcal{L}_{\infty}\|}>0$.

Applying Proposition \ref{prop:intermediate} for $\lam = 1/(\alpha T)$ gives with probability at least $1-\delta$,

\begin{align*}
&\left\|\mathcal{C}_{M, \lambda}^{-1 / 2}\left(\widehat{\mathcal{Z}}_M^* \mathbf{v}-\mathcal{Z}_M^* G^*\right)\right\|\\
 &\leq 
2 \left(\kappa\sqrt{\alpha T} \left(\frac{1}{\nU}+\frac{1}{\nX}\right)+5\sqrt{\cN_{\mathcal{L}_{\infty}}((\alpha T)^{-1})}\left(\frac{1}{\sqrt{\nU}}+\frac{1}{\sqrt{\nX}}\right)\right) \log ^{3/2}\left(\frac{12}{\delta} \right)\;  \\
&\leq 
\frac{1}{8\sqrt{\alpha T}}+10\sqrt{\cN_{\mathcal{L}_{\infty}}((\alpha T)^{-1})}\left(\frac{1}{\sqrt{\nU}}+\frac{1}{\sqrt{\nX}}\right)\log ^{3/2}\left(\frac{12}{\delta} \right)\; ,
\end{align*}

if $M\geq C_{\alpha,\kappa,\|\mathcal{L}_{\infty}\|}\tilde{d}T\log(T)\log(1/\delta),\,\nU,\nX\geq C_{\alpha,\kappa}T\log^2\left(1/\delta\right)$ for some $C_{\alpha,\kappa,\|\mathcal{L}_{\infty}\|},C_{\alpha,\kappa}>0$.

Collecting all pieces gives with probability at least $1-3\delta$
\begin{align}
\label{eq:the-fourth}
&\alpha \sum_{t=0}^{T}\left\|  (Id - \alpha \widehat{\cC}_M)^t\zeta^{(4)}_{T-t} \right\|_\Theta \\
&\leq 8\sqrt{\alpha T}\left(\frac{1}{8\sqrt{\alpha T}}+10\sqrt{\cN_{\mathcal{L}_{\infty}}((\alpha T)^{-1})}\left(\frac{1}{\sqrt{\nU}}+\frac{1}{\sqrt{\nX}}\right)\log ^{3/2}\left(\frac{12}{\delta} \right)\right)\\
&\leq B_\tau/5  \;,
\end{align}
under the assumptions,

\begin{align}\label{Mzeta4}
&M\geq M(\zeta^{(4)}):=C_{\alpha,\kappa,\|\mathcal{L}_{\infty}\|}\tilde{d}T\log(T)\log^2(1/\delta),  \\
&\nU \geq  \nU (\zeta^{(4)}):=C_{\alpha,\kappa,\|\mathcal{L}_{\infty}\|}T \log(T)\log^2\left(1/\delta\right),\\
&\nX \geq  \nX (\zeta^{(4)}):= C_{\alpha,\kappa,\|\mathcal{L}_{\infty}\|}T \log(T)\log^2\left(1/\delta\right)
\end{align}
for some for some $C_{\alpha,\kappa,\|\mathcal{L}_{\infty}\|}>0$.

\vspace{0.3cm}

{\bf Bounding $\alpha \sum_{s=0}^{t}\left\|  (Id - \alpha \widehat{\cC}_M)^s\zeta^{(5)}_{t-s} \right\|_\Theta  $.}

We have 
\begin{align*}
 \alpha \sum_{t=0}^{T}\left\|  (Id - \alpha \widehat{\cC}_M)^t\zeta^{(5)}_{T-t} \right\|_\Theta 
&= \alpha \sum_{t=0}^{T}\left\|  (Id - \alpha \widehat{\cC}_M)^t \cZ_M^* G^*   \right\|_\Theta  \nonumber \\
&\leq  \alpha \sum_{t=0}^{T}\left\|  (Id - \alpha \widehat{\cC}_M)^t \widehat{\cC}_{M, \lam}  \right\|
\cdot \left\|  \widehat{\cC}_{M, \lam}^{-1}\cC_{M, \lam}   \right\|  \nonumber \\
& \;\;
\cdot\left\|  \cC_{M, \lam} ^{-1} \cZ^*_M \cL_{M, \lam}^{\frac{1}{2}}  \right\|
\cdot \left\| \cL_{M, \lam}^{-\frac{1}{2}}\cL_{\infty, \lam}^{\frac{1}{2}} \right\|  
\cdot \left\|  \cL_{\infty, \lam}^{-\frac{1}{2}}  G^* \right\|_{L^2}  \;.
\end{align*}
For the first term we obtain 
\begin{align*}
\alpha \sum_{t=0}^{T}\left\|  (Id - \alpha \widehat{\cC}_M)^t \widehat{\cC}_{M, \lam}  \right\|_\Theta 
&\leq  \sum_{t=0}^{T}\left\|  (Id - \alpha \widehat{\cC}_M)^t (\alpha \widehat{\cC}_{M}) \right\|_\Theta + 
   \alpha \lam \cdot  \sum_{t=0}^{T}\left\|  (Id - \alpha \widehat{\cC}_M)^t\right\|_\Theta \\
&\leq 2\log(T) + \lam (\alpha T )\\
&\leq 3\log(T) \;,
\end{align*}
if we again let $\lam = 1/(\alpha T)$.


From Proposition \ref{Opbound5} we find with probability at least $1-4\delta$ 
\begin{align}
\left\|\widehat{\mathcal{C}}_{M,\lambda}^{-1} \mathcal{C}_{M, \lambda}^{1}\right\|&\leq  2 \left(\kappa^2\left(\frac{\alpha T}{\nU}+\frac{\alpha T}{\nX}\right)+5\kappa\sqrt{\cN_{\mathcal{L}_{\infty}}((\alpha T)^{-1})}\left(\sqrt{\frac{\alpha T}{\nU}}+\sqrt{\frac{\alpha T}{\nX}}\right)\right) \log ^{3/2}\left(\frac{12}{\delta} \right)\\
&\leq 1+ 10\kappa\sqrt{\cN_{\mathcal{L}_{\infty}}((\alpha T)^{-1})}\left(\sqrt{\frac{\alpha T}{\nU}}+\sqrt{\frac{\alpha T}{\nX}}\right)\log ^{3/2}\left(\frac{12}{\delta} \right)
\end{align}
if $M\geq M(\zeta^{(4)}),\nU \geq  \nU (\zeta^{(4)}),\nX \geq  \nX (\zeta^{(4)})$.

From Proposition \ref{Opbound1} we obtain almost surely 
\[ \|\mathcal{C}_{M,\lambda}^{-1}\mathcal{Z}^*_M\mathcal{L}_{M,\lambda}^{\frac{1}{2}}\|\leq 2 \;. \]

Furthermore, \cite[Proposition A.14]{nguyen2023random} gives with $1-\delta$ 
\[ \| \cL_{M, \lam}^{-\frac{1}{2}}\cL_{\infty, \lam}^{\frac{1}{2}}  \| \leq 2\;. \]

By Assumption \ref{ass:source} and Proposition \ref{rboundpropo},

\[ \left\| \cL_{\infty, \lam}^{-\frac{1}{2}}  G^* \right\|_{L^2} \leq  R\kappa^r(\alpha T)^{(\frac{1}{2}-r)^{+}} .\]

Combining the previous bounds gives for $\lam = 1/(\alpha T)$ with probability at least $1-5\delta$
\begin{align}
\label{eq:the-fith}
 \alpha \sum_{t=0}^{T}\left\|  (Id - \alpha \widehat{\cC}_M)^t\zeta^{(5)}_{T-t} \right\|_\Theta 
&\leq 12\log(T) R\kappa^r(\alpha T)^{(\frac{1}{2}-r)^{+}}\cdot \cB_\delta\left(\frac{1}{\alpha T}\right)  \;,\\
&= \frac{1}{5}C_{R,r,\kappa,\alpha} \log(T) T^{(\frac{1}{2}-r)^{+}}\cdot \cB_\delta\left(\frac{1}{\alpha T}\right)\\
& = \frac{B_\tau}{5}
\end{align}
where we set 
\begin{align}
\label{eq:def-Bdelta}
 \cB_\delta\left(\frac{1}{\alpha T} \right) &:=  1+ 10\kappa\sqrt{\cN_{\mathcal{L}_{\infty}}((\alpha T)^{-1})}\left(\sqrt{\frac{\alpha T}{\nU}}+\sqrt{\frac{\alpha T}{\nX}}\right)\log ^{3/2}\left(\frac{12}{\delta} \right)\;,\\
 C_{R,r,\kappa,\alpha}&:=60R\kappa^r\alpha^{(\frac{1}{2}-r)^+}
\end{align}
\vspace{0.3cm}


{\bf Collecting everything.}
Finally, combining the bounds of all error terms and plugging in $\lam = 1/(\alpha T)$ gives with probability at least $1-10\delta$,  for all $t \in [T+1]$
\begin{align*}
\| \theta_ t  - \theta_0\|_\Theta 
&\leq  \;  B_\tau  
\end{align*}

provided that 
\begin{align}\label{Mnfinbounds}
M &\geq \max \left\{M(\zeta^{(j)}) \; : \; j =1, 2,3, 4 \right\} ,\\
\nU &\geq \max \left\{\nU(\zeta^{(j)}) \; : \; j =1, 4 \right\} ,\\
\nX &  \geq \max \left\{\nX(\zeta^{(j)}) \; : \; j =1, 4 \right\} .
\end{align}

\end{proof}


\vspace{0.4cm}

\begin{proof}[Proof of Corollary \ref{cor:weights}]
Recall that we assume $ T = c \nU^{\frac{1}{2r +b}}$.
First we have to check if in this case the assumption of Theorem \ref{weighttheo}:
$\nU \geq \max \left\{\nU(\zeta^{(j)}) \; : \; j =1,  4 \right\}$ is fullfilled for our choice of $T$. Plugging in the definitions of  $\nU(\zeta^{(j)})$ shows that $\nU$ must be larger than
\begin{align*}
\nU &\geq \max\{C_{\alpha,r,\kappa}^{-1}T^{2r+b},T\log(T)\log(12/\delta),C_{\alpha,\kappa,\|\mathcal{L}_{\infty}\|}T \log(T)\log^2\left(1/\delta\right)\}
\end{align*}
Therefore it is enough if 
$$
\nU\geq C_{\alpha,r,\kappa,\|\mathcal{L}_{\infty}\|}T^{2r+b} = C_{\alpha,r,\kappa,\|\mathcal{L}_{\infty}\|} \nU .
$$

for some constant $C_{\kappa,r,b,\alpha,\|\mathcal{L}_{\infty}\|}>0$ as long as $\nU\geq n_1(\delta)$ is large enough so that 
$$
\left(\log( T_{n_1})^2 T_{n_1} \log^2\left(\frac{1}{\delta}\right)\right) \leq  T_{n_1}^{2r+b}
$$ holds true. Therefore if $c\leq C_{\kappa,r,b,\alpha,\|\mathcal{L}_{\infty}\|}^{-1}$ our choice of $T$ fulfills the assumption of $\nU$. Now we can prove the statement.

\begin{enumerate}
\item  Let $r > \frac{1}{2}$. A short calculation shows that
\begin{align*}
 \cB_\delta\left(\frac{1}{\alpha T} \right) =  1+ 10\kappa\sqrt{\cN_{\mathcal{L}_{\infty}}((\alpha T)^{-1})}\left(\sqrt{\frac{\alpha T}{\nU}}+\sqrt{\frac{\alpha T}{\nX}}\right)\log ^{3/2}\left(\frac{12}{\delta} \right)\\
 \leq 1+ C_{\alpha,b,\kappa}\left(\sqrt{\frac{T^{1+b}}{\nU}}+\sqrt{\frac{T^{1+b}}{\nX}}\right)\log ^{3/2}\left(\frac{12}{\delta} \right)\\
\leq 1+ C_{\alpha,b,\kappa}\left(\nU^{\frac{1-2r}{4r+2b}}+\sqrt{\frac{T^{1+b}}{\nX}}\right)\log ^{3/2}\left(\frac{12}{\delta} \right)\leq 2
\end{align*}
if we let $c\leq 1$ and 
\begin{align*}
\nU &\geq n_2:= \left(2C_{\alpha,b,\kappa}\log ^{3/2}\left(\frac{12}{\delta} \right)\right)^{\frac{4r+2b}{2r-1}},\\
\nX & \geq \left(2C_{\alpha,b,\kappa}\log ^{3/2}\left(\frac{12}{\delta} \right)\right)^{2}T^{1+b}.
\end{align*}
Hence, $B_\tau(\delta , T) \leq C_{R,r,\kappa,\alpha} \log(T) $. 

For $r=\frac{1}{2}$ we similar obtain $\cB_\delta(T) \leq 2$ if we let $c\leq C_{\alpha,\kappa,b,\delta}$ be small enough.

We finally bound the number of neurons that are required for  achieving this rate from Theorem \ref{weighttheo}. 
We have
\begin{align*}
&\max \left\{M(\zeta^{(j)}) \; : \; j =1, 2,3, 4 \right\}\\
&\leq\max\Biggl\{C_3\log(T)T^{2r},\,C_{\alpha,\sigma,r,\kappa} B_\tau^4 T , \,C_{ \alpha,\sigma}   \, B^2_\tau  \,T\,\tilde{d}^2\,\log(\tilde{d}/\delta) ,\, C_{\alpha,\kappa,\|\mathcal{L}_{\infty}\|}\tilde{d}T\log(T)\log^2(1/\delta)\Biggr\}
\end{align*}
Bounding the maximum gives that it is enough if
\[M \geq C_{\kappa, \alpha ,r,\sigma,\|\mathcal{L}_\infty\|}\; \tilde{d}^2 \log^4(T) T^{2r} \log^2(\tilde{d}/\delta)  \;,\]

for some $ C_{\kappa, \alpha ,r,\sigma,\|\mathcal{L}_\infty\|}>0$. Similar we obtain for number of second stage samples that are required,
\begin{align*}
\nX\geq \max\left\{\,T\log(T)^2\log(12/\delta),\, C_{\alpha,\kappa,\|\mathcal{L}_{\infty}\|}T \log(T)\log^2\left(1/\delta\right),\left(2C_{\alpha,b,\kappa}\log ^{3/2}\left(\frac{12}{\delta} \right)\right)^{2}T^{1+b}\right\}
\end{align*}
Bounding the maximum gives that it is enough if
\begin{align*}
\nX\geq C_{\alpha,b,\kappa,\|\mathcal{L}_{\infty}\|} T^{1+b}\log^3(1/\delta),
\end{align*}

if $\nU\geq n_3(\delta)$ is large enough so that 
$$
\log^2( T_{n_3}) T_{n_3}  \leq  T_{n_3}^{1+b}
$$ holds true.

\item If $r < \frac{1}{2}$, then similar as before we have
\begin{align*}
 \cB_\delta (\alpha T) 
& \; \leq C_{\alpha,b,\kappa}T^{\frac{1}{2}-r}\\
\end{align*} 
for some $C_{\alpha,b,\kappa}>0$, if $c\leq \log^{-3}(1/\delta)$ and  $c^{-1}T^{2r+b} \leq  \nX$.
Hence, $B_\tau(T) \leq C_{R,r,\kappa,\alpha,b} \log(T) T^{1-2r}$. 
A lower bound for the number of neurons follows the same way as in the other case:   
\begin{align*}
M &\geq C_{\kappa, \alpha ,r,\sigma,\|\mathcal{L}_\infty\|}\; \tilde{d}^2 \log^4(T) T^{5-8r} \log^2(\tilde{d}/\delta) ,\\
\nX&\geq C_{\alpha,b,\kappa,\|\mathcal{L}_{\infty}\|}T^{2r+b}\log^3(1/\delta).
\end{align*}
To sum up the statement holds true if we set 
\begin{align*}
c&:=\min\left\{1,C_{\kappa,r,b,\alpha,\|\mathcal{L}_{\infty}\|}^{-1}, C_{\alpha,\kappa,b,\delta}\right\},\\
n&\geq n_0:=\max\{n_1,n_2,n_3\}.
\end{align*}

\end{enumerate}
\end{proof}

\subsection{Error bounds of the weights}

\begin{lemma}[\cite{nguyen2023random}]
\label{lem:norm-fstar}
Let
 \begin{align}\label{Fstar}
F^*_t := \cS_M^* \phi_t(\cL_M) G^*  \in \cH_M \;,
\end{align}
where $\phi_t$ denotes the spectral regularization function associated to gradient descent , see e.g. \cite{Muecke2017op.rates}. Then 
\begin{align*}
||F^*_t ||_\mathcal{H_M} \leq C'_{\kappa, \alpha}\;   t^{\max\{0, \frac{1}{2}-r\}} \;, 
\end{align*}

for some $C'_{\kappa, R,\alpha} < \infty$. 
\end{lemma}

\begin{proposition}
\label{xi1bound}
For all $t \in [T]$, let $\|\theta_t-\theta_{0}\|\leq B_\tau$, and $T \leq C_{\alpha,r,\kappa} \nU^{\frac{1}{2r+b}}$ , $\,\nX \geq t^{2\max\{0, \frac{1}{2}-r\}}$ and $M \geq M_{III}$ with $M_{III}$ from Theorem \ref{prop:random-feature-result}. Then there exists some $C_{\alpha,\tau,\sigma,r,\kappa}>0$, such that with probability at least $1-\delta$,
\begin{align*}
 &\alpha \sum_{t=0}^{T}  ||\zeta^{(1)}_{t}||_\Theta\\
 &\leq C_{\alpha,\tau,\sigma,r,\kappa} \left(\frac{T B_\tau^5+TB_\tau^3 T^{\max\{0, \frac{1}{2}-r\}}}{M}\,+\,\left(\frac{T\log(T)B_\tau^2}{\sqrt{\nU M}} +\frac{T\log(T)B_\tau^2}{\sqrt{\nX M}}\right) \log \left(\frac{12}{\delta} \right) \,+\,\frac{\eta_r(T)B_\tau^2}{\sqrt{M}} \right)\,.
\end{align*}
\end{proposition}

\begin{proof}
By using once again a Taylor expansion we obtain  for $t \in [T]$,
\begin{align}\label{startxi1}
\left\|\zeta^{(1)}_{t}\right\| &= \left\|\frac{1}{\nU\nX}\sum_{i=1}^{\nU} \sum_{j=1}^{\nX}\left( G_{\theta_t}(u_i)(x_j) - G^*(u_i)(x_j)\right) \left(\nabla G_{\theta_t} (u_i)(x_j) - \nabla G_{\theta_0}(u_i)(x_j )\right)\right\|_\Theta\\
&\leq \left\|\frac{1}{\nU\nX}\sum_{i=1}^{\nU} \sum_{j=1}^{\nX}r_{(\theta_t,\theta_0)}(u_i)(x_j)\left(\nabla G_{\theta_t} (u_i)(x_j) - \nabla G_{\theta_0}(u_i)(x_j )\right)\right\|_\Theta+\\
&\quad\left\|\frac{1}{\nU\nX}\sum_{i=1}^{\nU} \sum_{j=1}^{\nX}\left( H_t(u_i)(x_j) - F^*_t(u_i)(x_j)\right) \left(\nabla G_{\theta_t} (u_i)(x_j) - \nabla G_{\theta_0}(u_i)(x_j )\right)\right\|_\Theta+\\
&\quad \left\|\frac{1}{\nU\nX}\sum_{i=1}^{\nU} \sum_{j=1}^{\nX}\left( F_{t}^*(u_i)(x_j) - G^*(u_i)(x_j)\right) \left(\nabla G_{\theta_t} (u_i)(x_j) - \nabla G_{\theta_0}(u_i)(x_j )\right)\right\|_\Theta\\[3pt]
&:=I_t+II_t+III_t\,,
\end{align}
where $F^*_t$ was defined in \eqref{Fstar}.

\textbf{Bound $I$ :}
By using Proposition \ref{prop:LipschitzGradient} and \ref{prop6}, we obtain  for $t \in [T]$
\begin{align*}
I_t\leq \frac{C_{\nabla G}(B_\tau)^2}{M}\|\theta_t-\theta_{0}\|^3_{\Theta}\leq\frac{C_{\nabla G}(B_\tau)^2B_\tau^3}{M}
\end{align*}

\textbf{Bound $II$ :}
Note that $F^*_t \in \mathcal{H}_M$, therefore there exist some $\theta_{t}^*$ such that $F_t^*(u)(x)=\langle \nabla G_{\theta_0}(u)(x), \theta_{t}^*\rangle$ and $\|F_t^*\|_{\mathcal{H}_M}=\|\theta_{t}^*\|_\Theta$ (see for example \cite{Ingo} Theorem 4.21). 
Using these equations we obtain

\begin{align}\label{dkfskur0}
II_t\leq \frac{1}{\nU\nX}\sum_{i=1}^{\nU} \sum_{j=1}^{\nX}\left\| \left(\nabla G_{\theta_t} (u_i)(x_j) - \nabla G_{\theta_0}(u_i)(x_j )\right)\left(\nabla G_{\theta_0}(u_i)(x_j )\right)^T\right\| \left(\|\theta_{t}-\theta_{0}\|_\Theta+\|F^*_t\|_{\mathcal{H}_M}\right).
\end{align}
By definition of the operator norm and Proposition \ref{prop:LipschitzGradient}, we have
\begin{align}\label{gijidd}
\nonumber&\left\| \left(\nabla G_{\theta_t} (u_i)(x_j) - \nabla G_{\theta_0}(u_i)(x_j )\right)\left(\nabla G_{\theta_0}(u_i)(x_j )\right)^T\right\|\\
\nonumber&= \max_{k}\left\| \left(\nabla G_{\theta_t} (u_i)(x_j) - \nabla G_{\theta_0}(u_i)(x_j )\right)\left(\nabla G_{\theta_0}(u_i)(x_j )\right)^T e_k\right\|_2\\
\nonumber&\leq\left\| \nabla G_{\theta_t} (u_i)(x_j) - \nabla G_{\theta_0}(u_i)(x_j )\right\|_2 \frac{\kappa}{\sqrt{M}}\\
&\leq \frac{C_{\nabla G}(B_\tau)\kappa}{M}\|\theta_t-\theta_{0}\|_{\Theta},
\end{align}
where $e_k$ denotes the $k.$ unit vector in $\mathbb{R}^P$. Plugging \eqref{gijidd} into \eqref{dkfskur0} and applying Lemma \ref{lem:norm-fstar} leads to
\begin{align}\label{dfsdfjfgg}
II_t\leq \frac{C_{\nabla G}(B_\tau)\kappa}{M}\|\theta_t-\theta_{0}\|_{\Theta} \left(\|\theta_{t}-\theta_{0}\|_{\Theta}+\|F^*_t\|_{\mathcal{H}_M}\right)\leq \frac{C_{\nabla G}(B_\tau)B_\tau\kappa}{M}\left(B_\tau+C'_{\kappa, \alpha}\;   t^{\max\{0, \frac{1}{2}-r\}}\right).
\end{align}

\textbf{Bound $III$ :}

By Proposition \ref{prop:LipschitzGradient} we have
\begin{align}\label{dgfjdsff}
III_t &= \left\|\frac{1}{\nU\nX}\sum_{i=1}^{\nU} \sum_{j=1}^{\nX}\left( F_{t}^*(u_i)(x_j) - G^*(u_i)(x_j)\right) \left(\nabla G_{\theta_t} (u_i)(x_j) - \nabla G_{\theta_0}(u_i)(x_j )\right)\right\|_\Theta\\
&\leq  \frac{C_{\nabla G}(B_\tau)B_\tau}{\sqrt{M}} \cdot\left(III^{(i)}_t+III^{(ii)}_t\right),
\end{align}
where
\begin{align*}
III^{(i)}_t&=\left|\frac{1}{\nU\nX}\sum_{i=1}^{\nU} \sum_{j=1}^{\nX}\left| F_{t}^*(u_i)(x_j) - G^*(u_i)(x_j)\right| -\mathbb{E} \left| F_{t}^*(u)(x) - G^*(u)(x)\right|\right|\,,\\
III^{(ii)}_t&=\mathbb{E} \left| F_{t}^*(u)(x) - G^*(u)(x)\right|\,.
\end{align*}

\textbf{To sum up:}
If we plug the bounds of $I,II,III$ into \eqref{startxi1} we get for some $C_{\alpha,\sigma,r,\kappa}>0$,

\begin{align}\label{dfjhjjjd}
 \alpha \sum_{t=0}^{T}  ||\zeta^{(1)}_{t}||_\Theta\leq C_{\alpha,\sigma,r,\kappa} \left(\frac{T B_\tau^5+TB_\tau^3 T^{\max\{0, \frac{1}{2}-r\}}}{M}\,+\,\frac{C_{\nabla G}(B_\tau)B_\tau}{\sqrt{M}} \cdot\left(\sum_{t=0}^{T} III^{(i)}_t+III^{(ii)}_t\right)\right)\,.
\end{align}

From Proposition \ref{noise} we have for all $t\in[T]$ with $T \leq C_{\alpha,r,\kappa} \nU^{\frac{1}{2r+b}}$ , $\,\nX \geq t^{2\max\{0, \frac{1}{2}-r\}}$ and $M \geq M_{III}$ with $M_{III}$ from Theorem \ref{prop:random-feature-result}, that with probability at least $1-\delta$, 
\begin{align}
III^{(i)}_t&=\left|\frac{1}{\nU\nX}\sum_{i=1}^{\nU} \sum_{j=1}^{\nX}\left| F_{t}^*(u_i)(x_j) - G^*(u_i)(x_j)\right| -\mathbb{E} \left| F_{t}^*(u)(x) - G^*(u)(x)\right|\right|\\
&\leq \tilde{C}_{\alpha,\kappa}\left(\frac{1}{\sqrt{\nU}} +\frac{1}{\sqrt{\nX}}\right) \log^{3/2} \left(\frac{12}{\delta} \right).
\end{align}
Combining this bound with Proposition \ref{conditioning}, we obtain,

\begin{align}\label{gzzztzgu}
\sum_{t=0}^{T}III^{(i)}_t \leq \tilde{C}_{\alpha,\kappa}\left(\frac{T\log(T)}{\sqrt{\nU}} +\frac{T\log(T)}{\sqrt{\nX}}\right) \log\left(\frac{12}{\delta} \right).
\end{align}
With the same assumptions as for $III^{(i)}_t$ we have from \cite[Proposition A.1]{nguyen2023random}  ( with $\lam = (\alpha t)^{-1}$ )
\begin{align}
\label{eq:bound-RF}
III^{(ii)}_t&=\mathbb{E} \left| F_{t}^*(u)(x) - G^*(u)(x)\right| \leq || \cS_M  F^*_t -   G^*  ||_{L^2(\mu_u)} \leq C_{\alpha,r}\;   t^{-r} \;, \\
\sum_{t=0}^{T}III^{(ii)}_t& \leq C_{\alpha,r}\;  \eta_r(T) \;, 
\end{align}
for some $C_{r,\alpha} < \infty$ and where we used Jensen inequality in the first equation.

Plugging \eqref{eq:bound-RF} and \eqref{gzzztzgu} into \eqref{dfjhjjjd}, implies

\begin{align*}
 &\alpha \sum_{t=0}^{T}  ||\zeta^{(1)}_{t}||_\Theta\\
 &\leq C_{\alpha,\sigma,r,\kappa} \left(\frac{T B_\tau^5+TB_\tau^3 T^{\max\{0, \frac{1}{2}-r\}}}{M}\,+\,\left(\frac{T\log(T)B_\tau^2}{\sqrt{\nU M}} +\frac{T\log^{3/2}(T)B_\tau^2}{\sqrt{\nX M}}\right) \log \left(\frac{12}{\delta} \right) \,+\,\frac{\eta_r(T)B_\tau^2}{\sqrt{M}} \right)\,.
\end{align*}

\end{proof}

\begin{proposition}
\label{xi2bound}
For all $t \in [T]$ and some $T \in \mathbb{N}$, let $\|\theta_t-\theta_{0}\|\leq B_\tau$. Then there exists some $\tilde{C}_{ \sigma}>0$, such that with probability at least $1-\delta$,

\begin{align*}
\max_{t\in [T]}\| \zeta_{t}^{(2)} \|_{\Theta}\leq \frac{\tilde{C}_{ \sigma}   \,\tilde{d}\, B^2_\tau}{\sqrt{\nU  \cdot M}} \; \;\sqrt{\log(\tilde{d}/\delta)} .
\end{align*}
\end{proposition}

\begin{proof}[Proof of Lemma \ref{xi2bound}]

We start with the following decomposition for any $t\in [T]$,  
\begin{align}
\nonumber&|| \zeta_{t}^{(2)} ||_{\Theta} = \left\| \frac{1}{\nU\nX}\sum_{i=1}^{\nU} \sum_{j=1}^{\nX} (G^*(u_i) (x_j)- v_i(x_j)) (\nabla G_{\theta_s} (u_i)(x_j) - \nabla G_{\theta_0}(u_i) (x_j))\right\|_\Theta\\
\nonumber&=\Biggl[\frac{1}{M}\sum_{m=1}^M\left(\frac{1}{\nU\nX}\sum_{i=1}^{\nU} \sum_{j=1}^{\nX}(G^*(u_i) (x_j)- v_i(x_j)) \left(\sigma((b_m^{(t)})^\top J(u_i)(x_j))-\sigma((b_m^{(0)})^\top J(u_i)(x_j)))\right)\right)^2+\\
\nonumber&\,\,\frac{1}{M}\sum_{m=1}^M\sum_{k=1}^{\tilde{d}}\Biggl(\frac{1}{\nU\nX}\sum_{i=1}^{\nU} \sum_{j=1}^{\nX}(G^*(u_i) (x_j)- v_i(x_j))\\
\nonumber&\,\,\,\,\,\,\,\,\,\,\,\,\,\,\,\,\,\,\,\,\,\,\,\,\,\,\,\,\,\,\cdot\biggl(a_m^{(t)}\sigma'((b_m^{(t)})^\top J(u_i)(x_j)) J(u_i)(x_j)^{(k)}-a_m^{(0)}\sigma'((b_m^{(0)})^\top J(u_i)(x_j)) J(u_i)(x_j)^{(k)}\biggr)\Biggr)^2\Biggr]^{\frac{1}{2}}\\
&=: \sqrt{I + II}\label{xi2zerleg}
\end{align}
We will now bound $I$ and $II$ separately.

\textbf{$I)$} Using the mean-value theorem for $b_m^{(t)}$ we obtain for some $\tilde{b}_m$ on the line between $b_m^{(t)}$ and $b_m^{(0)}$:
\begin{align*}
I=&\frac{1}{M}\sum_{m=1}^M\left(\frac{1}{\nU\nX}\sum_{i=1}^{\nU} \sum_{j=1}^{\nX}(G^*(u_i) (x_j)- v_i(x_j)) \left(\sigma((b_m^{(t)})^\top J(u_i)(x_j))-\sigma((b_m^{(0)})^\top J(u_i)(x_j)))\right)\right)^2\\
=&\frac{1}{M}\sum_{m=1}^M\left(\frac{1}{\nU\nX}\sum_{i=1}^{\nU} \sum_{j=1}^{\nX}(G^*(u_i) (x_j)- v_i(x_j))\sigma'(\tilde{b}_m^\top J(u_i)(x_j))\left\langle b_m^{(t)}-b_m^{(0)}, \,J(u_i)(x_j)\right\rangle \right)^2\\
\leq &\frac{1}{M}\sum_{m=1}^M\left\|b_m^{(t)}-b_m^{(0)}\right\|^2_2\left\|\frac{1}{\nU\nX}\sum_{i=1}^{\nU} \sum_{j=1}^{\nX}(G^*(u_i) (x_j)- v_i(x_j))\sigma'(\tilde{b}_m^\top J(u_i)(x_j)) J(u_i)(x_j) \right\|^2_2\\
\leq &\frac{1}{\sqrt{M}}\left\|\theta_{t}-\theta_{0}\right\|_\Theta^2 \sup_{\|b\|_2\leq C_{B_\tau}}\left\|\frac{1}{\nU\nX}\sum_{i=1}^{\nU} \sum_{j=1}^{\nX}(G^*(u_i) (x_j)- v_i(x_j))\sigma'(b^\top J(u_i)(x_j)) J(u_i)(x_j) \right\|_2^2\\
\leq &\frac{B_\tau^2}{M} \sum_{k=1}^{\tilde{d}}\sup_{\|b\|_2\leq C_{B_\tau}}\left(\frac{1}{\nU\nX}\sum_{i=1}^{\nU} \sum_{j=1}^{\nX}(G^*(u_i) (x_j)- v_i(x_j))\sigma'(b^\top J(u_i)(x_j)) J(u_i)(x_j)^{(j)} \right)^2,
\end{align*}

with $C_{B_\tau}:=B_\tau+1$. For the penultimate inequality we used that $max_{i\in[M], t\in[T]}\{|a_t^{(i)}-a_0^{(i)}|, \|b_t^{(i)}-b_0^{(i)}\|_2\} \leq B_\tau$ and therefore 
\begin{align}
\max_{i\in[M]}\{|\tilde a^{(i)}|, \|\tilde b^{(i)}\|_2 \} \leq \max_{i\in[M]}\{|a_t^{(i)}-a_0^{(i)}|+|a_0^{(i)}|, \|b_t^{(i)}-b_0^{(i)}\|_2+\|b_0^{(i)}\|_2\}\leq B_\tau+1 .\label{initialugl}
\end{align}

\textbf{$II)$} Similar we obtain with \eqref{initialugl} and the Cauchy-Schwarz inequality

\begin{align}
\nonumber II=&\,\,\frac{1}{M}\sum_{m=1}^M\sum_{k=1}^{\tilde{d}}\Biggl(\frac{1}{\nU\nX}\sum_{i=1}^{\nU} \sum_{j=1}^{\nX}(G^*(u_i) (x_j)- v_i(x_j))\\
\nonumber&\,\,\,\,\,\,\,\,\,\,\,\,\,\,\,\,\,\,\,\,\,\,\,\,\,\,\,\,\,\,\cdot\biggl(a_m^{(t)}\sigma'((b_m^{(t)})^\top J(u_i)(x_j)) J(u_i)(x_j)^{(k)}-a_m^{(0)}\sigma'((b_m^{(0)})^\top J(u_i)(x_j)) J(u_i)(x_j)^{(k)}\biggr)\Biggr)^2\Biggr]^{\frac{1}{2}}\\
\nonumber\leq&\frac{2}{M}\sum_{m=1}^M\sum_{k=1}^{\tilde{d}}\left(\frac{1}{\nU\nX}\sum_{i=1}^{\nU} \sum_{j=1}^{\nX}(G^*(u_i) (x_j)- v_i(x_j))\left(a_m^{(t)}-a_m^{(0)}\right)\sigma'((b_m^{(t)})^\top J(u_i)(x_j))J(u_i)(x_j)^{(k)}\right)^2+\\
\nonumber&\frac{2}{M}\sum_{m=1}^M\sum_{k=1}^{\tilde{d}}\Biggl(\frac{1}{\nU\nX}\sum_{i=1}^{\nU} \sum_{j=1}^{\nX}(G^*(u_i) (x_j)- v_i(x_j))a_m^{(0)}\\
\nonumber&\,\,\,\,\,\,\,\,\,\,\,\,\,\,\,\,\,\,\,\,\,\,\,\,\,\,\,\,\,\,\cdot \left(\sigma'((b_m^{(t)})^\top J(u_i)(x_j))J(u_i)(x_j)^{(k)}-\sigma'((b_m^{(0)})^\top J(u_i)(x_j))J(u_i)(x_j)^{(k)}\right)\Biggr)^2\\
\nonumber\leq&\frac{2B_\tau^2}{M}\sum_{k=1}^{\tilde{d}}\sup_{\|b\|_2\leq C_{B_\tau}}\left(\frac{1}{\nU\nX}\sum_{i=1}^{\nU} \sum_{j=1}^{\nX}(G^*(u_i) (x_j)- v_i(x_j))\sigma'(b^\top J(u_i)(x_j))J(u_i)(x_j)^{(k)}\right)^2+\\
\nonumber&\frac{2}{M}\sum_{m=1}^M\sum_{k=1}^{\tilde{d}}\Biggl(\frac{1}{\nU\nX}\sum_{i=1}^{\nU} \sum_{j=1}^{\nX}(G^*(u_i) (x_j)- v_i(x_j))\\
&\,\,\,\,\,\,\,\,\,\,\,\,\,\,\,\,\,\,\,\,\,\,\,\,\,\,\,\,\,\,\cdot\left(\sigma'((b_m^{(t)})^\top J(u_i)(x_j))J(u_i)(x_j)^{(k)}-\sigma'((b_m^{(0)})^\top J(u_i)(x_j))J(u_i)(x_j)^{(k)}\right)\Biggr)^2. \label{IIsecsum}
\end{align}

For the second sum \eqref{IIsecsum} we use again the mean-value theorem to obtain
\begin{align*}
&\frac{2}{M}\sum_{m=1}^M\sum_{k=1}^{\tilde{d}}\Biggl(\frac{1}{\nU\nX}\sum_{i=1}^{\nU} \sum_{j=1}^{\nX}(G^*(u_i) (x_j)- v_i(x_j))\\
&\,\,\,\,\,\,\,\,\,\,\,\,\,\,\,\,\,\,\,\,\,\,\,\,\,\,\,\,\,\,\cdot\left(\sigma'((b_m^{(t)})^\top J(u_i)(x_j))J(u_i)(x_j)^{(k)}-\sigma'((b_m^{(0)})^\top J(u_i)(x_j))J(u_i)(x_j)^{(k)}\right)\Biggr)^2\\
&\leq \frac{2}{M}\sum_{m=1}^M\sum_{k=1}^{\tilde{d}}\left(\frac{1}{\nU\nX}\sum_{i=1}^{\nU} \sum_{j=1}^{\nX}(G^*(u_i) (x_j)- v_i(x_j))\sigma''(\tilde{b}_m^\top J(u_i)(x_j))\left\langle b_m^{(t)}-b_m^{(0)}, \,J(u_i)(x_j)\right\rangle J(u_i)(x_j)^{(k)}\right)^2\\
&\leq \frac{2}{M}\sum_{m=1}^M \left\|  b_m^{(t)}-b_m^{(0)}\right\|^2_2\sum_{k=1}^{\tilde{d}}\left\|\frac{1}{\nU\nX}\sum_{i=1}^{\nU} \sum_{j=1}^{\nX}(G^*(u_i) (x_j)- v_i(x_j))\sigma''(\tilde{b}_m^\top J(u_i)(x_j))\,J(u_i)(x_j) J(u_i)(x_j)^{(k)}\right\|^2_2\\
&\leq \frac{2B_\tau^2}{M}\sum_{k,l=1}^{\tilde{d}} \sup_{\|b\|_2\leq C_{B_\tau}} \left(\frac{1}{\nU\nX}\sum_{i=1}^{\nU} \sum_{j=1}^{\nX}(G^*(u_i) (x_j)- v_i(x_j))\sigma''(b^\top J(u_i)(x_j)) J(u_i)(x_j)^{(k)}J(u_i)(x_j)^{(l)}\right)^2.
\end{align*}

To sum up, we obtain by plugging the bounds of $I$ and $II$ into \eqref{xi2zerleg},
\begin{align}
\nonumber \max_{t\in [T]}\| \zeta_{t}^{(2)} \|_{\Theta}\leq &\frac{B_\tau}{\sqrt{M}}\Biggl[3\sum_{k=1}^{\tilde{d}}\sup_{\|b\|_2\leq C_{B_\tau}}\left(\frac{1}{\nU\nX}\sum_{i=1}^{\nU} \sum_{j=1}^{\nX}(G^*(u_i) (x_j)- v_i(x_j))\sigma'(b^\top J(u_i)(x_j)) J(u_i)(x_j)^{(k)} \right)^2+\\
&2\sum_{k,l=1}^{\tilde{d}}   \sup_{\|b\|_2\leq C_{B_\tau}} \left(\frac{1}{\nU\nX}\sum_{i=1}^{\nU} \sum_{j=1}^{\nX}(G^*(u_i) (x_j)- v_i(x_j))\sigma''(b^\top J(u_i)(x_j)) J(u_i)(x_j)^{(k)}J(u_i)(x_j)^{(l)}\right)^2\Biggr]^{\frac{1}{2}}.\label{root_I+II}
\end{align}

Setting $z=(u, v) \in \cZ:=\cU \times \cV$ we define for all $k,l \in[\tilde{d}]$, 
\begin{align*}
h_{k,j}(z)&= (G^*(u) (x_j)- v(x_j))J(u)(x_j)^{(k)},\,\\
h_{k,l,j}(z)&= (G^*(u) (x_j)- v(x_j))J(u)(x_j)^{(k)}J(u)(x_j)^{(l)},
\end{align*}
and the following events,
\begin{align*}
E_{j}&:= \left\{ \sup_{\|b\|_2\leq C_{B_\tau}}\left|\frac{1}{\nU\nX}\sum_{i=1}^{\nU} \sum_{j=1}^{\nX} h_{k,j}(z_i)\sigma'(b^\top J(u_i)(x_j)) \right| \leq \frac{24  C_\sigma }{\sqrt{\nU}}\left(2C_{B_\tau}C_\sigma  + \sqrt{\log \frac{2}{\delta}}\right)\right\},\\
E_{j,k}&:= \left\{ \sup_{\|b\|_2\leq C_{B_\tau}} \left|\frac{1}{\nU\nX}\sum_{i=1}^{\nU} \sum_{j=1}^{\nX}h_{k,l,j}(z_i)\sigma''(b^\top J(u_i)(x_j)) \right| \leq \frac{24  C_\sigma }{\sqrt{\nU}}\left(2C_{B_\tau}C_\sigma  + \sqrt{\log \frac{2}{\delta}}\right)\right\}.
\end{align*}
If we can prove that each of the above events holds true with probability at least $1-\delta$ we obtain from 
Proposition \ref{conditioning}  together with \ref{root_I+II} that with probability at least $1-\delta$
\begin{align*}
 \max_{t\in [T]}\| \zeta_{t}^{(2)} \|_{\Theta} &\leq  \frac{ B_\tau}{\sqrt{M}}(\sqrt{3\tilde{d}}  + \sqrt{2\tilde{d}^2} ) \frac{24  C_\sigma }{\sqrt{\nU }}\left(2C_{B_\tau}C_\sigma  + \sqrt{\log \frac{2(\tilde{d}^2+\tilde{d})}{\delta}}\right)  \\
&\leq \frac{\tilde{C}_{ \sigma}   \,\tilde{d}\, B^2_\tau}{\sqrt{\nU \cdot M}} \; \;\sqrt{\log(\tilde{d}/\delta)} ,
\end{align*}
with $\tilde{C}_{\sigma}>0$, depending only on $C_\sigma$. This proves the claim.

To see that the events $E_j, E_{j,k}$ occur with probability at least $1-\delta$, note that $\|h_{k,j}\|_{\infty}, \|h_{k,l,j}\|_{\infty} \leq 2$, 
$$\mathbb{E}_\mu [h_{k,j}(z)\sigma'(b^\top x) ] =  \mathbb{E}_\mu [h_{k,l,j}(z)\sigma''(b^\top x)] = 0 $$
and that $\sigma', \sigma''$ are $C_\sigma$-Lipschitz. The claim now follows from Proposition \ref{unifboundE_{j,k}}.

\vspace{0.3cm}

\end{proof}

\section{Technical Inequalities}


\subsection{Lipschitz Bounds}

In this section we provide some technical inequalities. 
We start proving that the gradient of our neural operator is pointewisely Lipschitz-continuous 
on a ball around $\theta_0$ which we define as 
$$ 
Q_{\theta_0}(R)\coloneqq \left\{\theta=(a,B)\in\Theta:\left\|\theta-\theta_0\right\|_{\Theta}\leq R\right\}.
$$

\begin{proposition}
\label{prop:LipschitzGradient}
Suppose the Assumptions \ref{ass:neurons}, \ref{ass:input},  and  is satisfied. 
We have for any $x\in\mathcal{X},\,u\in\mathcal{U}$ with $\|u\|_\infty\leq 1$, where $\|u\|_{\infty}:= \sup_{x\in supp(\rho_X)}\|u(x)\|_2$ and for any $\theta, \tilde{\theta} \in Q_{\theta_0}(R)$,
\begin{align*}
\left\|\nabla G_{\theta}(u)(x)-\nabla G_{\tilde{\theta}}(u)(x)\right\|_\Theta
\leq \frac{C_{\nabla G}(R)}{\sqrt{M}}\|\theta-\tilde{\theta}\|_\Theta,
\end{align*}
where
\begin{align*}
C_{\nabla G}(R)&\coloneqq  2 \max \left\{C_\sigma,C_\sigma\left(R+\tau \right)\right\}\; , 
\end{align*}
for some $C_\sigma < \infty$.
\end{proposition}

\begin{proof}[Proof of Proposition \ref{prop:LipschitzGradient}]
Recall the definition of our operator class 
\begin{eqnarray*}
\mathcal{F}_{M} &:=\left\{ G_\theta :\mathcal{U} \to \mathcal{V}  \mid \; G_\theta(u) =  
\frac{1}{\sqrt M} \sum_{m=1}^M a_m \sigma\left( \langle b_m, J(u)\rangle \right)\;, \right. \\
& \quad \left. \theta =(a, B) \in \mathbb{R}^M \times \mathbb{R}^{ M \times (d_k+d_y+d_b)}   \right\}\;.
\end{eqnarray*}

 The partial derivatives are given by
\[
\partial_{a_{m}} G_{\theta}(u)(x)=\frac{1}{\sqrt{M}} \sigma\left(\langle b_m,  J(u)(x)\rangle\right), \quad \partial_{b_{m}^j} G_{\theta}(u)(x)=\frac{a_{m}}{\sqrt{M}} \sigma^{\prime}\left(\langle b_m,  J(u)(x)\rangle\right) (J(u)(x))^{(j)}).
\]
Denoting further $\nabla G_{\theta}(u)=\left(\partial_{a} G_{\theta}(u), \partial_{B} G_{\theta}(u)\right)$
\[
\partial_{a} G_{\theta}(u)=\left(\partial_{a_{1}} G_{\theta}(u), \ldots, \partial_{a_{M}} G_{\theta}(u)\right) \in \mathbb{R}^{M}, \quad \partial_{B} G_{\theta}(u)=\left(\partial_{b_{m}^j} G_{\theta}(u)(x)\right)_{m,j} \in \mathbb{R}^{M \times(d_k +d_y+d_b)}
\]
we then have for any $u \in\mathcal{U}$, $x\in supp(\rho_X)$,

\begin{align*}
\left\|\partial_{a} G_{\theta}(u)(x)-\partial_{\tilde{a}} G_{\tilde{\theta}}(u)(x)\right\|_{2}^2&=\sum_{m=1}^{M}\left(\partial_{a_{m}} G_{\theta}(u)(x)-\partial_{\tilde{a}_{m}} G_{\tilde{\theta}}(u)(x)\right)^{2} \\
& =  \frac{1}{M}\sum_{m=1}^M\left(\sigma(\langle b_m,  J(u)(x)\rangle)-\sigma(\langle \tilde b_m, J(u)(x)\rangle )\right)^2\\
& \leq \frac{\|\sigma'\|_\infty^2}{M}\sum_{m=1}^M\left\|b_m-\tilde b_m\right\|_2^2\|J(u)(x)\|_2^2\\
& \leq \frac{C_\sigma^2}{M}\|B-\tilde{B}\|_F^2,
\end{align*}

where we used the assumption $\|J(u)(x)\|_2\leq\|J(u)(x)\|_1\leq 1$.

In addition, 
\begin{align*}
&\left\|\partial_{B} G_{\theta}(u)(x)-\partial_{\tilde{B}} G_{\tilde{\theta}}(u)(x)\right\|_{2}^2=\sum_{m=1}^{M}\sum_{j=1}^{d_k+d_y+d_b} \left(\partial_{b_{m}} G_{\theta}(u)(x)-\partial_{\tilde{b}_{m}} G_{\tilde{\theta}}(u)(x)\right)^{2} \\
& =  \frac{1}{M}\sum_{m=1}^M \sum_{j=1}^{d_k+d_y+d_b} \left(a_m\sigma'(\langle b_m, J(u)(x)\rangle) (J(u)(x))^{(j)})-\tilde{a}_M\sigma'(\langle \tilde b_m, J(u)(x)\rangle )(J(u)(x))^{(j)})\right)^2\\
& \leq (I)+(II)
\end{align*}

where
\[
(I)=\frac{2}{M} \sum_{m=1}^{M}\sum_{j=1}^{d_k+d_y+d_b} \left(a_{m}-\tilde{a}_{m}\right)^{2} \left(\sigma^{\prime}\left\langle b_m,  J(u)(x)\right\rangle\right)^{2} ((J(u)(x))^{(j)})^{2} \leq \frac{2C_\sigma^2}{M}\|a-\tilde{a}\|_{2}^{2}
\]

and where
\begin{align*}
(II)&=\frac{2}{M}\sum_{m=1}^M\sum_{j=1}^{d_k+d_y+d_b} (\tilde{a}_m-a_{0,m})^2\left(\sigma'(\langle b_m,  J(u)(x)\rangle)-\sigma'(\langle \tilde b_m,  J(u)(x)\rangle)\right)^2((J(u)(x))^{(j)}))^2\\
&\,\,\,\,\,\,\,\,\,+\frac{2}{M}\sum_{m=1}^M \sum_{j=1}^{d_k+d_y+d_b} a_{0,m}^2\left(\sigma'(\langle b_m ,J(u)(x)\rangle)-\sigma'(\langle\tilde{b}_m, J(u)(x)\rangle)\right)^2((J(u)(x))^{(j)}))^2\\
&\leq\frac{2R^2C_\sigma^2}{M}\sum_{m=1}^M\sum_{j=1}^{d_k+d_y+d_b} \left(\langle b_m-\tilde{b}_m, J(u)(x)\rangle\right)^2((J(u)(x))^{(j)}))^2\\
&\,\,\,\,\,\,\,\,\,+\frac{2\tau^2C_\sigma^2}{M}\sum_{m=1}^M \sum_{j=1}^{d_k+d_y+d_b} \left(\langle b_m-\tilde{b}_m,  J(u)(x)\rangle\right)^2((J(u)(x))^{(j)}))^2\\
&\leq\frac{2C_\sigma^2(R^2+\tau^2)}{M}\|B-\tilde{B}\|^2_F
\end{align*}

Finally,

\[
\begin{aligned}
\left\|\nabla G_{\theta}(u)(x)-\nabla g_{\tilde{\theta}}(u)(x)\right\|_{\Theta}^{2} &=\left\|\partial_{a} G_{\theta}(u)(x)-\partial_{\tilde{a}} G_{\tilde{\theta}}(u)(x)\right\|_{2}^{2}+\left\|\partial_{B} G_{\theta}(u)(x)-\partial_{\tilde{B}} G_{\tilde{\theta}}(u)(x)\right\|_{F}^{2} \\
& \leq \frac{C_\sigma^2}{M}\|B-\tilde{B}\|_F^2+\frac{2C_\sigma^2}{M}\|a-\tilde{a}\|_{2}^{2}+\frac{162C_\sigma^2(R^2+\tau^2)}{M}\|B-\tilde{B}\|^2_F\\
& \leq \frac{2}{M} \max \left\{C_\sigma^2,C_\sigma^2\left(R^2+\tau^2 \right)\right\}\left(\|a-\tilde{a}\|_{2}^{2}+\|B-\tilde{B}\|_{F}^{2}\right)\\
&\leq \frac{C_{\nabla G}^2}{M}\|\theta-\tilde{\theta}\|_{F}^2.
\end{aligned}
\]
\end{proof}

\vspace{0.3cm}

\begin{proposition}
\label{prop6}
Denote by $r_{(\theta,\bar{\theta})}$ the remainder term of the Taylor expansion from
\begin{align*}
G_{\theta}(u)(x)-G_{\bar{\theta}} (u)(x)&= \left\langle\nabla G_{\bar{\theta}}(u)(x), \theta-\bar{\theta}\right\rangle_{\Theta}+r_{(\theta,\bar{\theta})}(u)(x).
\end{align*}
Suppose Assumption \ref{ass:neurons} holds and $\theta, \bar{\theta} \in Q_{\theta_0}(R)$. 
Then, for all $x\in\cX$, $u \in\mathcal{U}$, we have 
\begin{align*}
&a)\,\,\,\,\,\,\, |r_{(\theta,\bar{\theta})}(u)(x)| \leq \frac{C_{\nabla G}(R)}{\sqrt{M}} \|\theta-\bar{\theta}\|^2_\Theta,\\[10pt]
&b) \,\,\,\,\,\,\, \left\langle\nabla G_{\theta}(u)(x)-\nabla G_{\bar{\theta}}(u)(x), \theta-\bar{\theta}\right\rangle_{\Theta}\leq \frac{C_{\nabla G}(R)}{\sqrt{M}}\|\theta-\bar{\theta}\|^2_{\Theta}.
\end{align*}
where $C_{\nabla G}(R)$  is defined in Proposition \ref{prop:LipschitzGradient}. 
\end{proposition}

\begin{proof}[Proof of Proposition \ref{prop6}]
$a)\,$ Recall that the remainder term $r_{(\theta,\bar{\theta})}$ of the Taylor expansion is given by
\[
r_{(\theta,\bar{\theta})}(u)(x) =\left(\theta-\bar{\theta}\right)^T \nabla^2 G_{\tilde{\theta}}(u)(x)\left(\theta-\bar{\theta}\right).
\]
where $\nabla^2 G_{\tilde{\theta}}(u)(x)$ denotes the Hessian matrix evaluated at some $\tilde{\theta}$ on the line between $\theta$ and $\bar{\theta}$. 
Note that the remainder term can be bounded by the Lipschitz constant $L$ of the gradients i.e. for any $\mathbf{v} \in \Theta$ we have,
\begin{align*}
    \begin{aligned}
 \nabla^2 G_{\tilde{\theta}}(u)(x) \mathbf{v} &= \lim _{h \rightarrow 0} \frac{\nabla G_{\tilde{\theta}+h\mathbf{v}}(u)(x)-\nabla G_{\tilde{\theta}}(u)(x)}{h} \\
 &\leq \lim _{h \rightarrow 0} \frac{\|\nabla G_{\tilde{\theta}+h\mathbf{v}}(x)-\nabla G_{\tilde{\theta}}(u)(x)\|_2}{|h|} \\
 &\leq \lim _{h \rightarrow 0} L \frac{|h|\|\mathbf{v}\|_2}{|h|} \\
 &\leq L\|\mathbf{v}\|.
\end{aligned}
\end{align*}
Proposition \ref{prop:LipschitzGradient} shows that for the set $Q_{\theta_0}(R)$ this Lipschitz constant is given by $L = C_{\nabla G}/\sqrt{M}$.
Therefore by setting $\mathbf{v}= \theta-\bar{\theta}$ we obtain for the remainder term,
\[
|r_{(\theta,\bar{\theta})}(u)(x)| \leq \frac{C_{\nabla G}(R)}{\sqrt{M}} \|\theta-\bar{\theta}\|^2_\Theta.
\]
$b)\,$ Using again Proposition \ref{prop:LipschitzGradient} together with Cauchy - Schwarz inequality, we obtain for the second inequality,
\[\left\langle\nabla G_{\theta}(u)(x)-\nabla G_{\bar{\theta}}(u)(x), \theta-\bar{\theta}\right\rangle_{\Theta}
\leq \frac{C_{\nabla G}(R)}{\sqrt{M}}\|\theta-\bar{\theta}\|^2_{\Theta}.\]
\end{proof}

\vspace{0.3cm}
\begin{proposition}
\label{NNObound}
Suppose Assumption \ref{ass:neurons} holds and $\theta, \bar{\theta} \in Q_{\theta_0}(R)$. Then for any $u \in \mathcal{U},x\in\mathcal{X}$, $t\geq0$ we have
\begin{align*}
\left|G_{\theta} (u)(x)-G_{\bar{\theta}} (u)(x)\right| \leq \kappa\|\theta-\bar{\theta}\|_{\Theta}+ \frac{C_{\nabla G}(R)}{\sqrt{M}}\|\theta-\bar{\theta}\|^2_{\Theta}.
\end{align*}

\end{proposition}

\begin{proof}
From \ref{prop6} we have
\begin{align*}
G_{\theta}(u)(x)-G_{\bar{\theta}} (u)(x)&= \left\langle\nabla G_{\bar{\theta}}(u)(x), \theta-\bar{\theta}\right\rangle_{\Theta}+r_{(\theta,\bar{\theta})}(u)(x)\\
&\leq \kappa\|\theta-\bar{\theta}\|_{\Theta}+ \frac{C_{\nabla G}(R)}{\sqrt{M}}\|\theta-\bar{\theta}\|^2_{\Theta}.
\end{align*}
\end{proof}


\subsection{Uniform Bounds in Hilbert Spaces}

In this section we recall uniform convergence bounds for functions with values in a real separable Hilbert space. 


\begin{proposition}
\label{Rademacher-Mohri}
Let $\mathcal{G}$ be a family of functions mapping from $\mathcal{Z}:=\mathcal{U}\times\mathcal{V}$ to $[-C,C]$. Then, for any $\delta>0$, with probability at least $1-\delta$ over the draw of an i.i.d. sample $S$ of size $\nU$, the following holds

\begin{align*}
\sup_{g\in\mathcal{G}}\left| \frac{1}{\nU} \sum_{i=1}^n g\left(z_i\right)-\mathbb{E}[g(z)] \right|\leq 8C \widehat{\Re}_S(\mathcal{G})+12C \sqrt{\frac{\log \frac{2}{\delta}}{2 \nU}},
\end{align*}

where $$\widehat{\Re}_S(\mathcal{G}):=\mathbb{E}_r\sup_{g\in \mathcal{G}}\frac{1}{\nU}\sum_{i=1}^{\nU}g(x_i)r_i.$$
denotes the empirical Rademacher complexity with $r_i$ being iid rademacher variables. 
\end{proposition}
\begin{proof}
The proof can be found in \cite{Mohri}.
\end{proof}


\begin{proposition}

\label{prop:Rademacher-Lipschitz}
Let $\tilde\sigma$ be Lipschitz continuous with Lipschitz constant $c_\sigma$, $\,h_j:\mathcal{Z}=\mathcal{U}\times\mathcal{V}\rightarrow\mathbb{R}$ and define 
$$
\mathcal{G}_C=\{g:B_{\tilde{d}}(C)\times\mathcal{Z} \rightarrow \mathbb{R}: g(b,z)=\frac{1}{\nX}\sum_{j=1}^{\nX}h_j(z)\tilde\sigma(\langle b,J(u)(x_j) \rangle)\}.
$$
where $B_{\tilde{d}}(C)$ denotes the closed ball of radius $C$ in $\mbr^{\tilde{d}}$. 
Provided the Assumptions \ref{ass:input} and \ref{ass:neurons} we have
$$
\widehat{\Re}_S(\mathcal{G}_C) \leq \frac{CC_hc_\sigma }{\sqrt{\nU}},
$$
with $C_h:=\max_{j\in[\nX]}\|h_j\|_\infty.$
\end{proposition}

\begin{proof}[Proof of Proposition \ref{prop:Rademacher-Lipschitz}]
By definition of $\widehat{\Re}_S(\mathcal{G}_C)$ we have,
$$
\widehat{\Re}_S(\mathcal{G}_C)=\frac{1}{\nU} \mathbb{E}_r\left[\sup _{ \|b\|_2\leq C} \sum_{i=1}^{\nU} g\left(b,z_i\right)r_i\right]=\frac{1}{\nU} {\mathbb{E}}_{r_1, \ldots, r_{\nU-1}}\left[\mathbb{E}_{r_{\nU}}\left[\sup _{ \|a\|_2\leq C} q_{\nU-1}(b)+g\left(b,z_{\nU}\right)r_{\nU}\right]\right],
$$
where $q_{{\nU}-1}(b)=\sum_{i=1}^{{\nU}-1} g\left(b,z_i\right)r_i$. By definition of the supremum, for any $\epsilon>0$, there exist $\|b_{1}\|_2, \|b_{2}\|_2\leq C $ such that

\begin{align*}
& q_{{\nU}-1}\left(b_{1}\right)+g(b_{1},z_{\nU}) \geq(1-\epsilon)\left[\sup _{ \|b\|_2\leq C} q_{{\nU}-1}(b)+g\left(b,z_{\nU}\right)\right] \\
\text { and } & q_{{\nU}-1}\left(b_{2}\right)-g\left(b_{2},z_{\nU}\right) \geq(1-\epsilon)\left[\sup _{ \|b\|_2\leq C} q_{\nU-1}(b)-g\left(b,z_{\nU}\right)\right] .
\end{align*}

Thus by definition of $\mathbb{E}_{r_{\nU}}$,
\begin{align*}
& (1-\epsilon) \mathbb{E}_{r_{\nU}}\left[\sup _{ \|b\|_2\leq C} q_{{\nU}-1}(b)+r_{\nU}g(b,z_{\nU})\right] \\
= & (1-\epsilon)\left[\frac{1}{2}\sup _{ \|b\|_2\leq C}\left[q_{{\nU}-1}(b)+g(b,z_{\nU})\right]+\frac{1}{2}\sup _{\|b\|_2\leq C} \left[q_{{\nU}-1}(b)-g(b,z_{\nU})\right]\right] \\
\leq & \frac{1}{2}\left[q_{{\nU}-1}\left(b_{1}\right)+g(b_{1},z_{\nU})\right]+\frac{1}{2}\left[q_{{\nU}-1}\left(b_{2}\right)-g(b_{2},z_{\nU})\right] .
\end{align*}

Now set  $s^{{\nU}} =\operatorname{sgn}\tilde{g}(z_{\nU},b_1-b_2)$ with $\tilde{g}(z,b):=\frac{1}{\nX}\sum_{j=1}^{\nX}h_j(z)\langle b,J(u)(x_j) \rangle$ (note that $s^{\nU}$ is independent of $r_{\nU}$). Then, the previous inequality implies together with the Lipschitz property:

\begin{align*}
& (1-\epsilon) \mathbb{E}_{r_{\nU}}\left[\sup _{ \|b\|_2\leq C} q_{{\nU}-1}(b)+r_{\nU}g(b,z_{\nU})\right] \\
& \leq \frac{1}{2}\left[q_{{\nU}-1}\left(b_{1}\right)+q_{{\nU}-1}\left(b_{2}\right)+c_\sigma  s^{\nU} \tilde{g}(z_{\nU},b_1-b_2)  \right] \\
& \leq \frac{1}{2} \sup _{ \|b\|_2\leq C}\left[q_{{\nU}-1}(b)+c_\sigma  s^{\nU} \tilde{g}(z_{\nU},b)\right]+\frac{1}{2} \sup _{ \|b\|_2\leq C}\left[q_{{\nU}-1}(b)-c_\sigma  s^{\nU} \tilde{g}(z_{\nU},b)\right] \\
& = \frac{1}{2} \sup _{ \|b\|_2\leq C}\left[q_{{\nU}-1}(b)+c_\sigma  \tilde{g}(z_{\nU},b)\right]+\frac{1}{2} \sup _{ \|b\|_2\leq C}\left[q_{{\nU}-1}(b)-c_\sigma  \tilde{g}(z_{\nU},b)\right] \\
& = \mathbb{E}_{r_n}\left[\sup _{ \|b\|_2\leq C} q_{{\nU}-1}(b)+r_{\nU}c_\sigma  \tilde{g}(z_{\nU},b)\right].
\end{align*}
Therefore we proved
\begin{align}
(1-\epsilon) \mathbb{E}_{r}\left[\sup _{ \|b\|_2\leq C} q_{{\nU}}(b)\right]\leq\mathbb{E}_{r_{\nU}}\left[\sup _{ \|b\|_2\leq C} q_{{\nU}-1}(b)+r_{\nU}c_\sigma  \tilde{g}(z_{\nU},b)\right].\label{radecondition}
\end{align}

Proceeding in the same way for all other $q_i$ with  $i\leq {\nU}-1$ leads to 
\begin{align*}
(1-\epsilon)^{\nU}\frac{1}{{\nU}} \mathbb{E}_r\left[\sup _{ \|b\|_2\leq C} \sum_{i=1}^{\nU} g\left(b,z_i\right)r_i\right]&\leq \frac{1}{{\nU}} \mathbb{E}_r\left[\sup _{ \|b\|_2\leq C} \sum_{i=1}^{\nU} r_ic_\sigma  \tilde{g}(z_{i},b)\right]\\
&=\frac{c_\sigma}{{\nU}} \mathbb{E}_r\left[\sup _{ \|b\|_2\leq C} \left\langle b,\sum_{i=1}^{\nU} r_i  \frac{1}{\nX}\sum_{j=1}^{\nX}h_j(z)J(u)(x_j)\right\rangle \right]\\
&\leq \frac{Cc_\sigma}{{\nU}} \mathbb{E}_r\left[\left\|\sum_{i=1}^{\nU} r_i  \frac{1}{\nX}\sum_{j=1}^{\nX}h_j(z)J(u)(x_j)\right\|_2\right]\\
&\leq \frac{Cc_\sigma}{{\nU}} \sqrt{\sum_{k=1}^{\tilde{d}} \mathbb{E}_r\left(\sum_{i=1}^{\nU} r_i  \frac{1}{\nX}\sum_{j=1}^{\nX}h_j(z)J(u)(x_j)^{(k)}\right)^2}\\
&= \frac{Cc_\sigma}{{\nU}} \sqrt{\sum_{k=1}^{\tilde{d}} \sum_{i=1}^{\nU}  \left(\frac{1}{\nX}\sum_{j=1}^{\nX}h_j(z)J(u)(x_j)^{(k)}\right)^2 }\\
&\leq \frac{CC_hc_\sigma}{{\nU}} \sqrt{\frac{1}{\nX}\sum_{j=1}^{\nX} \sum_{i=1}^{\nU}\sum_{k=1}^{\tilde{d}}  \left(J(u)(x_j)^{(k)}\right)^2 }\\
&\leq \frac{CC_hc_\sigma}{\sqrt{{\nU}}}.
\end{align*}

Therefore we obtain for $\epsilon \rightarrow 0$
$$\frac{1}{{\nU}} \mathbb{E}_r\left[\sup _{ \|b\|_2\leq C} \sum_{i=1}^{\nU} g\left(b,z_i\right)r_i\right]\leq \frac{CC_hc_\sigma}{\sqrt{{\nU}}}.$$

\end{proof}


\begin{proposition}
\label{unifboundE_{j,k}}
Let $\tilde\sigma$ be Lipschitz continuous with Lipschitz constant $c_\sigma$ and  $h_j:\mathcal{Z}=\mathcal{U}\times\mathcal{V}\rightarrow\mathbb{R}$  such that $\mathbb{E} [h_j(z)\tilde\sigma(b^\top J(u)(x)) ]=0$ for any $\|b\|_2\leq C_{B}$ and $\max_{j\in[\nX]}\|h_j\|_\infty\leq c_h$. Then we have with probability at least $1-\delta$,
$$
 \sup_{\|b\|_2\leq C_{B}}\left|\frac{1}{{\nU}\nX}\sum_{i=1}^{\nU}\sum_{j=1}^{\nX}h_j(z_i)\tilde\sigma(b^\top J(u_i)(x_j)) \right|\leq \frac{12C_h\|\tilde{\sigma}\|_\infty }{\sqrt{{\nU}}}\left(C_{B}c_\sigma C_h + \sqrt{\log \frac{2}{\delta}}\right)\,.
$$

\end{proposition}
\begin{proof}
From Proposition \ref{Rademacher-Mohri} we have with probability $1-\delta$
\begin{align}
\sup_{\|b\|_2\leq C_{B}}\left|\frac{1}{{\nU}\nX}\sum_{i=1}^{\nU}\sum_{j=1}^{\nX}h_j(z_i)\tilde\sigma(b^\top J(u_i)(x_j)) \right|\leq 8C_h\|\tilde{\sigma}\|_\infty \widehat{\Re}_S(\mathcal{G}_{C_B})+12C_h\|\tilde{\sigma}\|_\infty \sqrt{\frac{\log \frac{2}{\delta}}{2 {\nU}}}, \label{mohribound}
\end{align}

where $$\widehat{\Re}_S(\mathcal{G}):=\mathbb{E}_r\sup_{g\in \mathcal{G}}\frac{1}{{\nU}}\sum_{i=1}^{\nU}g(x_i)r_i.$$
denotes the empirical Rademacher complexity with $r_i$ being iid rademacher variables and the function class $\mathcal{G}_{C_B}$ is defined as 
$$
\mathcal{G}_{C_B}=\left\{g:B_{\tilde{d}}(C_B)\times\mathcal{Z} \rightarrow \mathbb{R}: g(b,z)=\sum_{j=1}^{\nX}h_j(z_i)\tilde\sigma(b^\top J(u_i)(x_j))\right\}.
$$
From Proposition \ref{prop:Rademacher-Lipschitz} we have 
$$
\widehat{\Re}_S(\mathcal{G}_{C_B}) \leq \frac{C_{B}c_\sigma C_h}{\sqrt{{\nU}}}.
$$
Plugging this inequality into \eqref{mohribound} proves the claim.

\end{proof}

\vspace{0.3cm}


\subsection{Elementary Inequalities}


\begin{lemma}
\label{prop0}
The function $f(x)=(1-x)^jx^r$ for $j,r >0$ has a global maximum 
on $[0,1]$ at $\frac{r}{r+j}$. 
Therefore we have $\sup_{x\in[0,1]}f(x)\leq \left(\frac{r}{r+j}\right)^r$.
\end{lemma}

\begin{proof}[Proof of Lemma \ref{prop0}]
If $f(x)$ is continuous and has an maximum, then $\log f(x)$ will at the same point because $\log x$ is a monotonically increasing function.
\begin{align*} \ln f(x) & = r \log x + j\log(1 - x) \\
\Rightarrow \frac{\operatorname{d} \ln f(x)}{\operatorname{d} x} & = \frac{r}{x} - \frac{j}{1-x} \\
\Rightarrow \frac{\operatorname{d}^2 \ln f(x)}{\operatorname{d} x^2} & = -\frac{r}{x^2} - \frac{j}{(1-x)^2}.
\end{align*}
The second derivative is obviously less than zero everywhere, and the first derivative is equal zero at $\frac{r}{r+j}$. Therefore it is a global maximum.
\end{proof}

\vspace{0.3cm}

\begin{lemma}
\label{prop1}
Let $v\geq0, t \in \mathbb{N}$. Then
\[
\sum_{i=1}^t i^{-v} \leq \eta_v(t)\,\,\,\,\,\, \text{ with } \,\,\eta_v(t)\coloneqq
\begin{cases}
\frac{v}{v-1} &v>1\\
1+\log(t) & v=1\\
\frac{t^{1-v}}{1-v}& v\in[0,1[ 
\end{cases}
\]
\end{lemma}

\begin{proof}[Proof of Lemma \ref{prop1}]
 We will only prove the first case $v>1$. The other cases can be bounded with the same arguments.
 Since
$$
 i^{-v} \leq \int_{i-1}^{i} x^{-v} d x
$$
we have for $v>1$,
$$
\sum_{i=1}^{t} i^{-v} \leq 1+\int_{1}^{t} x^{-v} d x \leq \frac{v}{v-1} .
$$
\end{proof}

\vspace{0.3cm}

\begin{lemma}
\label{lem:calcs}
Let $a\geq b\geq 0$. Then we have
\begin{align*}
\sum_{s=1}^{T-1} \frac{1}{s^a}\frac{1}{(T-s)^b} \leq \frac{2^a}{T^a}\eta_b(T-1) + \frac{2^b}{T^b}\eta_a(T-1).\\
\end{align*}

\end{lemma}

\begin{proof}[Proof of Lemme \ref{lem:calcs}]
For any $t\in [T]$ we have
\begin{align*}
\sum_{s=1}^{T-1}  \frac{1}{s^a}\frac{1}{(T-s)^b} &= \sum_{s=t+1}^{T-1} \frac{1}{s^a}\frac{1}{(T-s)^b} +\sum_{s=1}^t \frac{1}{s^a}\frac{1}{(T-s)^b} \\
&\leq \frac{1}{t^a}\sum_{s=t+1}^{T-1} \frac{1}{(T-s)^b} + \frac{1}{(T-t)^b}\sum_{s=1}^t \frac{1}{s^a}\\
&\leq \frac{1}{t^a}\sum_{s=1}^{T-1} \frac{1}{s^b} + \frac{1}{(T-t)^b}\sum_{s=1}^{T-1} \frac{1}{s^a}\\
&\leq \frac{1}{t^a}\eta_b(T-1) + \frac{1}{(T-t)^b}\eta_a(T-1).\\
\end{align*}
Setting $t=T/2$ the result follows. 
\end{proof}

\vspace{0.3cm}

\begin{proposition}
\label{sumbound}
For $T\in\mathbb{N}$, $a\in\{0,\frac{1}{2}\}$, $\alpha \in(0,\frac{1}{\kappa^2}]$ we have,
\begin{align*}
i)\,\,\,\,& \sum_{s=0}^{T-1} \| (\alpha\widehat{\Sigma}_M )^a (Id - \alpha \widehat{\Sigma}_M)^s\|\leq   \eta_a(T),\\
ii)\, \,&\sum_{s=0}^{T-1} \| (\alpha\widehat{\Sigma}_M )^a (Id - \alpha \widehat{\Sigma}_M)^s   \widehat{\cS}^*_M \|\leq   \eta_{a+\frac{1}{2}}(T)/\sqrt{\alpha},\\
iii)\,\,&  \sum_{s=0}^{T-1} \left\| (Id - \alpha \widehat{\cC}_M)^{s} (\sqrt{\alpha}\widehat{\cZ}_M^*)\right\|\leq \eta_{\frac{1}{2}}(T)=2\sqrt{T}.
\end{align*}
Further we have for any 
$M\geq 8 \,\tilde{d}\,\kappa^2 \beta_\infty T\vee 8\kappa^4\|\mathcal{L}_\infty\|^{-1}\log^2 
\frac{2}{\delta}$ with  $\beta_\infty=\log \frac{4 \kappa^2(\mathcal{N}_{\mathcal{L}_\infty}(T)+1)}{\delta\|\mathcal{L}_\infty\|} $ ,  
$\tilde{d}:=d_k+d_y+d_b$ that with probability at least $1-\delta/2$,
\begin{align*}
iv)\,\,\,&\sum_{s=0}^{T-1} \| (\alpha\widehat{\Sigma}_M )^a (Id - \alpha \widehat{\Sigma}_M)^s \mathcal{Z}_M\|\leq  \frac{2\eta_{a+\frac{1}{2}}(T)}{\sqrt{\alpha}}+\frac{2\eta_{a}(T)}{\sqrt{\alpha T}}, \,\\
\end{align*}
\end{proposition}

\begin{proof}
Note that the spectrum  of $\alpha\widehat{\Sigma}_M $ is contained in $[0,1]$, therefore by  Lemma \ref{prop0} we have for any $r\geq 0$,  
\begin{align}
\label{mandelmuß}
|| (\alpha\widehat{\Sigma}_M )^r (Id - \alpha \widehat{\Sigma}_M)^s || 
&= \sup_{x \in [0,1]} (1-x)^s x^r 
\leq \left(\frac{r}{r+s}\right)^r \;,
\end{align} 
where we use the convention $(0/0)^0 :=1$.\\
From \eqref{mandelmuß} and Lemma \ref{prop1} we further obtain for any $r\in[0,1]$
\begin{align}
\label{Spinatknödel}
\sum_{s=0}^{T-1} || (\alpha\widehat{\Sigma}_M )^r (Id - \alpha \widehat{\Sigma}_M)^s||&\leq  \sum_{s=0}^{T-1}\left(\frac{r}{r+s}\right)^r\\
&\leq  \sum_{s=1}^{T}\left(\frac{1}{s}\right)^r \leq\eta_r(T).
\end{align}

This already proves part $i)$.  For part $ii)$ we have from \eqref{mandelmuß}  Lemma \ref{prop1}
\begin{align*}
&\sum_{s=0}^{T-1} \|  (Id - \alpha \widehat{\Sigma}_M)^s \widehat{\cS}^*_M \|\\
&=\sum_{s=0}^{T-1} \sqrt{\| (\alpha\widehat{\Sigma}_M )^{2a+1} (Id - \alpha \widehat{\Sigma}_M)^{2s} \|/\alpha}\\
&\leq\frac{1}{\sqrt{\alpha}}\sum_{s=0}^{T-1} \sqrt{\left(\frac{2a+1}{2a+1 +2s}\right)^{2a+1}}
\leq\frac{1}{\sqrt{\alpha}}\sum_{s=1}^{T} \left(\frac{1}{s}\right)^{a+\frac{1}{2}} \leq \eta_{a+\frac{1}{2}}(T)/\sqrt{\alpha}.\\
\end{align*}

Similar as in part $ii)$ we have for $iii)$ we have from \eqref{mandelmuß}  Lemma \ref{prop1}

\begin{align*}
&\sum_{s=0}^{T-1} \|  (Id - \alpha \widehat{\cC}_M)^{s} (\sqrt{\alpha}\widehat{\cZ}_M^*) \|\\
&=\sum_{s=0}^{T-1} \sqrt{\|  (Id - \alpha \widehat{\cC}_M)^{2s} (\sqrt{\alpha}\widehat{\cC}_M) \|}\\
&\leq\frac{1}{\sqrt{\alpha}}\sum_{s=0}^{T-1} \sqrt{\left(\frac{1}{1 +2s}\right)}
\leq\sum_{s=0}^{T} \left(\frac{1}{s}\right)^{\frac{1}{2}} \leq \eta_{\frac{1}{2}}(T).\\
\end{align*}

For part $iv)$ we have with probability at least $1-\delta/2$
\begin{align*}
&\sum_{s=0}^{T-1} \| (\alpha\widehat{\Sigma}_M )^a (Id - \alpha \widehat{\Sigma}_M)^s \mathcal{Z}_M\|\\
&\leq\frac{1}{\sqrt{\alpha}}\sum_{s=0}^{T-1} \| (\alpha\widehat{\Sigma}_M )^{a} (Id - \alpha \widehat{\Sigma}_M)^s(\alpha\widehat{\Sigma}_{M,\lambda} )^{\frac{1}{2}}\|\|\widehat{\Sigma}_{M,\lambda}^{-\frac{1}{2}}\Sigma_{M,\lambda}^{\frac{1}{2}}\| \|\Sigma_{M,\lambda}^{-\frac{1}{2}}\mathcal{Z}_M\|\\
&\leq \frac{2\eta_{a+\frac{1}{2}}(T)}{\sqrt{\alpha}}+\frac{2\eta_{a}(T)}{\sqrt{\alpha T}},
\end{align*}
where we used in the second inequality \eqref{Spinatknödel},  \cite[Proposition A.15]{nguyen2023random}  (for  $\lambda:= (\alpha T)^{-1})$ and \ref{Opbound1}.
\end{proof}

\begin{proposition}
\label{conditioning}
Let  $E_i$ be events with probability at least $1-\delta_i$ and set
$$
E:=\bigcap^k_{i=1}E_i
$$
If we can show for some event $A$ that $\mathbb{P}(A|E)\geq 1-\delta$  then we also have
\begin{align*}
\mathbb{P}(A)&\geq\int_{E}\mathbb{P}(A|\omega)d\mathbb{P}(\omega)
\geq (1-\delta)\mathbb{P}(E)\\
&=(1-\delta)\left(1-\mathbb{P}\left(\bigcup_{i=1}^k (\Omega/E_i)\right)\right)\geq(1-\delta)\left(1-\sum_{i=1}^k\delta_i\right).
\end{align*}
\end{proposition}

\section{Operator Bounds and Concentration Inequalities}
\label{app:op-bounds}


\subsection{Operator Bounds}


\begin{proposition}
\label{Opbound1}
\begin{itemize}
\item[]
\item[a)]For any $\lambda>0$ we have
$$
\left\|\widehat{\mathcal{C}}^{-\frac{1}{2}}_{M,\lambda}\widehat{\mathcal{Z}}^*_M\right\|\leq1,\,\,\,\left\|\Sigma^{-\frac{1}{2}}_{M,\lambda}\mathcal{Z}_M\right\|\leq1,\,\,\,
$$
\item[b)]  For any $\lambda>0$ we have
$$
\|\mathcal{C}_{M,\lambda}^{-1}\mathcal{Z}^*_M\mathcal{L}_{M,\lambda}^{\frac{1}{2}}\|\leq 2.
$$

\end{itemize}

\end{proposition}

\begin{proof}
The bound of $b)$ can be found in \cite{rudi2017generalization} (see proof of lemma 3).
For $a)$ we have 
\begin{align*}
\left\|\widehat{\mathcal{C}}^{-\frac{1}{2}}_{M,\lambda}\widehat{\mathcal{Z}}^*_M\right\|^2&=
\left\|(\widehat{\mathcal{Z}}_M^*\widehat{\mathcal{Z}}_M+\lambda)^{-1/2}\widehat{\mathcal{Z}}_M^*\right\|^2\\
&=\left\|(\widehat{\mathcal{Z}}_M^*\widehat{\mathcal{Z}}_M+\lambda)^{-1/2}\widehat{\mathcal{Z}}_M^*\widehat{\mathcal{Z}}_M(\widehat{\mathcal{Z}}_M^*\widehat{\mathcal{Z}}_M+\lambda)^{-1/2}\right\|\\
&=\left\|\widehat{\mathcal{Z}}_M^*\widehat{\mathcal{Z}}_M(\widehat{\mathcal{Z}}_M^*\widehat{\mathcal{Z}}_M+\lambda)^{-1}\right\|\leq1.
\end{align*}
The second inequality follows the same way where we used that $\Sigma_M=\mathcal{Z}_M\mathcal{Z}^*_M$.
\end{proof}


\vspace{0.3cm}

\begin{proposition}
\label{OPboundnonconc}
Let $\mathcal{H}$ be a separable Hilbert space and let $A$ and $B$  be two bounded self-adjoint positive linear operators on $\mathcal{H}$ and $\lambda>0$. Then

$$
\left\|A_\lambda^{-\frac{1}{2}}B_\lambda ^{\frac{1}{2}}\right\| \leq(1-c)^{-\frac{1}{2}}, \quad \left\|A_\lambda ^{\frac{1}{2}}B_\lambda ^{-\frac{1}{2}}\right\| \leq (1+c)^{\frac{1}{2}}
$$
with
$$
c=\left\|B_\lambda ^{-\frac{1}{2}}(A-B)B_\lambda ^{-\frac{1}{2}}\right\|.
$$
\end{proposition}

\vspace{0.3cm}

\begin{proposition}
\label{rboundpropo} 
For $r\geq 0$ we have
\[ \left\| \cL_{\infty, \lam}^{-\frac{1}{2}}\cL_{\infty}^{r}    \right\| \leq \kappa^r\lambda^{-(\frac{1}{2}-r)^{+}}. \]

\end{proposition}

\begin{proof}
For $r\geq 1/2$ we have $\left\| \cL_{\infty, \lam}^{-\frac{1}{2}}\cL_{\infty}^{r}    \right\|\leq \kappa^{r-1/2}$.
For $r< 1/2$ we have
\begin{align*}
\left\| \cL_{\infty, \lam}^{-\frac{1}{2}}\cL_{\infty}^{r}    \right\|=\max_{i \in \mathbb{N}} \left|\frac{\mu_i^r}{(\mu_i+\lambda)^{1/2}}\right|\leq\lambda^{r-1/2}
\end{align*}
since
$\begin{aligned}
\left|\frac{\mu_i^r}{(\mu_i+\lambda)^{1/2}}\right|& \leq 
\begin{cases}
&\lambda^{r-1/2} \quad\quad\quad \lambda\geq \mu_i \\
&\mu_i^{r-1/2} \quad\quad\quad \mu_i\geq \lambda
\end{cases}\\
&\leq \lambda^{r-1/2}.
\end{aligned}$
\end{proof}

\begin{proposition}[ \cite{aleksandrov2009operatorholderzygmundfunctions}, \cite{Muecke2017op.rates} (Proposition B.1.) ]
\label{ineq1}
Let $B_{1}, B_{2}$ be two non-negative self-adjoint operators on some Hilbert space with $\left\|B_{j}\right\| \leq a, j=1,2$, for some non-negative a.
\begin{itemize}
\item[(i)] If $0 \leq r \leq 1$, then
$$
\left\|B_{1}^{r}-B_{2}^{r}\right\| \leq \left\|B_{1}-B_{2}\right\|^{r},
$$
for some $C_{r}<\infty$.
\item[(ii)] If $r>1$, then
$$
\left\|B_{1}^{r}-B_{2}^{r}\right\| \leq C_{a, r}\left\|B_{1}-B_{2}\right\|,
$$
for some $C_{a, r}<\infty$. 
\end{itemize}
\end{proposition}

\subsection{Concentration Inequalities}
\label{app:conc-inequ}

In this subsection we generalize Bernstein type inequalities for double sums and apply them to our operators from section \ref{prelimss}. The first inequality was first established for matrices by  \cite{Tropp_2011}. For the general case including operators the proof can for example be found in  \cite{spectral.rates} (see Lemma 26):

\vspace{0.3cm}
\begin{proposition}
\label{OPbound0}
Let $\xi_1, \cdots, \xi_m$ be a sequence of independently and identically distributed selfadjoint Hilbert-Schmidt operators on a separable Hilbert space. Assume that $\left\|\xi_1\right\| \leq B/2$ almost surely for some $B>0$. Let $\mathcal{V}$ be a positive trace-class operator such that $\mathbb{E}\left[\xi_1^2\right] \preccurlyeq \mathcal{V}$. Then with probability at least $1-\delta,(\delta \in] 0,1[)$, there holds
$$
\left\|\frac{1}{m} \sum_{i=1}^m \xi_i-\mathbb{E}[\xi]\right\| \leq \frac{2 B \beta}{3 m}+\sqrt{\frac{2\|\mathcal{V}\| \beta}{m}}, \quad \beta=\log \frac{4 \operatorname{tr} \mathcal{V}}{\|\mathcal{V}\| \delta}\,\,.
$$
\end{proposition}

\vspace{0.3cm}

\begin{corollary}
\label{OPboundcorro0}
Let $X_1, \cdots, X_m$ and $Y_1,\dots Y_n$ be independent sequences of independently and identically distributed random variables such that $f(X_i,Y_j)$ defines a  selfadjoint Hilbert-Schmidt operator on a separable Hilbert space $\cH$, for some function $f:\mathcal{X}\times\mathcal{Y}\rightarrow \mathcal{F}(\cH,\cH)$. Assume that $\left\|f(X_1,Y_1)\right\| \leq B/2$ almost surely for some $B>0$. Let for any $(x,y)\in\mathcal{X}\times\mathcal{Y}$,  $f(x,y)^2$  and $\mathcal{V}$ be positive trace-class operators such that $\mathbb{E}\left[f(X,Y)^2\right] \preccurlyeq \mathcal{V}$ . Then with probability at least $1-2\delta$, there holds
$$
\left\|\frac{1}{mn} \sum_{i=1}^m\sum_{j=1}^n f(X_i,Y_j)-\mathbb{E}f(X,Y)\right\| \leq \frac{2 B \tilde{\beta}}{3 n}+\frac{2 B \beta}{3 m}+\sqrt{\frac{2B\|\tilde{\mathcal{V}}\| \tilde{\beta}}{n}}+\sqrt{\frac{2\|\mathcal{V}\| \beta}{m}}\,,
$$
with $\beta:=\log \frac{4 \operatorname{tr} \mathcal{V}}{\|\mathcal{V}\| \delta},\,\,\tilde{\beta}:=\log \frac{4 \operatorname{tr} \tilde{\mathcal{V}}}{\|\tilde{\mathcal{V}}\| \delta}$ and $\tilde{\mathcal{V}}:=\frac{1}{m}\sum_{i=1}^m\mathbb{E}_Y f(X_i,Y)$ .
\end{corollary}

\begin{proof}
We start by proving that the following events hold true with probability $1-\delta$,
\begin{align}
A&:=\left\{\left\|\frac{1}{mn} \sum_{i=1}^m\sum_{j=1}^n f(X_i,Y_j)-\mathbb{E}_Yf(X_i,Y)\right\|\leq\frac{2 B \tilde{\beta}}{3 n}+\sqrt{\frac{2B\|\tilde{\mathcal{V}}\| \tilde{\beta}}{n}}\right\},\\
B&:=\left\{\left\|\frac{1}{m} \sum_{i=1}^m \mathbb{E}_Yf(X_i,Y)-\mathbb{E}f(X,Y)\right\| \leq\frac{2 B \beta}{3 m}+\sqrt{\frac{2\|\mathcal{V}\| \beta}{m}}\right\}\,.
\end{align}
\vspace{0.3cm}

\textbf{Bounding A:}
Using that $X_i,Y_j$ are independent, we can condition on $\{X_i=x_i\}_{i=1}^{m}$ as follow, 
\begin{align}\label{condprobeq}
\mathbb{P}\left(A\right)= \mathbb{E}_{\mu_{x}^{\otimes m}}\left[\mathbb{P}\left(A|x_1,\dots,x_{m}\right)\right].
\end{align}
We set $\xi_j:=\frac{1}{m}\sum_{i=1}^m f(x_i,Y_j)$. Note that conditioned on $\{X_i=x_i\}_{i=1}^{m}$, where $x_i$ is contained in the support of $\mu_X$, we have that $\xi_j$ are iid w.r.t. $\mu_Y$. By assumption we have $\|\xi_j\|\leq B/2$ and for the second moment we have by Jensen inequality and by assumption,  $\mathbb{E}_Y[\xi^2]\preccurlyeq \frac{1}{m}\sum_{i=1}^m\mathbb{E}_Y f(x_i,Y)^2\preccurlyeq B\frac{1}{m}\sum_{i=1}^m\mathbb{E}_Y f(x_i,Y):=B\tilde{\mathcal{V}}$.

From Proposition \ref{OPbound0} we therefore have, that with probability (over $\mu_Y^{\otimes n}$) at least $1-\delta$,
\begin{align*}
\left\|\frac{1}{mn} \sum_{i=1}^m\sum_{j=1}^n f(x_i,Y_j)-\mathbb{E}_Yf(x_i,Y)\right\|\leq\frac{2 B \tilde{\beta}}{3 n}+\sqrt{\frac{2B\|\tilde{\mathcal{V}}\| \tilde{\beta}}{n}}\,. 
\end{align*}

From \eqref{condprobeq} we therefore have that $P(A)\geq 1-\delta$.\\

\vspace{0.2cm}

\textbf{Bounding B:}
Again we set $\xi_i= \mathbb{E}_Y[f(X_i,Y)]$.
By assumption we have $\|\xi\|\leq B/2$ and  $\mathbb{E}_X \xi^2 \preccurlyeq \mathcal{V}$.
From Proposition \ref{OPbound0} we therefore have, that with probability at least $1-\delta$,
the event $B$ holds true.

\vspace{0.2cm}
\textbf{To sum up:}
Using the bounds of event $A$ and $B$, we therefore have with probability at least $1-2\delta$,
\begin{align}
\left\|\frac{1}{mn} \sum_{i=1}^m\sum_{j=1}^n f(X_i,Y_j)-\mathbb{E}f(X,Y)\right\|_{\cH} 
&\leq \left\|\frac{1}{mn} \sum_{i=1}^m\sum_{j=1}^n f(X_i,Y_j)-\mathbb{E}_Yf(X_i,Y)\right\| \\
&\,\,\,\,\,\,\,\,\,\,+\left\|\frac{1}{m} \sum_{i=1}^m \mathbb{E}_Yf(X_i,Y)-\mathbb{E}f(X,Y)\right\| \\[4pt]
&\leq \frac{2 B \tilde{\beta}}{3 n}+\frac{2 B \beta}{3 m}+\sqrt{\frac{2B\|\tilde{\mathcal{V}}\| \tilde{\beta}}{n}}+\sqrt{\frac{2\|\mathcal{V}\| \beta}{m}}\,.
\end{align}

This proves the claim.
\end{proof}

\vspace{0.3cm}

The following concentration result for Hilbert space valued random variables can be found in \cite{Caponetto}.
\vspace{0.2cm}

\begin{proposition}[Bernstein Inequality]
\label{concentrationineq0}
Let $W_{1}, \cdots, W_{n}$ be i.i.d random variables in a separable Hilbert space $\cH$ with norm $\|\cdot\|_{\cH}$. 
Suppose that there are two positive constants $B$ and $V$ such that
\begin{align}
\label{cons}
\mathbb{E}\left[\left\|W_{1}-\mathbb{E}\left[W_{1}\right]\right\|_{\cH}^{l}\right] \leq \frac{1}{2} l ! B^{l-2} V^{2}, 
\quad \forall l \geq 2 \;. 
\end{align}
Then for any $\delta \in (0,1]$, the following holds with probability at least $1-\delta$:
$$
\left\|\frac{1}{n} \sum_{k=1}^{n} W_{k}-\mathbb{E}\left[W_{1}\right]\right\|_{\cH} 
\leq 2\left(\frac{B}{n}+\frac{V}{\sqrt{n}}\right) \log \left(\frac{4}{\delta} \right)\;  .
$$
In particular, \eqref{cons} holds if
$$
\left\|W_{1}\right\|_{\cH} \leq B / 2 \quad \text { a.s., } \quad \text { and } 
\quad \mathbb{E}\left[\left\|W_{1}\right\|_{\cH}^{2}\right] \leq V^{2} \;.
$$
\end{proposition}

\vspace{0.3cm}

\begin{corollary}
\label{OPboundcorro}
Let $X_1, \cdots, X_m$ and $Y_1,\dots Y_n$ be independent sequences of independently and identically distributed random variables such that $f(X_i,Y_j)$  is contained in a separable Hilbert space $\cH$ with norm $\|\cdot\|_{\cH}$, for some function $f:\mathcal{X}\times\mathcal{Y}\to\mathcal{H}$.  Suppose that there are two positive constants $B$ and $V$ such that
$$
\left\|f(X_1,Y_1)\right\|_{\cH} \leq B / 2 \quad \text { a.s., } \quad \text { and } 
\quad \mathbb{E}\left[\left\|f(X_1,Y_1)\right\|_{\cH}^{2}\right] \leq V^{2} \;.
$$
Then for any $\delta \in (0,1]$, the following holds with probability at least $1-\delta$:
$$
\left\|\frac{1}{mn} \sum_{i=1}^m\sum_{j=1}^n f(X_i,Y_j)-\mathbb{E}f(X,Y)\right\|_{\cH} 
\leq 2 \left(\frac{B}{m}+\frac{V}{\sqrt{m}} + \frac{B}{n}+\sqrt{\frac{Z}{n}}\right) \log \left(\frac{12}{\delta} \right)\;  ,
$$
with $Z:=\left(\frac{B^2}{m}+\frac{BV}{\sqrt{m}}\right) \log \left(\frac{12}{\delta} \right)+ V^2$.
\end{corollary}

\begin{proof}
We start by proving that the following events hold true with high probability,
\begin{align}
A&:=\left\{\left|\frac{1}{m}\sum_{i=1}^m \mathbb{E}_Y[\|f(X_i,Y)\|_{\cH}^2]-\mathbb{E}[\|f(X,Y)\|_{\cH}^2]\right|\leq  \left(\frac{B^2}{m}+\frac{BV}{\sqrt{m}}\right) \log \left(\frac{4}{\delta} \right)\right\},\\
B&:=\left\{\left\|\frac{1}{mn} \sum_{i=1}^m\sum_{j=1}^n f(X_i,Y_j)-\mathbb{E}_Yf(X_i,Y)\right\|_{\cH}\leq2 \left(\frac{B}{n}+\sqrt{\frac{\frac{1}{m}\sum_{i=1}^m \mathbb{E}_Y[\|f(X_i,Y)\|_{\cH}^2]}{n}}\right) \log \left(\frac{4}{\delta} \right)\right\},\\
C&:=\left\{\left\|\frac{1}{mn} \sum_{i=1}^m\sum_{j=1}^n f(X_i,Y_j)-\mathbb{E}_Yf(X_i,Y)\right\|_{\cH}\leq2 \left(\frac{B}{n}+\sqrt{\frac{Z}{n}}\right) \log \left(\frac{4}{\delta} \right)\right\},\\
D&:=\left\{\left\|\frac{1}{m} \sum_{i=1}^m \mathbb{E}_Yf(X_i,Y)-\mathbb{E}f(X,Y)\right\|_{\cH} \leq2 \left(\frac{B}{m}+\frac{V}{\sqrt{m}}\right) \log \left(\frac{4}{\delta} \right)\right\},
\end{align}
with $Z:=\left(\frac{B^2}{m}+\frac{BV}{\sqrt{m}}\right) \log \left(\frac{4}{\delta} \right)+ V^2$.\\
\vspace{0.3cm}

\textbf{Bounding A:}
Set $W_i= \mathbb{E}_Y[\|f(X_i,Y)\|_{\cH}^2]$.  Then by assumption $|W_i|\leq B^2/4$ and $\mathbb{E}[|W_i|^2]\leq\frac{B^2V^2}{4} $. From Proposition \ref{concentrationineq0} we therefore have with probability at least $1-\delta$,
$$\left|\frac{1}{m}\sum_{i=1}^m \mathbb{E}_Y[\|f(X_i,Y)\|_{\cH}^2]-\mathbb{E}[\|f(X,Y)\|_{\cH}^2]\right|\leq  \left(\frac{B^2}{m}+\frac{BV}{\sqrt{m}}\right) \log \left(\frac{4}{\delta} \right).$$

\vspace{0.2cm}

\textbf{Bounding B:}
Using that $X_i,Y_j$ are independent, we can condition on $\{X_i=x_i\}_{i=1}^{m}$ as follow, 
\begin{align}\label{condprobeq2}
\mathbb{P}\left(B\right)= \mathbb{E}_{\mu_{x}^{\otimes m}}\left[\mathbb{P}\left(B|x_1,\dots,x_{m}\right)\right].
\end{align}
We set $W_j:=\frac{1}{m}\sum_{i=1}^m f(x_i,Y_j)$. Note that conditioned on $\{X_i=x_i\}_{i=1}^{m}$ we have that $W_j$ are iid w.r.t. $\mu_Y$.
By assumption we have $\|W_j\|_\cH\leq B/2$ and for the second moment we have
\begin{align*}
\mathbb{E}_Y[\|W\|_{\cH}^2]&\leq \frac{1}{m}\sum_{i=1}^m \mathbb{E}_Y[\|f(x_i,Y)\|_{\cH}^2].
\end{align*}

From Proposition \ref{concentrationineq0} we therefore have, that with probability (over $\mu_y^{\otimes n}$) at least $1-\delta$,
\begin{align*}
\left\|\frac{1}{mn} \sum_{i=1}^m\sum_{j=1}^n f(x_i,Y_j)-\mathbb{E}_Yf(x_i,Y)\right\|_{\cH}\leq 2 \left(\frac{B}{n}+\sqrt{\frac{\frac{1}{m}\sum_{i=1}^m \mathbb{E}_Y[\|f(x_i,Y)\|_{\cH}^2]}{n}}\right) \log \left(\frac{4}{\delta} \right).
\end{align*}

From \eqref{condprobeq2} we therefore have that $P(B)\geq 1-\delta$.\\

\vspace{0.2cm}
\textbf{Bounding C:} Note that $P(C|A,B)=1$, since conditioned on $A$ we have by assumption,
\begin{align*}
\frac{1}{m}\sum_{i=1}^m \mathbb{E}_Y[\|f(X_i,Y)\|_{\cH}^2]&\leq  \left(\frac{B^2}{m}+\frac{BV}{\sqrt{m}}\right) \log \left(\frac{4}{\delta} \right) +\mathbb{E}[\|f(X,Y)\|_{\cH}^2]\\
&\leq  \left(\frac{B^2}{m}+\frac{BV}{\sqrt{m}}\right) \log \left(\frac{4}{\delta} \right)+ V^2,
\end{align*}
and plugging this inequality into the bound of $B$, leads to the bound of $C$. By Proposition \ref{conditioning} we therefore have, $P(C)\geq 1-2\delta$.\\

\vspace{0.2cm}
\textbf{Bounding D:}
Again we set $W_i= \mathbb{E}_Y[\|f(X_i,Y)\|_{\cH}]$ and obtain by Proposition \ref{concentrationineq0} with probability at least $1-\delta$,
\begin{align*}
\left\|\frac{1}{m} \sum_{i=1}^m \mathbb{E}_Yf(X_i,Y)-\mathbb{E}f(X,Y)\right\|_{\cH} \leq 2 \left(\frac{B}{m}+\frac{V}{\sqrt{m}}\right) \log \left(\frac{4}{\delta} \right).
\end{align*}

\vspace{0.2cm}
\textbf{To sum up:}
Using the bounds of event $C$ and $D$, we therefore have with probability at least $1-3\delta$,
\begin{align}
\left\|\frac{1}{mn} \sum_{i=1}^m\sum_{j=1}^n f(X_i,Y_j)-\mathbb{E}f(X,Y)\right\|_{\cH} 
&\leq \left\|\frac{1}{mn} \sum_{i=1}^m\sum_{j=1}^n f(X_i,Y_j)-\mathbb{E}_Yf(X_i,Y)\right\|_{\cH} \\
&\,\,\,\,\,\,\,\,\,\,+\left\|\frac{1}{m} \sum_{i=1}^m \mathbb{E}_Yf(X_i,Y)-\mathbb{E}f(X,Y)\right\|_{\cH} \\
&\leq 2 \left(\frac{B}{m}+\frac{V}{\sqrt{m}} + \frac{B}{n}+\sqrt{\frac{Z}{n}}\right) \log \left(\frac{4}{\delta} \right).
\end{align}

This proves the claim.
\end{proof}

\vspace{0.3cm}

\begin{proposition}[\cite{rudi2017generalization} (Lemma 9)]
\label{CMbound}
For any $M\geq8\kappa^4\|\mathcal{L}_\infty\|^{-1}\log^2 \frac{2}{\delta}$  we have with probability at least $1-\delta$
$$
\|\mathcal{L}_M\|\geq\frac{1}{2}\|\mathcal{L}_\infty\|.
$$
\end{proposition}

\begin{proposition}[\cite{nguyen2023random} Proposition A.18.]
\label{Opbound3}
For any $\lam>0$ with $M\geq \frac{8 \tilde{d}\kappa^2 C_{\lambda,\delta,\kappa}}{\lam}$ with $C_{\lambda,\delta,\kappa}:=\log \frac{80\kappa^4}{\|\mathcal{L}_{\infty}\| \lambda\delta}$
we have with probability at least $1-\delta$
$$
\cN_{\cL_{M}}(\lam)
\leq  4\left(1+2\log\frac{2}{\delta}\right)\mathcal{N}_{\cL_{\infty}}(\lam).
$$
\end{proposition}

\vspace{0.3cm}

\begin{proposition}
\label{OPboundevents0}
For any $\lambda\in(0,1]$, any $M\geq8\kappa^4\|\mathcal{L}_\infty\|^{-1}\log^2 \frac{2}{\delta}$ and $\nU,\,\nX\geq \frac{72\kappa^4C_{\lambda,\delta,\kappa}(1+\|\mathcal{L}_\infty\|)}{\lambda}$ with $C_{\lambda,\delta,\kappa}:=\log \frac{80\kappa^4}{\|\mathcal{L}_{\infty}\| \lambda\delta}$, define the following events,

\begin{align}
E_1=\left\{\left\|\mathcal{C}_{M,\lambda}^{-\frac{1}{2}}\left(\widehat{\mathcal{C}}_{M}-\mathcal{C}_{M}\right) \mathcal{C}_{M,\lambda}^{-\frac{1}{2}}\right\|\leq   3/4  \right\},  \\[7pt]
E_2=\left\{\left\|\widehat{\mathcal{C}}_{M,\lambda}^{-\frac{1}{2}}\mathcal{C}_{M, \lambda}^{\frac{1}{2}}\right\| \leq 2, \quad  \left\|\widehat{\mathcal{C}}_{M,\lambda}^{\frac{1}{2}}\mathcal{C}_{M, \lambda}^{-\frac{1}{2}}\right\| \leq 2 \right\}.
\end{align}
Providing Assumption \ref{ass:input}  we have that both events holds true with probability at least $1-3\delta$.
\end{proposition}

\begin{proof}
$E_1)$ Set $f(u,x):= \mathcal{C}_{M,\lambda}^{-\frac{1}{2}} \nabla G_{\theta_{0}}(u)(x) \nabla G_{\theta_{0}}(u)(x)^{\top}\mathcal{C}_{M,\lambda}^{-\frac{1}{2}}$.  We have $\|f(u,x)\|\leq \kappa^2/\lambda:=B$ and

$$
\mathbb{E}_{\rho_x}[f(u,x)^2]\preccurlyeq  \frac{\kappa^2}{\lambda}\mathcal{C}_{M}\mathcal{C}_{M,\lambda}^{-1}:=\mathcal{V}.
$$

From Proposition \ref{OPboundcorro0} we have with probability at least $1-\delta$,
\begin{align}\label{skyr}
\left\|\mathcal{C}_{M,\lambda}^{-\frac{1}{2}}\left(\widehat{\mathcal{C}}_{M}-\mathcal{C}_{M}\right) \mathcal{C}_{M,\lambda}^{-\frac{1}{2}}\right\|\leq \frac{2 \kappa^2 \tilde{\beta}}{3 \lambda\nX}+\frac{2 \kappa^2 \beta}{3\lambda \nU}+\sqrt{\frac{2\kappa^2\|\tilde{\mathcal{V}}\| \tilde{\beta}}{\lambda\nX}}+\sqrt{\frac{2\kappa^2 \beta}{\lambda\nU}}\,,
\end{align}

with $\beta:=\log \frac{4 \operatorname{tr} \mathcal{V}}{\|\mathcal{V}\| \delta},\,\,\tilde{\beta}:=\log \frac{4 \operatorname{tr} \tilde{\mathcal{V}}}{\|\tilde{\mathcal{V}}\| \delta}$ and $\tilde{\mathcal{V}}:=\frac{1}{\nU}\sum_{i=1}^{\nU}\mathbb{E}_{\rho_x} f(u_i,x)$  and where we used  $\|\mathcal{V}\|\leq\kappa^2/\lambda$ .

For $\beta$ we have
\begin{align}\label{skyr2}
\nonumber\beta&=\log \frac{4 \operatorname{tr} \mathcal{V}}{\|\mathcal{V}\| \delta} =\log \frac{4 \operatorname{Tr} \mathcal{C}_{M} \mathcal{C}_{M, \lambda}^{-1}}{\|\mathcal{C}_{M}\mathcal{C}_{M,\lambda}^{-1}\| \delta}\\
&= \log \frac{4 \operatorname{Tr} \mathcal{C}_{M} \mathcal{C}_{M, \lambda}^{-1}(\|\mathcal{C}_{M}\| +\lambda)}{\|\mathcal{C}_{M}\| \delta}\nonumber\\
&\leq \log \frac{4 \frac{\kappa^2}{\lambda}\|\mathcal{C}_{M}\| +4Tr(\mathcal{C}_M)}{\|\mathcal{C}_{M}\| \delta}\nonumber\\
&\leq \log \frac{8 \kappa^2}{\lambda\|\mathcal{C}_{M}\| \delta}\\
&\leq \log \frac{16 \kappa^4}{\lambda\|\mathcal{L}_\infty\| \delta}:=C_{\lambda,\delta},
\end{align}

where we used for the last equation, that we have from Proposition \ref{CMbound},
\begin{align}\label{skyr6}
\|\mathcal{C}_M\|=\|\mathcal{L}_M\|\geq\frac{1}{2}\|\mathcal{L}_\infty\|.
\end{align}

Now we want to further bound $\tilde{\mathcal{V}}$ and $\tilde{\beta}$. 

First we set $\xi_i=\mathbb{E}_{\rho_x} f(u_i,x)$. Note that $\|\xi\| \leq \kappa^2 /\lambda$ and 
$$\mathbb{E}\xi^2 \preccurlyeq \mathbb{E}f(u,x)^2 \preccurlyeq \mathcal{V} .$$ 

From Proposition \ref{OPbound0} and \eqref{skyr2} we therefore have with probability at least $1-\delta$,

$$
\left\|\tilde{\mathcal{V}}-\mathcal{C}_{M} \mathcal{C}_{M, \lambda}^{-1}\right\| 
\leq \frac{2 \kappa^2 C_{\lambda,\delta}}{3 \lambda\nU}+\kappa\sqrt{\frac{2C_{\lambda,\delta}}{\lambda\nU}}.
$$

Now let $\nU\geq \frac{8 C_{\lambda,\delta}\kappa^2}{\lambda}$, then we have $\left\|\tilde{\mathcal{V}}-\mathcal{C}_{M} \mathcal{C}_{M, \lambda}^{-1}\right\| 
\leq  1$ and therefore

\begin{align}\label{skyr4}
\left\|\tilde{\mathcal{V}}\right\|\leq\left\|\tilde{\mathcal{V}}-\mathcal{C}_{M} \mathcal{C}_{M, \lambda}^{-1}\right\|+\|\mathcal{C}_{M} \mathcal{C}_{M, \lambda}^{-1}\|\leq2. 
\end{align}

If we assume $\nU\geq \frac{72\kappa^4C_{\lambda,\delta}\|\mathcal{L}_\infty\|}{\lambda},$ we instead have
\begin{align}
\left\|\tilde{\mathcal{V}}-\mathcal{C}_{M} \mathcal{C}_{M, \lambda}^{-1}\right\| 
\leq \frac{\|\mathcal{L}_\infty\|}{5\kappa^2}.\nonumber
\end{align}
This implies together with the reversed triangle inequality, 
\begin{align}\label{skyr7}
\nonumber\|\tilde{\mathcal{V}}\|&\geq \|\mathcal{C}_{M} \mathcal{C}_{M, \lambda}^{-1} \|- \left\|\tilde{\mathcal{V}}-\mathcal{C}_{M} \mathcal{C}_{M, \lambda}^{-1}\right\|\nonumber \\
&\geq \|\mathcal{C}_{M} \mathcal{C}_{M, \lambda}^{-1} \|-  \frac{\|\mathcal{L}_\infty\|}{5\kappa^2}= \frac{\|\mathcal{C}_M\|}{\|\mathcal{C}_M\|+\lambda}-  \frac{\|\mathcal{L}_\infty\|}{5\kappa^2\nonumber}\\
&\geq \frac{\|\mathcal{L}_\infty\|}{4\kappa^2}-  \frac{\|\mathcal{L}_\infty\|}{5\kappa^2}\geq \frac{\|\mathcal{L}_\infty\|}{20\kappa^2},
\end{align}
where we used \eqref{skyr6}  and $\lambda\leq1\leq \kappa^2$ in the last step.

Plugging \eqref{skyr7} into $\tilde{\beta}$ gives,
\begin{align}\label{skyr8}
\tilde{\beta}=\log \frac{4 \operatorname{tr} \tilde{\mathcal{V}}}{\|\tilde{\mathcal{V}}\| \delta}\leq\log \frac{80\kappa^4}{\|\mathcal{L}_{\infty}\| \lambda\delta}:=C_{\lambda,\delta,\kappa}.
\end{align}
To sum up we have $\left\|\tilde{\mathcal{V}}\right\|\leq2$ and $\,\,\tilde{\beta},\,\beta\leq C_{\lambda,\delta,\kappa}$. Plugging these bounds into \eqref{skyr} leads to
\begin{align}
\left\|\mathcal{C}_{M,\lambda}^{-\frac{1}{2}}\left(\widehat{\mathcal{C}}_{M}-\mathcal{C}_{M}\right) \mathcal{C}_{M,\lambda}^{-\frac{1}{2}}\right\|\leq \frac{2 \kappa^2  C_{\lambda,\delta,\kappa}}{3 \lambda\nX}+\frac{2 \kappa^2  C_{\lambda,\delta,\kappa}}{3\lambda \nU}+\sqrt{\frac{4\kappa^2 C_{\lambda,\delta,\kappa}}{\lambda\nX}}+\sqrt{\frac{2\kappa^2  C_{\lambda,\delta,\kappa}}{\lambda\nU}}\,.
\end{align}
using the assumption $\nU,\,\nX\geq \frac{72\kappa^4C_{\lambda,\delta,\kappa}(1+\|\mathcal{L}_\infty\|)}{\lambda}$, proves the result.

$E_2)$ The result now follows from Proposition \ref{OPboundnonconc} and the event $E_1$.

\end{proof}

\vspace{0.3cm}

\begin{proposition}
\label{OPboundevents2}
For any $\lambda>0$ define the following events,

\begin{align}
&E_1=\left\{\left\|\mathcal{C}_{M,\lambda}^{-\frac{1}{2}}\left(\widehat{\mathcal{C}}_{M}-\mathcal{C}_{M}\right)\right\|_{HS}\leq  2 \left(\frac{\kappa^2}{\sqrt{\lambda}}\left(\frac{1}{\nU}+\frac{1}{\nX}\right)+\sqrt{25\kappa^2\cN_{\mathcal{L}_{\infty}}(\lam)}\left(\frac{1}{\sqrt{\nU}}+\frac{1}{\sqrt{\nX}}\right)\right) \log ^{3/2}\left(\frac{12}{\delta} \right)\;  \right\},\\
&E_2=\left\{\left\|\widehat{\mathcal{C}}_{M}-\mathcal{C}_{M}\right\|_{HS} \leq  
\left(2\kappa^2\left(\frac{1}{\nU}+\frac{1}{\nX}\right)+4\kappa^2\left(\frac{1}{\sqrt{\nU}}+\frac{1}{\sqrt{\nX}}\right)\right) \log ^{3/2}\left(\frac{12}{\delta} \right)   \right\}\,.
\end{align}
Providing Assumption \ref{ass:input}  we have for any $\delta \in(0,1)$ that each of the above events holds true with probability at least $1-\delta$  .
\end{proposition}

\begin{proof}

\textbf{$E_1)$} Set $f(u_i,x_j):= \mathcal{C}_{M,\lambda}^{-\frac{1}{2}} \nabla G_{\theta_{0}}(u_i)(x_j) \nabla G_{\theta_{0}}(u_i)(x_j)^{\top}$. Note that 
\begin{align*}
\|f(u_i,x_j)\|_{HS}&\leq \kappa^2/\sqrt{\lambda}:=B
\end{align*}

For the second moment we have, 

\begin{align*}
\mathbb{E}\left\|f(u,x)\right\|_{HS}^2&\leq\kappa^2 Tr(\mathcal{C}_{M,\lambda}^{-1}\mathbb{E}\left[ \nabla G_{\theta_{0}}(u)(x) \nabla G_{\theta_{0}}(u)(x)^{\top}\right])\\[5pt]
&= \kappa^2 Tr(\mathcal{C}_{M}\mathcal{C}_{M,\lambda}^{-1})=\kappa^2 \mathcal{N}_M(\lambda),
\end{align*}
where we used for the last equation, that  $\Sigma_M=Z_M Z_M^*$ and $\mathcal{C}_{M}=\mathcal{Z}_M^* \mathcal{Z}_M$, so therefore

\begin{align}\label{skyr3}
\mathcal{N}_M(\lambda)=\operatorname{Tr} \Sigma_M \Sigma_{M, \lambda}^{-1}=\operatorname{Tr} \mathcal{Z}_M^* \Sigma_{M, \lambda}^{-1} \mathcal{Z}_M=\operatorname{Tr} \mathcal{C}_{M} \mathcal{C}_{M, \lambda}^{-1}.
\end{align}

Further we have by Proposition \ref{Opbound3}:
\begin{align}\label{müslii3}
\mathcal{N}_{\mathcal{L}_{M}}(\lambda)\leq 
 \left(1+2\log\frac{4}{\delta}\right)4\mathcal{N}_{\mathcal{L}_{\infty}}(\lambda) 
 \leq 12 \log(4/\delta)\;\cN_{\mathcal{L}_{\infty}}(\lam) \;. 
\end{align}

From Proposition \ref{OPboundcorro} we therefore have for $V^2:= 12 \kappa^2\log(4/\delta)\;\cN_{\mathcal{L}_{\infty}}(\lam)$
$$\left\|\mathcal{C}_{M,\lambda}^{-\frac{1}{2}}\left(\widehat{\mathcal{C}}_{M}-\mathcal{C}_{M}\right)\right\|_{HS}\leq2 \left(\frac{B}{\nU}+\frac{V}{\sqrt{\nU}} + \frac{B}{\nX}+\sqrt{\frac{Z}{\nX}}\right) \log \left(\frac{12}{\delta} \right),$$
with $Z:=\left(\frac{B^2}{\nU}+\frac{BV}{\sqrt{\nU}}\right) \log \left(\frac{12}{\delta} \right)+ V^2$.

Note that if we let $\nU\geq \frac{\kappa^2}{\lambda}\log^2\left(12/\delta\right)$ we further obtain 
$Z\leq 25 \kappa^2 \log(4/\delta)\;\cN_{\mathcal{L}_{\infty}}(\lam)$, where we used that $ \log(4/\delta)\;\cN_{\mathcal{L}_{\infty}}(\lam)\geq1$. Therefore we have
\begin{align*}
&\left\|\mathcal{C}_{M,\lambda}^{-\frac{1}{2}}\left(\widehat{\mathcal{C}}_{M}-\mathcal{C}_{M}\right)\right\|_{HS}\\
 &\leq 
2 \left(\frac{\kappa^2}{\sqrt{\lambda}}\left(\frac{1}{\nU}+\frac{1}{\nX}\right)+\sqrt{25\kappa^2\cN_{\mathcal{L}_{\infty}}(\lam)}\left(\frac{1}{\sqrt{\nU}}+\frac{1}{\sqrt{\nX}}\right)\right) \log ^{3/2}\left(\frac{12}{\delta} \right)\;  .
\end{align*}

\textbf{$E_2)$}  Set $f(u_i,x_j):=  \nabla G_{\theta_{0}}(u_i)(x_j) \nabla G_{\theta_{0}}(u_i)(x_j)^{\top}$. Note that 
\begin{align*}
\|f(u_i,x_j)\|_{HS}&\leq \kappa^2=:B
\end{align*}

For the second moment we have, 

\begin{align*}
\mathbb{E}\left\|f(u,x)\right\|_{HS}^2&\leq\kappa^4=:V^2
\end{align*}
From Proposition \ref{OPboundcorro} we therefore have 
\begin{align*}
\left\|\widehat{\mathcal{C}}_{M}-\mathcal{C}_{M}\right\|_{HS}&
\leq2 \left(\frac{\kappa^2}{\nU}+\frac{\kappa^2}{\sqrt{\nU}} + \frac{\kappa^2}{\nX}+\sqrt{\frac{Z}{\nX}}\right) \log \left(\frac{12}{\delta} \right)\\
&\leq \left(2\kappa^2\left(\frac{1}{\nU}+\frac{1}{\nX}\right)+4\kappa^2\left(\frac{1}{\sqrt{\nU}}+\frac{1}{\sqrt{\nX}}\right)\right) \log ^{3/2}\left(\frac{12}{\delta} \right),
\end{align*}
with $Z:=\left(\frac{\kappa^4}{\nU}+\frac{\kappa^4}{\sqrt{\nU}}\right) \log \left(\frac{12}{\delta} \right)+ \kappa^4$.

\end{proof}


\vspace{0.3cm}

\begin{proposition}
\label{Opbound5}
For any $\lambda\in(0,1]$, any $M\geq8\kappa^4\|\mathcal{L}_\infty\|^{-1}\log^2 \frac{2}{\delta}$ and $\nU,\,\nX\geq \frac{72\kappa^4C_{\lambda,\delta,\kappa}(1+\|\mathcal{L}_\infty\|)}{\lambda}$ with $C_{\lambda,\delta,\kappa}:=\log \frac{80\kappa^4}{\|\mathcal{L}_{\infty}\| \lambda\delta}$, we have that with probability at least $1-4\delta$,
\begin{align}
\left\|\widehat{\mathcal{C}}_{M,\lambda}^{-1} \mathcal{C}_{M, \lambda}^{1}\right\|\leq  2 \left(\frac{\kappa^2}{\lambda}\left(\frac{1}{\nU}+\frac{1}{\nX}\right)+5\kappa\sqrt{\cN_{\mathcal{L}_{\infty}}(\lam)}\left(\frac{1}{\sqrt{\lambda\nU}}+\frac{1}{\sqrt{\lambda\nX}}\right)\right) \log ^{3/2}\left(\frac{12}{\delta} \right).
\end{align}
If in addition $ \nU,\,\nX\geq \frac{400\kappa^2\cN_{\mathcal{L}_{\infty}}(\lam)}{\lambda}\log ^{3}\left(\frac{12}{\delta}\right)$

\begin{align*}
\left\|\widehat{\mathcal{C}}_{M, \lambda}^{-1}\mathcal{C}_{M, \lambda}\right\| \leq 2.
\end{align*}
\end{proposition}

\begin{proof}
Using Proposition \ref{OPboundevents2} we obtain with probability at least $1-3\delta$,
\begin{align}
\nonumber&\left\|\widehat{\mathcal{C}}_{M,\lambda}^{-1} \mathcal{C}_{M, \lambda}\right\|\\
\nonumber&\leq\left\|\widehat{\mathcal{C}}_{M,\lambda}^{-1}\left(\widehat{\mathcal{C}}_{M}-\mathcal{C}_{M}\right)\right\|_{HS}+1\\
\nonumber&\leq\frac{1}{\sqrt{\lambda}} \left\|\widehat{\mathcal{C}}_{M,\lambda}^{-\frac{1}{2}} \mathcal{C}_{M, \lambda}^{\frac{1}{2}}\right\|\left\|\mathcal{C}_{M, \lambda}^{-\frac{1}{2}}\left(\widehat{\mathcal{C}}_{M}-\mathcal{C}_{M}\right)\right\|_{HS}+1\\
&\leq  \left\|\widehat{\mathcal{C}}_{M,\lambda}^{-\frac{1}{2}} \mathcal{C}_{M, \lambda}^{\frac{1}{2}}\right\| \left(\frac{\kappa^2}{\sqrt{\lambda}}\left(\frac{1}{\nU}+\frac{1}{\nX}\right)+5\sqrt{\kappa^2\cN_{\mathcal{L}_{\infty}}(\lam)}\left(\frac{1}{\sqrt{\lambda\nU}}+\frac{1}{\sqrt{\lambda\nX}}\right)\right) \log ^{3/2}\left(\frac{12}{\delta} \right)+1.\label{cond12}
\end{align}
From Proposition \ref{OPboundevents0} we have with probability at least $1-\delta$,
\begin{align}
\left\|\widehat{\mathcal{C}}_{M,\lambda}^{-\frac{1}{2}}\mathcal{C}_{M, \lambda}^{\frac{1}{2}}\right\| \leq 2\label{cond23}
\end{align}

and therefore 
\begin{align}
\left\|\widehat{\mathcal{C}}_{M,\lambda}^{-1} \mathcal{C}_{M, \lambda}^{1}\right\|\leq  2\left(\frac{\kappa^2}{\lambda}\left(\frac{1}{\nU}+\frac{1}{\nX}\right)+5\sqrt{\kappa^2\cN_{\mathcal{L}_{\infty}}(\lam)}\left(\frac{1}{\sqrt{\lambda\nU}}+\frac{1}{\sqrt{\lambda\nX}}\right)\right)\log ^{3/2}\left(\frac{12}{\delta} \right).
\end{align}
The above inequality therefore holds if we condition on the events (\eqref{cond12},  \eqref{cond23}). 
Using Proposition \ref{conditioning} now shows that the above inequality holds with probability at least $1-4\delta$ .

If in addition $\nU,\,\nX\geq \frac{400\kappa^2\cN_{\mathcal{L}_{\infty}}(\lam)}{\lambda}\log ^{3}\left(\frac{12}{\delta}\right)$ we further obtain

\begin{align*}
\left\|\widehat{\mathcal{C}}_{M, \lambda}^{-1}\mathcal{C}_{M, \lambda}\right\|\leq2.
\end{align*}

\end{proof}

\vspace{0.3cm}

\begin{proposition}
\label{prop:intermediate}
Let $M\geq \frac{8 \tilde{d}\kappa^2 C_{\lambda,\delta,\kappa}}{\lambda},\,\nU\geq \frac{\kappa^2}{\lambda}\log^2\left(12/\delta\right)$ with $C_{\lambda,\delta,\kappa}:=\log \frac{80\kappa^4}{\|\mathcal{L}_{\infty}\| \lambda\delta}$. 
We have with probability at least $1-\delta$ that the following event holds true:

\begin{align*}
&\left\|\mathcal{C}_{M, \lambda}^{-1 / 2}\left(\widehat{\mathcal{Z}}_M^* \mathbf{v}-\mathcal{Z}_M^* G^*\right)\right\|\\
 &\leq 
2 \left(\frac{\kappa}{\sqrt{\lambda}}\left(\frac{1}{\nU}+\frac{1}{\nX}\right)+5\sqrt{\cN_{\mathcal{L}_{\infty}}(\lam)}\left(\frac{1}{\sqrt{\nU}}+\frac{1}{\sqrt{\nX}}\right)\right) \log ^{3/2}\left(\frac{12}{\delta} \right)\;  .
\end{align*}

\end{proposition}

\begin{proof}[Proof of Proposition \ref{prop:intermediate}]

By Definition \ref{weightops} we have 
$$
\left\|\mathcal{C}_{M, \lambda}^{-1 / 2}\left(\widehat{\mathcal{Z}}_M^* \mathbf{v}-\mathcal{Z}_M^* G^*\right)\right\|_2 =\left\|\frac{1}{\nU\nX} \sum_{i=1}^{\nU} \sum_{j=1}^{\nX} f(z_i,x_i) - \mathbb{E} f(z,x)\right\|_2
$$

where $f(z_i,x_j):=\mathcal{C}_{M, \lambda}^{-1 / 2}\nabla G_{\theta_{0}}(u_i)(x_j) v_i(x_j)$ and $z_i=(u_i,v_i)$.

Note that by Assumption \ref{ass:input} we have 
\begin{align*}
\|f(z,x)\|_2\leq & \left\|\mathcal{C}_{M, \lambda}^{-1 / 2} \nabla G_{\theta_{0}}(u)(x) \right\| _2\leq \frac{1}{\sqrt{\lambda}} \sup _{x \in \mathcal{X},u\in\mathcal{U}}\left\|\nabla G_{\theta_{0}}(u)(x) \right\|_2\leq \frac{\kappa}{\sqrt{\lambda}}
\end{align*}
and 
\begin{align*}
\mathbb{E}\|f(z,x)\|^2_2\leq & \int_{\mathcal{U}\times\mathcal{X}}\left\|\mathcal{C}_{M, \lambda}^{-1 / 2} \nabla G_{\theta_{0}}(u)(x) \right\|^2_2:=\nu.
\end{align*}

Now we need to prove that $\nu=\mathcal{N}_M(\lambda)$. Note that we have $\Sigma_M=Z_M Z_M^*$ and $\mathcal{C}_{M}=\mathcal{Z}_M^* \mathcal{Z}_M$, so
$$
\mathcal{N}_M(\lambda)=\operatorname{Tr} \Sigma_M \Sigma_{M, \lambda}^{-1}=\operatorname{Tr} \mathcal{Z}_M^* \Sigma_{M, \lambda}^{-1} \mathcal{Z}_M=\operatorname{Tr} \mathcal{C}_{M} \mathcal{C}_{M, \lambda}^{-1},
$$
since  $\mathcal{Z}_M^* \Sigma_{M, \lambda}^{-1} \mathcal{Z}_M=\mathcal{C}_{M} \mathcal{C}_{M, \lambda}^{-1}$. By the the ciclicity of the trace and the definition of $\mathcal{C}_{M}$, we have
$$
\operatorname{Tr} \mathcal{C}_{M} \mathcal{C}_{M, \lambda}^{-1}=\int_{\mathcal{X}\times \mathcal{U}} \operatorname{Tr}\left(\nabla G_{\theta_{0}}(u)(x)  \nabla G_{\theta_{0}}(u)(x) ^{\top} \mathcal{C}_{M, \lambda}^{-1}\right) =\int_{\mathcal{X}\times \mathcal{U}}\left\|\mathcal{C}_{M, \lambda}^{-1 / 2} \nabla G_{\theta_{0}}(u)(x) \right\|^2_2 =\nu .
$$Further we have by Proposition \ref{Opbound3}:
\begin{align}
\mathcal{N}_{\mathcal{L}_{M}}(\lambda)\leq 
 \left(1+2\log\frac{4}{\delta}\right)4\mathcal{N}_{\mathcal{L}_{\infty}}(\lambda) 
 \leq 12 \log(4/\delta)\;\cN_{\mathcal{L}_{\infty}}(\lam) \;. 
\end{align}
So by setting $B=\frac{\kappa}{\sqrt{\lambda}}$ and $V^2= 12 \log(4/\delta)\;\cN_{\mathcal{L}_{\infty}}(\lam)$, it follows from Proposition \ref{OPboundcorro}, that with probability at least $1-\delta$ ,
\begin{align*}
&\left\|\mathcal{C}_{M, \lambda}^{-1 / 2}\left(\widehat{\mathcal{Z}}_M^* \mathbf{v}-\mathcal{Z}_M^* G^*\right)\right\|\\
 &\leq 
2 \left(\frac{B}{\nU}+\frac{V}{\sqrt{\nU}} + \frac{B}{\nX}+\sqrt{\frac{Z}{\nX}}\right) \log \left(\frac{12}{\delta} \right)\;  ,
\end{align*}
with $Z:=\left(\frac{B^2}{\nU}+\frac{BV}{\sqrt{\nU}}\right) \log \left(\frac{12}{\delta} \right)+ V^2$. Note that if we let $\nU\geq \frac{\kappa^2}{\lambda}\log^2\left(12/\delta\right)$ we further obtain 
$Z\leq 25  \log(4/\delta)\;\cN_{\mathcal{L}_{\infty}}(\lam)$, where we used that $ \log(4/\delta)\;\cN_{\mathcal{L}_{\infty}}(\lam)\geq1$. Therefore we have
\begin{align*}
&\left\|\mathcal{C}_{M, \lambda}^{-1 / 2}\left(\widehat{\mathcal{Z}}_M^* \mathbf{v}-\mathcal{Z}_M^* G^*\right)\right\|\\
 &\leq 
2 \left(\frac{\kappa}{\sqrt{\lambda}}\left(\frac{1}{\nU}+\frac{1}{\nX}\right)+\sqrt{25\cN_{\mathcal{L}_{\infty}}(\lam)}\left(\frac{1}{\sqrt{\nU}}+\frac{1}{\sqrt{\nX}}\right)\right) \log ^{3/2}\left(\frac{12}{\delta} \right)\;  .
\end{align*}

\end{proof}

\begin{proposition}
\label{cilast}
Given that $\theta_{t}\in Q_{\theta_0}(R)$ and Assumption \ref{ass:neurons}, we have with probability at least $1-\delta$, 
\begin{align*}
&\left\|\frac{1}{\nU\nX} \sum_{i=1}^{\nU}\sum_{j=1}^{\nX} \nabla G_{\theta_{0}}(u_i)(x_j)G_{\theta_t}(u_i)(x_j)  - \mathbb{E}_{\rho_x}\nabla G_{\theta_{0}}(u_i)(x)G_{\theta_t}(u_i)(x) \right\|_{\Theta} \\
&\leq \frac{4\kappa^2}{\sqrt{\nX}} \log \left(\frac{4}{\delta} \right)\left\|\theta_t-\theta_0\right\|_{\Theta}+ \frac{2 \kappa C_{\nabla g}(R)}{\sqrt{M}}\|\theta_t-\theta_{0}\|^2_{\Theta}.
\end{align*}
\end{proposition}

\begin{proof}
Using a Taylor expansion in $\theta_{0}$ we have 
$$
G_{\theta_t}(u)(x) (u)(x)= \left\langle\nabla G_{\theta_{0}}(u)(x), \theta_t-\theta_0\right\rangle_{\Theta}+r_{(\theta_t,\theta_0)}(u)(x).
$$

Therefore
\begin{align}
&\left\|\frac{1}{\nU\nX} \sum_{i=1}^{\nU}\sum_{j=1}^{\nX} \nabla G_{\theta_{0}}(u_i)(x_j)G_{\theta_t}(u_i)(x_j)  - \mathbb{E}_{\rho_x}\nabla G_{\theta_{0}}(u_i)(x)G_{\theta_t}(u_i)(x) \right\|_{\Theta} \\
&\leq \left\|\frac{1}{\nU\nX} \sum_{i=1}^{\nU}\sum_{j=1}^{\nX} \nabla G_{\theta_{0}}(u_i)(x_j)\left\langle\nabla G_{\theta_{0}}(u_i)(x_j), \theta_t-\theta_0\right\rangle_{\Theta}  - \mathbb{E}_{\rho_x}\nabla G_{\theta_{0}}(u_i)(x)\left\langle\nabla G_{\theta_{0}}(u_i)(x), \theta_t-\theta_0\right\rangle_{\Theta} \right\|_{\Theta}\,+ \\
&\quad  \left\|\frac{1}{\nU\nX} \sum_{i=1}^{\nU}\sum_{j=1}^{\nX} \nabla G_{\theta_{0}}(u_i)(x_j)r_{(\theta_t,\theta_0)}(u_i)(x_j)  - \mathbb{E}_{\rho_x}\nabla G_{\theta_{0}}(u_i)(x)r_{(\theta_t,\theta_0)}(u_i)(x) \right\|_{\Theta}\,\\
&:=I+II.
\end{align}

For $I$ we have 
\begin{align*}
I\leq \left\|\frac{1}{\nX} \sum_{j=1}^{\nX} W_j-\mathbb{E}_{\rho_x} W_j\right\|_{HS} \left\|\theta_t-\theta_0\right\|_{\Theta}
\end{align*}
with $W_j:=\frac{1}{\nU}\sum_{i=1}^{\nU}\nabla G_{\theta_{0}}(u_i)(x_j)\nabla G_{\theta_{0}}(u_i)(x_j)^T$.
Note that $\|W\|_{HS}\leq\kappa^2$, $\|W\|_{HS}^2\leq\kappa^4$. Therefore we have from Proposition \ref{concentrationineq0} with probability at least $1-\delta$,
$$
I\leq \frac{4\kappa^2}{\sqrt{\nX}} \log \left(\frac{4}{\delta} \right)\left\|\theta_t-\theta_0\right\|_{\Theta}.
$$

For $II$ we have from Proposition \ref{prop6},
\begin{align*}
II\leq 2 \kappa \frac{C_{\nabla g}(R)}{\sqrt{M}}\|\theta_t-\theta_{0}\|^2_{\Theta}.
\end{align*}

\end{proof}

\begin{proposition}
\label{noise}
Let $T \leq C \nU^{\frac{1}{2r+b}}$ , $2r+b>1$, $\,\nX \geq t^{2\max\{0, \frac{1}{2}-r\}}$ and $M \geq M_{III}$ with $M_{III}$ from Theorem \ref{prop:random-feature-result} and some $C >0$, defined in the proof.
With probability at least $1-\delta$, we have for $t\in[T]$,
\begin{align*}
&\left|\frac{1}{\nU\nX}\sum_{i=1}^{\nU} \sum_{j=1}^{\nX}\left| F_{t}^*(u_i)(x_j) - G^*(u_i)(x_j)\right| -\mathbb{E} \left| F_{t}^*(u)(x) - G^*(u)(x)\right|\right|\\
&\quad \leq \tilde{C}_{\alpha,\kappa}\left(\frac{1}{\sqrt{\nU}} +\frac{1}{\sqrt{\nX}}\right) \log^{3/2} \left(\frac{12}{\delta} \right).
\end{align*}
\end{proposition}

\begin{proof}
 First we start with bounding $F^*_t$ and $G^*$. 
 Note that $F^*_t \in \mathcal{H}_M$, therefore there exist some $\theta_{t}^*$ such that $F_t(u)(x)=\langle \nabla G_{\theta_0}(u)(x), \theta_{t}^*\rangle$ and $\|F_t^*\|_{\mathcal{H}_M}=\|\theta_{t}^*\|_\Theta$ (see for example \cite{Ingo} Theorem 4.21). Therefore we have from Lemma \ref{lem:norm-fstar}
\begin{align*}
|F_t(u)(x)|=|\langle \nabla G_{\theta_0}(u)(x), \theta_{t}^*\rangle| \leq \kappa \|F_t^*\|_{\mathcal{H}_M}\leq  C'_{\kappa, \alpha}\;   t^{\max\{0, \frac{1}{2}-r\}} .
\end{align*}
For $G^*$ we have from Assumption \ref{ass:input} 
$$
|G^*(u)(x)|\leq \int_{\mathcal{V}} |v(x)| \rho(dv|u) \leq 1.
$$

Now set $f(u_i,x_j):= \left| F_{t}^*(u_i)(x_j) - G^*(u_i)(x_j)\right|$.

Then we have $|f(u_i,x_j)|\leq C'_{\kappa, \alpha}\;   t^{\max\{0, \frac{1}{2}-r\}}+1:= C_{\kappa, \alpha}\;   t^{\max\{0, \frac{1}{2}-r\}}:=B $
and from  \cite[Proposition A.1]{nguyen2023random} we have ( with $\lam = (\alpha t)^{-1}$ )
\begin{equation}
\mathbb{E} f(u,x)^2 = || \cS_M  F^*_t -   G^*  ||_{L^2(\mu_u)}^2 \leq C_{\alpha,r}\;   t^{-2r} :=V^2\;, 
\end{equation}
for some $C_{r,\alpha} < \infty$ as long as $M \geq M_{III}, \,\,T \leq C\nU^{\frac{1}{2r+b}}$, with $C$ from Theorem \ref{prop:random-feature-result} .

From Proposition \ref{OPboundcorro} we therefore have,

\begin{align*}
&\left|\frac{1}{\nU\nX}\sum_{i=1}^{\nU} \sum_{j=1}^{\nX}\left| F_{t}^*(u_i)(x_j) - G^*(u_i)(x_j)\right| -\mathbb{E} \left| F_{t}^*(u)(x) - G^*(u)(x)\right|\right|\\
&\quad\leq 2 \left(\frac{B}{\nU}+\frac{V}{\sqrt{\nU}} + \frac{B}{\nX}+\sqrt{\frac{Z}{\nX}}\right) \log \left(\frac{12}{\delta} \right)
\end{align*}

with $Z:=\left(\frac{B^2}{\nU}+\frac{BV}{\sqrt{\nU}}\right) \log \left(\frac{12}{\delta} \right)+ V^2$. From $T \leq C_{\alpha,r}\nU^{\frac{1}{2r+b}}$ for some $C_{\alpha,r}>0$ we further have,

\begin{align*}
&\left|\frac{1}{\nU\nX}\sum_{i=1}^{\nU} \sum_{j=1}^{\nX}\left| F_{t}^*(u_i)(x_j) - G^*(u_i)(x_j)\right| -\mathbb{E} \left| F_{t}^*(u)(x) - G^*(u)(x)\right|\right|\\
&\quad\leq  \hat{C}_{\alpha,\kappa}\left(\frac{ t^{\max\{0, \frac{1}{2}-r\}}}{\nU}+\frac{1}{t^{r}\sqrt{\nU}} + \frac{ t^{\max\{0, \frac{1}{2}-r\}}}{\nX}+\frac{1}{\sqrt{\nX}}\right) \log^{3/2} \left(\frac{12}{\delta} \right)\\
&\quad\leq \tilde{C}_{\alpha,\kappa}\left(\frac{1}{\sqrt{\nU}} +\frac{1}{\sqrt{\nX}}\right) \log^{3/2} \left(\frac{12}{\delta} \right).
\end{align*}
\end{proof}

\vspace{0.3cm}


\begin{proof}[Proof of Proposition \ref{OPbound2}]

Recall that 
\begin{align*}
 K_M(u, u') &=  (\Phi^M_u)^*  \Phi^M_{u'} = \sum_{m=1}^{3M} \partial_mG_{\theta_{0}}(u)\otimes\partial_m G_{\theta_{0}}(u')\;\\
&=\frac{1}{M} \sum_{m=1}^{M} \sigma\left(\left\langle b_{m}^{(0) }, J(u)\right\rangle \right) \otimes\sigma\left(\left\langle b_{m}^{(0) }, J(u')\right\rangle\right)\,+\\
&\,\,\,\,\,\,\,\,\,\frac{1}{M} \sum_{m=1}^{M} \tau^2\sigma^{\prime}\left(\left\langle b_{m}^{(0) }, J(u) \right\rangle \right)J(u)\otimes \sigma^{\prime}\left(\left\langle b_{m}^{(0)}, J(u')\right \rangle \right)J(u'),
\end{align*}

We now set 
\begin{align*}
w_m &=  \sigma\left(\left\langle b^{(0) }_m, J(u)\right\rangle \right) \otimes\sigma\left(\left\langle b^{(0) }_m, J(u')\right\rangle\right)\,+\\
&\,\,\,\,\,\,\,\,\,\tau^2\sigma^{\prime}\left(\left\langle b^{(0) }_m, J(u)\right\rangle \right)J(u)\otimes \sigma^{\prime}\left(\left\langle b^{(0)}_m, J(u')\right \rangle \right)J(u') \,.
\end{align*}
Note that $\|w_1\|_{HS}\leq\kappa^2$and $\mathbb{E}\left[\left\|w_{1}\right\|^{2}_{HS}\right] \leq \kappa^4$. The result now follows from \ref{concentrationineq0}.
\end{proof}


\end{document}